\documentclass[nohyperref]{article}

\usepackage{microtype}
\usepackage{graphicx}
\usepackage{subfigure}
\usepackage{booktabs} 

\usepackage{hyperref}

\usepackage[accepted]{icml2023}

\usepackage{amsmath}
\usepackage{amssymb}
\usepackage{mathtools}
\usepackage{amsthm}

\usepackage[normalem]{ulem}

\usepackage[capitalize,noabbrev]{cleveref}

\theoremstyle{plain}
\newtheorem{theorem}{Theorem}[section]
\newtheorem{proposition}[theorem]{Proposition}
\newtheorem{lemma}[theorem]{Lemma}

\theoremstyle{definition}
\newtheorem{definition}[theorem]{Definition}

\theoremstyle{remark}

\newtheorem{example}[theorem]{Example}

\usepackage[textsize=tiny]{todonotes}

\usepackage{bm}
\usepackage{thm-restate}

\usepackage{xspace}
\newcommand{\ttpois}{TT-POIS\@\xspace}
\newcommand{\pois}{POIS\@\xspace}

\newcommand{\E}{\mathbb{E}}

\newcommand{\Var}{\mathbb{V}\mathrm{ar}}

\newcommand{\argmax}{\operatornamewithlimits{argmax}}

\icmltitlerunning{Truncating Trajectories in Monte Carlo Reinforcement Learning}

\begin{document}

\twocolumn[
\icmltitle{Truncating Trajectories in Monte Carlo Reinforcement Learning}





\begin{icmlauthorlist}
\icmlauthor{Riccardo Poiani}{polimi}
\icmlauthor{Alberto Maria Metelli}{polimi}
\icmlauthor{Marcello Restelli}{polimi}
\end{icmlauthorlist}

\icmlaffiliation{polimi}{Diparimento di Elettronica, Informazione e Bioingegneria, Politecnico di Milano, Milan, Italy}

\icmlcorrespondingauthor{Riccardo Poiani}{riccardo.poiani@polimi.it}

\icmlkeywords{Reinforcement Learning, Monte Carlo Simulation, Budget Optimization, Policy Optimization}

\vskip 0.3in
]



\printAffiliationsAndNotice{}  

\begin{abstract}
In Reinforcement Learning (RL), an agent acts in an unknown environment to maximize the expected cumulative discounted sum of an external reward signal, i.e., the expected return. In practice, in many tasks of interest, such as policy optimization, the agent usually spends its interaction budget by collecting episodes of \emph{fixed length} within a simulator (i.e., Monte Carlo simulation). However, given the discounted nature of the RL objective, this data collection strategy might not be the best option. Indeed, the rewards taken in early simulation steps weigh exponentially more than future rewards. Taking a cue from this intuition, in this paper, we design an a-priori budget allocation strategy that leads to the collection of trajectories of different lengths, i.e., \emph{truncated}. The proposed approach provably minimizes the width of the confidence intervals around the empirical estimates of the expected return of a policy. After discussing the theoretical properties of our method, we make use of our trajectory truncation mechanism to extend Policy Optimization via Importance Sampling~\citep[POIS,][]{metelli2018policy} algorithm. Finally, we conduct a numerical comparison between our algorithm and POIS: the results are consistent with our theory and show that an appropriate truncation of the trajectories can succeed in improving performance.
\end{abstract}

\section{Introduction}
\label{sec:intro}

In Reinforcement Learning \citep[RL,][]{sutton2018reinforcement}, an agent acts in an unknown, or partially known, environment to maximize the expected cumulative discounted sum of an external reward signal, referred to as expected return. This abstract scenario models a large variety of sequential decision-making problems \citep[e.g.,][]{mnih2016asynchronous,casas2017deep,schulman2017proximal,shi2020cooperative}, and, consequently, is constantly gaining attention from the community. A particular appealing feature is that RL is fully data-driven. Indeed, the designer of the learning system only needs to let the agent interacts with the environment to gather experience on the task of interest, and no additional expert knowledge on the problem is required. 

However, given the well-known data inefficiency of RL algorithms, in real-world scenarios, a simulator is commonly adopted, and the agent interacts with it, usually in parallel over a cluster of machines to gather knowledge (e.g., performance and gradient estimates) on the task being solved \citep[e.g.,][]{espeholt2018impala,liang2018rllib}. Furthermore, since the goal consists on estimating/maximizing an infinite sum of rewards, in practice, the designer usually chooses a sufficiently large horizon $T$, so that the agent will gather information up to time $T$ via Monte Carlo simulation \citep{owen2013monte}, after which the state of the system is reset to a (possibly stochastic) initial state. Although there exists alternatives, such as Temporal Difference~\citep[TD,][]{sutton2018reinforcement} methods, that do not require a finite horizon nor a reset possibility, a large variety of successful RL approaches still rely on Monte Carlo evaluation. Indeed, differently from TD methods, Monte Carlo approaches can be transparently applied to non-Markovian environments, as often happens in real-world domains.
This is, indeed, the usual case of policy search methods \citep[e.g.,][]{williams1992simple,baxter2001infinite,lillicrap2015continuous,schulman2015trust,schulman2017proximal,metelli2018policy,cobbe2021phasic}. 
While these algorithms can differ across a large number of dimensions (see, for instance, \citet{metelli2018policy} for an in-depth taxonomy), most of them share a common aspect: the evaluation and optimization of the objective function are performed by collecting, via Monte Carlo simulation, a batch of $K$ episodes of length $T$ each. In this sense, they allocate the budget of $\Lambda = KT$ transitions \emph{uniformly} w.r.t. the horizon.

However, given the discounted nature of the RL objective, coupled with the fact that, in practice, we have to estimate the expected return with sample means, \emph{is this uniform-in-the-horizon budget allocation strategy the best option?} Indeed, the discounted objective weighs each reward collected at step $t$ with the factor $\gamma^t$, and, consequently, the early interaction steps weigh exponentially more than the late ones. Building on this observation, in this work, we aim at answering the previous question by investigating alternative and non-uniform budget allocation strategies. More specifically, we tackle the problem from a worst-case scenario, which is agnostic to the underlying MDP and policy to evaluate/optimize, and we investigate whether it is possible to design an alternative schedule of trajectories' length that comes with desirable robustness properties w.r.t. to the usual uniform-in-the-horizon scheme. In other words, we aim at understanding whether the possibility of resetting trajectories, which is usually available in a large variety of RL simulators, and indispensable for Monte Carlo simulation, can successfully be exploited to increase some quality index related to the estimation accuracy.

\paragraph{Contributions and Outline} After introducing the background (Section \ref{sec:prelim}), we consider the problem of estimating the expected return of a policy via trajectory-based Monte Carlo simulation with a finite budget $\Lambda$ of transitions (Section \ref{sec:eval}). For presentation purposes, we first focus on the on-policy setting, and propose a novel estimator for the expected return, which uses trajectories of \emph{different lengths}, i.e., \emph{truncated}. Then, to investigate alternative budget allocation strategies from our worst-case perspective, we provide a generalization of the H\"{o}effding confidence intervals  \citep{boucheron2003concentration} to our estimator, and we frame our goal as finding the trajectories' length schedule that minimizes such intervals. In this sense, we design an approximately optimal fixed strategy that provably minimizes the width of these confidence intervals around the empirical mean of the RL objective. As our theory verifies, our schedule leads to collecting trajectories of different lengths. We then analyze our solution from a Probabilistic Approximately Correct \citep[PAC,][]{even2002pac} perspective and discuss its benefits, in terms of the resulting PAC bound, w.r.t. the usual uniform-in-the-horizon approach. We conclude Section \ref{sec:eval} by extending our approach to the more challenging off-policy evaluation setting. To this end, we minimize a generalization of the off-policy confidence intervals presented in \citet{metelli2018policy}. 
Then, while in principle one could employ the derived length schedule in any estimation/optimization algorithm that uses the trajectory-based Monte Carlo simulation as a building block, we leverage it to extend the Policy Optimization via Importance Sampling (\pois) algorithm \cite{metelli2018policy}. This choice is justified by the fact that the confidence intervals optimized in Section~\ref{sec:eval} are explicitly employed in POIS to quantify the uncertainty injected in the estimation and to build a surrogate objective function of the expected return, which is then optimized via gradient methods. For this reason, in Section \ref{sec:optim}, we make use of the truncated off-policy estimator, together with our optimized confidence intervals, to extend \pois by incorporating our trajectory truncation mechanism, presenting Truncating Trajectories in
Policy Optimization via Importance Sampling (TT-POIS). Finally, in Section \ref{sec:exp}, we empirically compare our algorithm and POIS across multiple control domains, varying the discount factor and the available budget. Our results are consistent with our theory and show that an appropriate truncation of the trajectories succeeds in improving performance.




\section{Preliminaries}\label{sec:prelim}
In this section, we provide the necessary backgrounds and notations that will be used throughout the rest of the article.

\paragraph{Markov Decision Process} A discrete-time Markov Decision Process \citep[MDP,][]{puterman1990markov} is defined as a tuple $\left(\mathcal{S}, \mathcal{A}, R, P, \gamma, \nu \right)$, where $\mathcal{S}$ is the set of states, $\mathcal{A}$ is the set of actions, $R: \mathcal{S} \times \mathcal{A} \rightarrow [0,1]$ is the reward function that assigns the reward $R(s,a)$ for taking action $a$ in state $s$, $P: \mathcal{S} \times \mathcal{A} \rightarrow \Delta(\mathcal{S})$\footnote{We denote with $\Delta(\mathcal{X})$ the set of probability distributions over a generic set $\mathcal{X}$.} is the transition kernel that specifies the probability distribution $P(\cdot|s,a)$ over the next state when taking action $a$ in state $s$, $\gamma \in (0,1)$ is the discount factor, and $\nu \in \Delta(\mathcal{S})$ is the initial-state distribution. The behavior of the agent is defined by a policy $\pi: \mathcal{S} \rightarrow \Delta(\mathcal{A})$ that provides a mapping between states and distributions over action. We define a trajectory of length $h$ as $\bm{\tau}_h \coloneqq (s_0, a_0, \dots, s_{T-1}, a_{h-1}, s_h)$, i.e., a sequence of state-action pairs of length $h$, and we define the trajectory return as $G(\bm{\tau}_h) \coloneqq \sum_{t=0}^{h-1} \gamma^t R(s_t, a_t)$. Each trajectory of length $h$ belongs to a trajectory space denoted with $\mathcal{T}_h$. The performance of the agent is evaluated in terms of \emph{expected return}, i.e., the expected cumulative discounted sum of rewards over the estimation horizon $T$:\footnote{As usual in the policy gradient literature \citep[see e.g.,][]{papini2019smoothing}, we consider the infinite-horizon discounted MDP model in our setting, but a finite horizon $T$ when introducing estimators. This is justified by the fact that, if $T = \mathcal{O}\left( \frac{1}{1-\gamma} \log \frac{1}{\epsilon} \right)$, the expected return with horizon $T$ is $\epsilon$-close to the infinite-horizon case~\citep{kakade2003sample}.} $J(\pi) \coloneqq \E_{\pi} \left[ G(\bm{\tau}_{T}) \right]$, where the expectation is taken w.r.t. the stochasticity of the policy, the environment, and the initial-state distribution.

\paragraph{Policy Optimization} For what concerns optimization tasks, we focus on the case in which the agent's policy belongs to a parametric differentiable policy space $\Pi_{\Theta} \coloneqq \left\{ \pi_{\bm{\theta}} : \bm{\theta} \in \Theta \subseteq \mathbb{R}^u \right\}$. In this context, the expected return of any policy $\pi_{\bm{\theta}}$ is usually expressed as an integral over the trajectory space $\mathcal{T}_T$. In particular, the agent's maximization objective can be re-written as:
{\thinmuskip=1mu
\medmuskip=1mu
\thickmuskip=1mu
\begin{align}\label{eq:agent-goal}
\resizebox{.43\textwidth}{!}{$\displaystyle
\argmax_{\bm{\theta} \in \Theta} J(\bm{\theta}) \coloneqq  \E_{\pi_{\bm{\theta}}} \left[ G(\bm{\tau}_{T}) \right] =  \int_{\mathcal{T}_T} p(\bm{\tau}_T| \bm{\theta},T) G(\bm{\tau}_T) \mathrm{d}\bm{\tau}_T,$}
\end{align}
}where $p(\bm{\tau}_h | \bm{\theta}, h) \coloneqq \nu(s_0) \prod_{t=0}^{h-1} \pi_{\bm{\theta}}(a_t|s_t) P(s_{t+1}|s_t, a_t)$ is the trajectory density function for trajectories of length $h$. A typical approach for solving \eqref{eq:agent-goal} is to use stochastic gradient ascent methods. For instance, the well-known REINFORCE algorithm \citep{williams1992simple}, at each iteration, spends its interaction budget $\Lambda=KT$ in collecting $K$ i.i.d. trajectories of length $T$, i.e., $\{\bm{\tau}_T^{(i)}\}_{i=1}^K$, and applies the update rule $\bm{\theta}' = \bm{\theta} + \alpha \hat{\nabla}_{\bm{\theta}} J(\bm{\theta})$, where:
\begin{align*}
\hat{\nabla}_{\bm{\theta}} J(\bm{\theta}) = \frac{1}{K} \sum_{i=1}^K \left( \sum_{t=0}^{T-1} \nabla_{\bm{\theta}} \log \pi_{\bm{\theta}}\left(a_t^{(i)} | s_t^{(i)}\right) \right) G(\bm{\tau}^{(i)}_T)
\end{align*}
represents the estimator of the \emph{policy gradient}~\citep{sutton1999policy} and $\alpha \ge 0$ is the step size.

\paragraph{Importance Sampling} Let $P$ and $Q$ be two probability measures defined over a measurable space $(\mathcal{X}, \mathcal{F})$, and assume that $P \ll Q$, i.e., $P$ is absolutely continuous w.r.t. $Q$. Let $p$ and $q$ be the density functions corresponding to $P$ and $Q$ respectively. In this setting, Importance Sampling \citep[IS,][]{owen2013monte} is a statistical tool that allows estimating expectation $\mu = \E_{x \sim P} \left[ f(x) \right]$ of a bounded function $f$ (i.e., $\|f\|_{\infty} < +\infty$) under the target distribution $P$ with samples collected with the behavioral distribution $Q$. More specifically, the IS estimator corrects the distribution mismatch via the \emph{importance weights} $\omega_{P/Q}(x) = p(x)/q(x)$:
\begin{align}\label{eq:is}
\hat{\mu}_{P/Q} = \frac{1}{K} \sum_{i=1}^K \omega_{P/Q}(x_i) f(x_i),
\end{align}
where $ \{x_i\}_{i=1}^K \sim Q$. The moments of the importance weights can be expressed in terms of the exponentiated Rényi divergence. More specifically, let $\alpha \in [0,+\infty]$, the $\alpha$-Rényi divergence between $P$ and $Q$ is defined as:
\begin{align*}
D_{\alpha}(P\|Q) = \frac{1}{\alpha - 1} \log \int_{\mathcal{X}} q(x) \left(\frac{p(x)}{q(x)}\right)^\alpha \mathrm{d}x.
\end{align*}
We define $d_{\alpha}(P\|Q) \coloneqq \exp\left(D_{\alpha}(P\|Q)\right)$ as the exponentied $\alpha$-Rényi divergece, then $\E_{x \sim Q} \left[ \omega_{P/Q}(x)^{\alpha} \right] = d_{\alpha}(P\|Q)^{\alpha-1}$. The second order moments can be used to construct the following confidence intervals on the target estimation \citep{metelli2018policy} that holds with probability at least $1-\delta$:
\begin{align}\label{eq:off-policy-ci-basics}
\E_{x \sim Q}\left[f(x)\right] \ge \hat{\mu}_{P/Q} - \|f\|_{\infty}\sqrt{\frac{(1-\delta)d_2(P\|Q)}{\delta K}}
\end{align}


\section{Truncating Trajectoris in Monte Carlo Evaluation}\label{sec:eval}
In this section, we provide the theoretical groundings behind truncating trajectories in Monte Carlo RL, with a specific focus on the problem of estimating the discounted return. Before diving into the details of our approach, we first formally specify how an agent makes use of its interaction budget. For this purpose, we introduce the novel concept of \emph{Data Collection Strategy} (DCS).



\begin{definition}[Data Collection Strategy]\label{def:dcs}
A \emph{Data Collection Strategy} (DCS) for a transition budget $\Lambda \in \mathbb{N}$ is defined as a $T$-dimensional vector $\bm{m} \coloneqq (m_1, \dots, m_T)$ such that $m_h \in \mathbb{N}$ for all $h \in \{1, \dots, T \}$, and $\sum_{h=1}^{T} m_h  h = \Lambda$.
\end{definition}

More specifically, $m_h$ represents the number of trajectories of length $h$ that the agent collects in the environment. We notice that there is a tight relationship between $\bm{m}$ and the total number of samples that the agent collects at step $t$. In particular, let $\bm{n} \coloneqq (n_0,\dots, n_{T-1})$ be the $T$-dimensional vector, where each component $n_t$ represents the number of samples collected at time $t$; then, we have that $n_t = m_t$, if $t=T-1$, and $n_{t} = n_{t+1} + m_{t+1}$ otherwise.\footnote{Notice that this implies that $n_{t} \ge n_{t+1}$.}  It follows that, given $\bm{m}$, $\bm{n}$ is uniquely identified, and vice versa; for this reason, in the rest of this paper, we will use the most convenient symbol depending on the context.
Finally, we remark that each DCS corresponds to $p_{\bm{m}}(\cdot|\bm{\theta})$, which represents the density function of the data generation process of the trajectories collected under policy $\pi_{\bm{\theta}}$ following $\bm{m}$, namely $\mathcal{D} \coloneqq \{ \{ \bm{\tau}^{(i)}_h \}_{i=1}^{m_h} \}_{h=1}^T$.



\subsection{On-Policy Data Collection Strategy}
Let us now consider the on-policy problem of estimating $J(\bm{\theta})$ with trajectories collected via Monte Carlo simulation using $\pi_{\bm{\theta}}$. We begin by investigating whether, having fixed an arbitrary DCS, it is possible to build unbiased estimators for $J(\bm{\theta})$.\footnote{Notice that a na\"ive Monte Carlo estimator such as $\frac{1}{\Lambda}\sum_{h=1}^{T} \sum_{i=1}^{m_h} \sum_{t=0}^{h-1} \gamma^t R(s_t^{(i)}, a_t^{(i)})$ is, in general, biased.} The answer turns out to be positive for a restricted class of DCSs, namely the ones for which $m_T \ge 1$ holds. Intuitively, this condition ensures that the agent gathers at least one sample for each interaction step $t \in \{0, \dots, T-1 \}$. Thus, for any DCS $\bm{m}$, we design the following estimator:

\begin{align}\label{eq:onpolicy-dcs-est}
\hat{J}_{\bm{m}}(\bm{\theta}) = \sum_{h=1}^{T} \sum_{i=1}^{m_h} \sum_{t=0}^{h-1} \gamma^t \frac{R(s_t^{(i)}, a_t^{(i)})}{n_t},
\end{align}
To provide an interpretation, we notice that, given $\mathcal{D} \sim p_{\bm{m}}(\cdot|\bm{\theta})$, Equation \eqref{eq:onpolicy-dcs-est} sums over the collected trajectories of different lengths (i.e., the two external summations) a rescaled empirical truncated return in which each reward at step $t$ is \emph{properly} divided by $n_t$, i.e., the number of samples gathered at step $t$. Intuitively, this rescaling is required to prevent the estimate to be biased toward the steps for which $n_t$ is larger. 
Furthermore, this estimator has several interesting properties. First of all, as already anticipated, as long as $m_T \ge 1$, it provides an unbiased estimate of $J(\bm{\theta})$, namely $\mathop{\E}_{p_{\bm{m}}(\cdot|\bm{\theta})} \left[ \hat{J}_{\bm{m}}(\bm{\theta}) \right] = J(\bm{{\theta}})$ (proof in Appendix \ref{app:proofs}). 
Moreover, given the uniform-in-the-horizon DCS, i.e., $\bm{m} = \left( 0, \dots, 0, \frac{\Lambda}{T}\right)$, Equation \eqref{eq:onpolicy-dcs-est} recovers the usual Monte Carlo on-policy estimator of $J(\bm{\theta})$. As we shall later see, these properties will also naturally extends to the off-policy estimation problem. 

At this point, our main objective can be framed as finding the best possible DCS among the ones that preserve the unbiasedness property. In this sense, we need to define a proper index to evaluate the candidates. In this work, we take a worst-case scenario w.r.t. the underlying MDP and policy, and we choose to minimize confidence intervals around the estimated expected return. More specifically, we derive the following generalization of the H\"{o}effding confidence intervals \citep{boucheron2003concentration} that holds for a generic DCS (proof in Appendix \ref{app:proofs}). \footnote{We remark that Proposition \ref{prop:hoeffding-general} does not directly follow from a na\"ive application of the Hoeffding inequality, and some tecnical manipulations are required to obtain Equation \eqref{eq:hoeffding-general}.}


\begin{restatable}{proposition}{hoeffdinggeneral}\label{prop:hoeffding-general}
Consider an optimization budget $\Lambda \ge T$, a generic DCS $\bm{m}$ such that $m_T \ge 1$ and $\delta \in (0, 1)$. Then, with probability at least $1-\delta$ it holds that:
\begin{align}\label{eq:hoeffding-general}
\left|\hat{J}_{\bm{m}}(\bm{{\theta}}) - J({\bm{{\theta}}})\right| \le \sqrt{\frac{1}{2} \log\left(\frac{2}{\delta}\right) \sum_{t=0}^{T-1} \frac{c_t}{n_t}},
\end{align}
where $c_t = \frac{\gamma^t(\gamma^t + \gamma^{t+1} - 2\gamma^T)}{1-\gamma}$. 
\end{restatable}

As we can notice, Equation \eqref{eq:hoeffding-general} is always well-defined. Indeed, $m_T \ge 1$ implies $n_t \ge 1$ for all $t \in \{0, \dots, T-1 \}$. Furthermore, a single term within the summation $\frac{c_t}{n_t}$ relates the width of the confidence intervals w.r.t. the number of samples gathered at timestep $t$. More specifically, we notice that $c_t$ is a decreasing function of time. Intuitively, if we are given a fixed budget $\Lambda$, we expect that, to minimize Equation \eqref{eq:hoeffding-general}, more samples should be allocated at the beginning of the horizon, corroborating our initial intuition, i.e., the convenience of truncating trajectories. Moreover, we notice that the discount factor $\gamma$ plays a crucial role in the expression of $c_t$. The lower $\gamma$, the faster the aforementioned decreasing rate, meaning that, for small $\gamma$s, a larger portion of the budget $\Lambda$ will be allocated to earlier interaction steps when minimizing Equation \eqref{eq:hoeffding-general}. Finally, it is possible to verify that, when Proposition \ref{prop:hoeffding-general} is applied with the uniform DCS, we recover the usual H\"{o}effding confidence intervals for the Monte Carlo estimation of $J(\bm{\theta})$. Therefore, given a fixed budget $\Lambda$, if we are able to find the DCS that minimizes Equation \eqref{eq:hoeffding-general}, we implicitly obtain a \emph{robustness} property w.r.t. the uniform-in-the-horizon strategy. In order to find the DCS that minimizes Equation \eqref{eq:hoeffding-general}, we formulate the following optimization problem:
\begin{equation}\label{sys:onpolicy-hard}
\begin{aligned} 
\min_{\bm{n}} \quad & f(\bm{n}) \coloneqq \sqrt{\frac{1}{2} \log\left( \frac{2}{\delta} \right) \sum_{t=0}^{T-1} \frac{c_t}{n_t} } \\
\textrm{s.t.} \quad &  \sum_{t=0}^{T-1} n_t = \Lambda \\
  & n_t \ge n_{t+1}, \quad \forall t \in \{0, \dots, T-2\}  \\
  & n_t \in \mathbb{N}_+, \quad \forall t \in \{0, \dots, T-1\}
\end{aligned}
\end{equation}
where the constraint $n_t \ge n_{t+1}$ arises from the aforementioned relationships between $\bm{m}$ and $\bm{n}$. Problem~\eqref{sys:onpolicy-hard} is a non-linear integer program that, in principle, could be addressed by means of complex solvers. However, such an approach would fail to provide an interpretable result (e.g., closed-form expression) and, thus, would be of little interest for statistical analysis.
For this reason, we follow a different path and derive an analytical expression for an approximately optimal DCS, whose form arises from solving a convex relaxation of \eqref{sys:onpolicy-hard} in which the integer constraint on $n_t$ is dropped and, then, the obtained optimal relaxed solution is rounded down and the remaining budget is allocated uniformly.
For the sake of presentation, the following Theorem (proof in Appendix \ref{app:proofs}) summarizes our result for a sufficiently large budget $\Lambda \ge \Lambda_0$. We refer the reader to Appendix \ref{app:proofs} for the exact expression of $\Lambda_0$ and for the symmetric version of Theorem \ref{theo:onpolicy-sol} that holds when $T \le \Lambda < \Lambda_0$.

\begin{restatable}{theorem}{sol}\label{theo:onpolicy-sol}
Consider an optimization budget $\Lambda \ge \Lambda_0$, let $\bm{n}^*$ be the optimal solution of \eqref{sys:onpolicy-hard}.
Let $g_t = \frac{\sqrt{c_t}}{\sum_{i=0}^{T-1} \sqrt{c_i}}\Lambda$, and let $k = \Lambda - \sum_{t=0}^{T-1} \lfloor g_t \rfloor$. Define the $t$-th component of the approximately optimal DCS $\tilde{\bm{n}}^*$ as 
$\tilde{n}^*_t \coloneqq \lfloor g_t \rfloor + \mathbf{1} \{ t < k \}$.
Then, it holds that:
\begin{align}\label{eq:onpolicy-opt-obj}
f(\bm{n}^*) \le f(\tilde{\bm{n}}^*) \le \sqrt{2} f(\bm{n}^*).
\end{align}
\end{restatable}

\begin{figure}[ht]
\vskip 0.2in
\begin{center}
\centerline{\includegraphics[width=5.5cm]{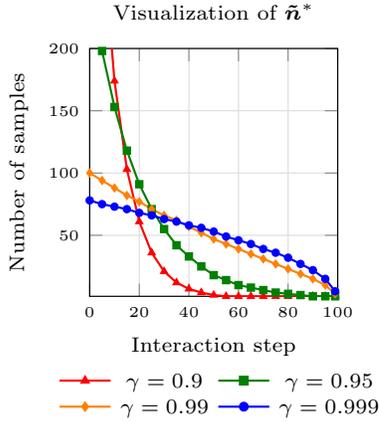}}
\caption{Visualization of $\bm{\tilde{n}}^*$ for $\Lambda=5000$, $T=100$, and different values of $\gamma$. More specifically, the $x$-axis denotes the interaction timestep $t$ and the $y$-axis reports, for each value of $t$, the number of steps prescribed by $\bm{\tilde{n}}^*$, namely $\tilde{n}^*_t$. As we can see, the behavior of $\tilde{n}^*_t$ is monotonically decreasing in $t$; furthermore, the smaller $\gamma$ is, the faster the decrease rate of $\tilde{n}^*_t$ is.}
\label{fig:strat-vis-main}
\end{center}
\vskip -0.2in
\end{figure}

Theorem \ref{theo:onpolicy-sol} deserves some comments. First of all, it provides a closed form expression for an approximately optimal DCS $\tilde{\bm{n}}^*$. Indeed, from Equation \eqref{eq:onpolicy-opt-obj} we can infer that, up to constant factors, $\tilde{\bm{n}}^*$ achieves the same confidence intervals as the true optimal DCS $\bm{n}^*$ that minimizes \eqref{sys:onpolicy-hard} (i.e., $f(\tilde{\bm{n}}^*) = \Theta(f(\bm{n}^*))$). Now, let us focus on the expression of a single term $\tilde{n}^*_t$, whose shape in visualized in Figure~\ref{fig:strat-vis-main}. \footnote{Further visualizations are provided in Appendix \ref{app:vis}.} Neglecting the indicator function and the floor, which both arise from technicalities in the analysis, the most relevant term is given by $g_t$. As one can notice, $g_t$ partitions the available budget $\Lambda$ among the different timesteps with a proportion of $\sqrt{c_t}$, which is an exponentially decreasing function of time and whose decrease rate is given by $\gamma$. In this sense, the approximately optimal DCS provably truncates trajectories by allocating more samples to the initial steps of interactions. Moreover, the smaller $\gamma$ is, the larger the amount of samples that will be allocated to earlier steps $t$, as aggressive discounting makes the future less relevant. 

To conclude this section, we remark that Theorem \ref{theo:onpolicy-sol} provably shows that truncating trajectories can successfully minimize the confidence intervals around the estimated return, for any possible pair of MDP and target policy. However, the complexity of the expression of the approximately optimal DCS $\tilde{\bm{n}}^*$ does not allow to easily quantify the improvement w.r.t. the uniform strategy. For this reason, we resort to PAC analysis \cite{even2002pac}. More specifically, given some desired confidence level $\delta \in (0, 1)$ and accuracy $\epsilon > 0$, we aim at answering the following question: which is the minimum amount of budget $\Lambda$ such that $|\hat{J}_{\bm{m}}(\bm{\theta}) - J(\bm{\theta})| \le \epsilon$ holds with probability at least $1-\delta$? It is easy to see that, for the uniform DCS, $\Lambda = \mathcal{O} \left( \frac{T\log(2/\delta)}{(1-\gamma)^2 \epsilon^2} \right)$ is sufficient for enforcing $|\hat{J}_{\bm{m}}(\bm{\theta}) - J(\bm{\theta})| \le \epsilon$. For our approximately optimal DCS, instead, we derive the following result:

\begin{restatable}{theorem}{pac}\label{theo:pac}
Let $\delta \in (0,1)$ and $\epsilon > 0$ such that $8T \epsilon^2 \le \log(2/\delta) c_0$ holds. Then, with probability at least $1-\delta$, $|\hat{J}_{\tilde{\bm{m}}^*}(\bm{\theta}) - J(\bm{\theta})| \le \epsilon$ holds provided that:
\begin{align}\label{eq:pac-improvement}
\Lambda = \mathcal{O} \left( \min \left\{ \frac{T\log(2/\delta)}{(1-\gamma)^2 \epsilon^2}, \frac{\log(2/\delta)}{(1-\gamma)^3 \epsilon^2}  \right\}  \right)
\end{align}
\end{restatable}
Theorem \ref{theo:pac} reveals a PAC bound under an assumption on the relationship between $\epsilon$, $\delta$, and $T$; technically, this is only needed to guarantee that $\Lambda \ge 2T$ holds, which is clearly a mild condition. 
That being said, Theorem \ref{theo:pac}, first of all, shows a robustness property of $\tilde{\bm{m}}^*$ w.r.t. the uniform DCS. Moreover, it improves the standard result whenever $T > \mathcal{O}\left( (1-\gamma)^{-1}\right)$, as shown in the following example.

\begin{example}
Suppose that the agent is interested in estimating the infinite-horizon expected discounted return using a finite horizon $T$ that guarantees that the final estimate has bias bounded by $1/\exp\left((1-\gamma)^{-1}\right)$. In this case, we need to select $T = \mathcal{O}\left( (1-\gamma)^{-2} \right)$ \citep{kakade2003sample}, and the improvement of our approximately optimal DCS is given by a factor $\mathcal{O}\left((1-\gamma)^{-1}\right)$ factor. This result should not be surprising. Indeed, intuitively, if the horizon increases, the difference between the optimal DCS and the uniform one increases as well (see the expression of $\tilde{n}_t^*$ vs $\Lambda/T$). 
\end{example}

At this point, we are ready to extend our result to the off-policy estimation problem.

\subsection{Off-Policy Data Collection Strategy}\label{sec:off-policy-dcs}
Consider the off-policy problem of estimating $J(\bm{\bar{\theta}})$ with trajectories collected via Monte Carlo simulation using a possibly different policy $\pi_{\bm{\theta}}$. For an arbitrary DCS $\bm{m}$, we extend Equation \eqref{eq:onpolicy-dcs-est} by proposing the following estimator:
\begin{align}\label{eq:offpolicy-dcs-est}
\hat{J}_{\bm{m}}(\bm{\bar{\theta}}/\bm{\theta}) = \sum_{h=1}^T \sum_{i}^{m_h} \omega_{\bm{\bar{\theta}}/\bm{\theta}}(\bm{\tau}_h^{(i)}) \sum_{t=0}^{h-1} \gamma^t \frac{R(s_t^{(i)}, a_t^{(i)})}{n_t},
\end{align}
where $\omega_{\bm{\bar{\theta}}/\bm{\theta}}(\bm{\tau}_h) = \prod_{t=0}^{h-1} \frac{\pi_{\bar{\bm{\theta}}}(a_t|s_t)}{\pi_{\bm{\theta}}(a_t|s_t)}$ is the importance weight for a trajectory of length $h$. Equation \eqref{eq:offpolicy-dcs-est} enjoys the same properties discussed for Equation \eqref{eq:onpolicy-dcs-est}; the only difference, indeed, stands in $\omega_{\bm{\bar{\theta}}/\bm{\theta}}(\bm{\tau}_h)$, whose purpose is taking into account the distribution shift. Following the same rationale of the on-policy setting, we derive a generalization for the off-policy confidence intervals of \citet{metelli2018policy} for the discounted off-policy return $J(\bm{\bar{\theta}})$ that holds for a generic DCS $\bm{m}$.\footnote{We remark that for Equation \eqref{eq:offpolicy-dcs-est}, that uses IS, we cannot easily apply the H\"oeffding's inequality since the importance weight distribution might be heavy tailed~\citep{lugosi2019mean}.} More specifically, we prove the following result (proof in Appendix \ref{app:proofs}).\footnote{Theorem \ref{theo:off-policy-ci-main} makes use of Cantelli’s inequality and provides one-sided tail bounds. Two-sided tail bounds can be straightforwardly derived by using Chebyshev’s inequality.}
\begin{theorem}\label{theo:off-policy-ci-main}
Consider $\pi_{\bm{\bar{\theta}}}$, $\pi_{\bm{{\theta}}} \in \Pi_{\Theta}$ such that $\pi_{\bm{\bar{\theta}}}(\cdot|s) \ll \pi_{\bm{\theta}}(\cdot |s)$ a.s. for all $s \in \mathcal{S}$. Consider an optimization budget $\Lambda \ge T$ and a generic DCS $\bm{m}$. Then, with probability at least $1-\delta$ it holds that:
\begin{align}\label{eq:offpolicy-bound-tight}\resizebox{.42\textwidth}{!}{$\displaystyle
J(\bm{\bar{\theta}}) \ge \hat{J}_{\bm{m}}(\bm{\bar{\theta}}/\bm{\theta}) - \sqrt{ \beta_{\delta} \sum_{h=1}^T m_h \phi_h^2 d_2(p(\cdot|\bm{\bar{\theta}},h) \| p(\cdot|\bm{\theta},h)) },$}
\end{align}
where $\beta_{\delta} = \frac{1-\delta}{\delta}$ and $\phi_h \coloneqq \sum_{t=0}^{h-1} \frac{\gamma^t}{n_t}$ .
\end{theorem}
However, Equation \eqref{eq:offpolicy-bound-tight} is of little practical use to derive a DCS. Indeed, minimizing Equation \eqref{eq:offpolicy-bound-tight} as a function of $\bm{m}$ entails computing the Rényi divergence over the trajectory space. This, in turn, requires both to compute the approximation of a complex integral, and, for stochastic environments, the knowledge of the transition kernel $P$ of the underlying MDP \citep{metelli2018policy}. Therefore, to derive a tractable expression, we further bound each term as $d_2(p(\cdot|\bm{\bar{\theta}},h) \| p(\cdot|\bm{\theta},h))\le d_2(p(\cdot|\bm{\bar{\theta}},T) \| p(\cdot|\bm{\theta},T))$, thus leading to:
\begin{align}\label{eq:offpolicy-bound-loose}
\sqrt{ \beta_{\delta} d_2(p(\cdot|\bm{\bar{\theta}},T) \| p(\cdot|\bm{\theta},T)) \sum_{t=0}^{T-1} \frac{c_t}{n_t}  }.
\end{align}
However, it is easy to verify that, finding the DCS that minimizes this new expression, leads to an optimization problem with the same structure of Problem \eqref{sys:onpolicy-hard}. Indeed, $2\log(2/\delta)$ and $\beta_{\delta} d_2(p(\cdot|\bm{\bar{\theta}},T) \| p(\cdot|\bm{\theta},T))$ can be seen as constants that do not impact the result of the optimization. For this reason, it is possible to derive an equivalent of Theorem \ref{theo:onpolicy-sol} for the off-policy setting that we defer to Appendix \ref{app:proofs}. Notice, however, that the form of the approximately optimal DCS $\tilde{\bm{m}}^*$ is left unchanged, and, consequently, all the previous comments about the on-policy solution extends to the off-policy setting as well. To conclude, we remark that, although a further bound has been applied to obtain Equation \eqref{eq:offpolicy-bound-loose}, once the data has been collected, one can still make use of the tighter bound of Equation \eqref{eq:offpolicy-bound-tight} to obtain a confidence interval on $J(\bm{\bar{\theta}})$. Indeed, since $d_2(p(\cdot|\bm{\bar{\theta}},h) \| p(\cdot|\bm{\theta},h))\le d_2(p(\cdot|\bm{\bar{\theta}},T) \| p(\cdot|\bm{\theta},T))$, this implies a further source of improvement w.r.t. the uniform strategy.

\subsection{Discussion}
We now discuss the choice of our confidence interval metrics to optimize the DCS. As we have seen, our method provably minimizes confidence intervals on the expected discounted return of a given policy. More specifically, the choice of the confidence intervals that we adopt, together with the methodology that we present, leads to a novel \emph{fixed} DCS (i.e., $\tilde{\bm{m}}^*$) that can be adopted for estimation purposes. Minimizing confidence intervals is well-known to be a robust solution against heavy-tailed distributions \citep{lugosi2019mean}, and, consequently, our work comes with desired statistical properties that hold for \emph{any} possible pair of MDP and target policy. 
Moreover, given that the proposed DCS is pre-determined, it nicely fits situations where the agent collects its experience (i.e., spends $\Lambda$) in parallel over a cluster of machines, which is a typical scenario for policy gradient methods. 
At this point, one might object that there might exist MDPs in which, intuitively, truncating trajectories is a sub-optimal solution. In particular, suppose that the agent gathers rewards different from $0$ only in the last interaction step (e.g., a \emph{goal-based} problem). In this situation, we can imagine that the uniform strategy should be preferred over any other allocation strategy, even in a discounted setting.\footnote{We propose a variance analysis for these scenarios in Appendix \ref{app:proofs}.}
Our approach does not capture this \emph{problem-dependent} feature since it is designed to be agnostic w.r.t. the underlying structure of the MDP and target policy. However, we remark that,  without any sort of prior knowledge, our method provably minimizes the \emph{worst-case} scenario.
Furthermore, we also notice that when dealing with sparse rewards, $\gamma$ is usually selected to be close to $1$ to avoid nullifying the positive reward gathered at the end of the trajectory. In such a scenario, $\bm{\tilde{m}}^*$ tends to the uniform strategy.

\section{Truncating Trajectories in Policy Optimization via Importance Sampling}\label{sec:optim}
In this section, we discuss how to use our approximately optimal DCS $\tilde{\bm{m}}^*$ in a policy optimization algorithm. In particular, given the result from Section \ref{sec:eval}, Policy Optimization via Importance Sampling \citep[\pois,][]{metelli2018policy}, as we shall see in a moment, turns out to be an natural choice. \pois is a recent off-policy optimization algorithm that alternates \emph{online} interactions with the environment (i.e., data collection) with \emph{offline} optimization. In particular, \pois first makes use of the uniform DCS to collect a batch of $K$ episodes of length $T$ under the current policy $\pi_{\bm{\theta}}$. Then, it searches, by gradient steps, for the next policy $\pi_{\bm{\bar{\theta}}}$ that maximizes an empirical version of the statistical surrogate for the off-policy return derived from Equation \eqref{eq:offpolicy-bound-tight}. Namely, the agent optimizes for:
\begin{align*}
\hat{J}(\bm{\bar{\theta}}/\bm{\theta})-\sqrt{\beta_{\delta} \hat{d}_2(p(\cdot | \bm{\bar{\theta}},T ) | p(\cdot| \bm{\theta},T)) \left(\sum_{t=0}^{T-1} {\gamma^t}\right)^2 \frac{T}{\Lambda}}  ,
\end{align*}
where $\hat{d}_2(p(\cdot | \bm{\bar{\theta}},T ) | p(\cdot| \bm{\theta},T))$ is a sampled-based estimation of the Rényi divergence ${d}_2(p(\cdot | \bm{\bar{\theta}},T ) | p(\cdot| \bm{\theta},T))$ (see Equation $41$ in \citet{metelli2018policy}). In other words, POIS limits the update step via an adaptive trust region defined by the confidence intervals on the estimation of $J(\bm{\bar{\theta}})$ given that data have been collected using a different policy $\pi_{\bm{\theta}}$.

In this work, we build on POIS, and we propose to employ our optimized DCS $\bm{\tilde{m}}^*$, together with the corresponding estimator, to build a tighter surrogate of the off-policy return. More specifically, from Equation \eqref{eq:offpolicy-bound-tight}, we define $\mathcal{L}_{\delta}(\bm{\bar{\theta}}/\bm{\theta})$ as our empirical objective function:
\begin{equation}\label{eq:tt-pois-obj}
\begin{aligned}
& \mathcal{L}_{\delta}(\bm{\bar{\theta}}/\bm{\theta})  \coloneqq \hat{J}_{\bm{\tilde{m}}^*}(\bm{\bar{\theta}}/\bm{\theta}) \\
& \;\;\;\;-\sqrt{\beta_{\delta} \sum_{h=1}^T \tilde{m}_h^* (\tilde{\phi}_h^*)^2 \hat{d}_2(p(\cdot | \bm{\bar{\theta}},h ) | p(\cdot| \bm{\theta},h) )},
\end{aligned}
\end{equation}
where $\tilde{\phi}_h^* = \sum_{t=0}^{h-1} \frac{\gamma^t}{\tilde{n}_t^*}$ and $\hat{d}_2(p(\cdot | \bm{\bar{\theta}},h ) | p(\cdot| \bm{\theta},h))$ is a sampled-based estimation for ${d}_2(p(\cdot | \bm{\bar{\theta}},h) | p(\cdot| \bm{\theta},h))$. Notice that \ttpois, in Equation \eqref{eq:tt-pois-obj}, makes use of the \emph{tighter} bound of Equation \eqref{eq:offpolicy-bound-tight} with the approximately optimal DCS derived while optimizing Equation \eqref{eq:offpolicy-bound-loose}. This choice is justified by the discussion at the end of Section \ref{sec:off-policy-dcs}.


In the following, we will refer to the algorithm using Equation \eqref{eq:tt-pois-obj} as objective function as Truncating Trajectories in Policy Optimization via Importance Sampling (TT-POIS).
The pseudo-code, together with other practical implementation details can be found in Appendix \ref{app:pois-details}.


We conclude by highlighting that \ttpois can make better use of the collected data from a statistical perspective (i.e., smaller confidence intervals), suggesting that the surrogate loss will be closer to the true return  $J(\bm{\bar{\theta}})$. This implies that the adaptive trust region over the parameter space defined by Equation \eqref{eq:tt-pois-obj} will allow for larger update steps.

\section{Experiments}\label{sec:exp}

\begin{figure*}[t!]
\centering\includegraphics[width=13cm]{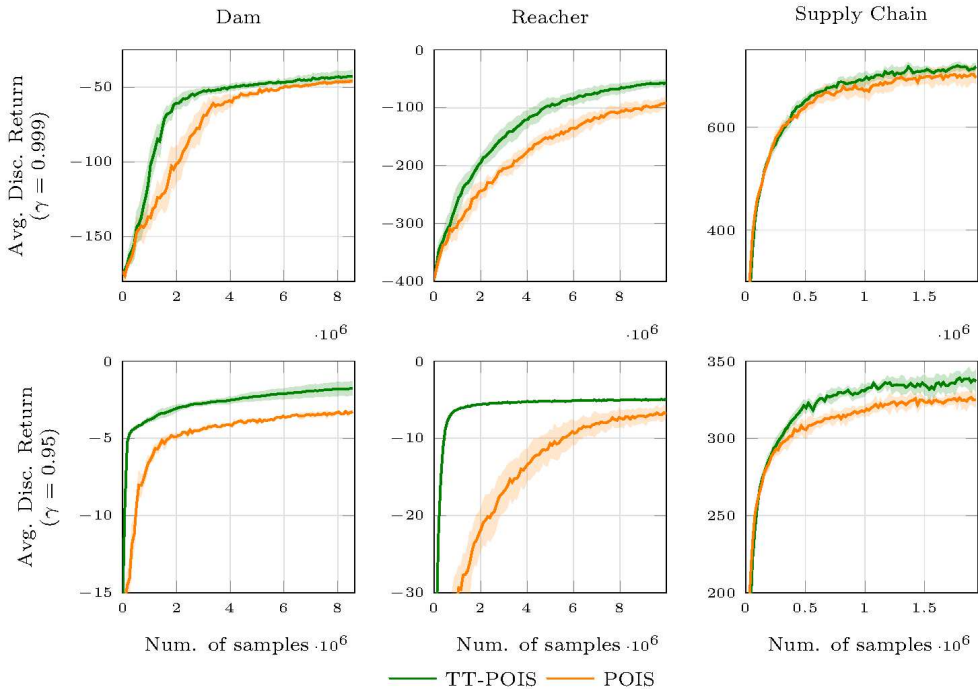} 
\caption{Experimental results (mean and $95\%$ confidence intervals of $5$ runs). The first row (resp. second row) reports average returns with $\gamma = 0.999$ (resp. $\gamma=0.95$).} 
\label{fig:mainresults}
\end{figure*}

In this section, we numerically validate our approach by targeting the comparison between \pois and \ttpois across multiple domains and varying both the discount factor and the available budget. In all experiments, we first tuned the hyper-parameters with \pois, and then, we applied \ttpois to the best hyper-parameter configuration of \pois. We now present our experimental domains, followed by a discussion of the results. Further details and additional experiments are deferred to Appendix \ref{app:exp}.


\paragraph{Dam Control}
In our first experimental domain, we consider a water resource management scenario \citep{castelletti2010tree,parisi2014policy,tirinzoni2018importance,liotet2022lifelong}. The goal of the agent is to learn a water release policy that trades off between some external demand $D$ (e.g., the needs of a town) and keeping the water level below a flooding threshold $F$. The dam is subject to an external and stochastic net inflow, that, each day,  determines the amount of additional water $i_t$ that will be stored (e.g., rain). More specifically, this inflow profile has a periodic shape defined over a period of one year; the demand, instead, is kept constant. The state of the system $s_t$ evolves according to a simple mass balance principle, namely $s_{t+1} = \max\{s_t - a_t + i_t, 0\}$, where $a_t$ is the amount of water that the agent intends to release at day $t$. The reward $R(s_t, a_t)$ is a convex combination of the two aforementioned objectives: $-c_1 \max\{0, s_t - F \}$ (i.e., flooding control) and $-c_2 \max\{0, D - a_t \}^2$ (i.e., meeting the demand), where $c_1, c_2 > 0$ are domain-dependent constants. 

\paragraph{Reacher}
In the second experiment, we consider the standard continuous control problem of a two-jointed robot arm \citep{mujoco}, whose goal is to move the robot's end effector close to a target spawned in a random position. The reward is a combination of a control cost together with a penalization for the end effector being far from the goal. 

\paragraph{Multi-Echelon Supply Chain}
Finally, we consider the problem of managing the complex inventory of a $4$ stage supply chain \citep{hubbs2020or}. During each day, the agent needs to decide, for each stage, how many products it should order from its supplier (i.e., the previous stage). Once the goods have been ordered at stage $i$, it takes a given amount of days $t_i$ (i.e., lead time of stage $i$) so that they are shipped and delivered to stage $i+1$. Goods at the last stage are sold according to some stochastic demand. If the retailer fails to meet the demand, lost orders are backlogged, meaning they are fulfilled later but for lower profit. The agent's key challenge is trading off the uncertainty of the demand (i.e., profit) with the costs incurred for storing products in the inventory. A mathematical description of the problem can be found in \citet{hubbs2020or}.

\paragraph{Results}
Figure~\ref{fig:mainresults} reports the average discounted return on the considered domains (mean and $95\%$ confidence intervals of $5$ runs) varying the discount factor $\gamma$. More specifically, the first row has been obtained with $\gamma=0.999$, while the second one with $\gamma=0.95$ (experiments with additional values of $\gamma$ can be found in Appendix \ref{app:additional-opt-results}). The considered budget per iteration is $\Lambda=8640$ for the Dam environment, $\Lambda=3900$ for the Supply Chain, and $\Lambda=8000$ for the Reacher. As we can notice, independently of the value of $\gamma$, \ttpois always performs better w.r.t. its original version \pois. It is worth noting what happens to the training curves when we change the value of $\gamma$. In these scenarios, as previously discussed, the dissimilarity between our non-uniform DCS $\tilde{\bm{m}}^*$ and the uniform one increases. Indeed, $\tilde{\bm{m}}^*$ will allocate a larger portion of the budget $\Lambda$ to the initial interaction steps. This, in turn, implies a larger difference in the confidence intervals of the surrogate objective function. Consequently, as soon as we decrease $\gamma$, we observe a larger performance improvement in each of the domains. This is consistent with our theory.  Due to the tighter confidence bounds, the agent is able to make better use of the collected data and takes larger update steps.



For what concerns experiments in which the budget $\Lambda$ changes, we report the results in Appendix \ref{app:additional-opt-results} for space constraints. 
However, we highlight that the behavior of the algorithms does not display significant differences. \ttpois always performs better than \pois, and what has been previously highlighted for Figure \ref{fig:mainresults} replicates consistently. 


\section{Conclusions and Future Works}\label{sec:disc}
In this work, we focused on how to allocate the interaction budget $\Lambda$ in Monte Carlo Reinforcement Learning. We started by building on the intuition that the common uniform-in-the-horizon strategy might not be the best option when discounted rewards are considered. To study the problem from a theoretical perspective, we introduced the novel concept of Data Collection Strategy (DCS), and we investigated alternative non-uniform solutions from the worst-case robust viewpoint of confidence intervals. More specifically, starting from the on-policy evaluation problem, we showed that, to minimize confidence intervals around the estimated expected return of a policy, non-uniform DCSs, which we provide in closed-form, represent a more appropriate solution, thus confirming our initial intuition. After theoretically analyzing the benefit of the proposed DCS from a PAC perspective, we further extended our reasoning to off-policy evaluation problems by generalizing the confidence bounds of \citet{metelli2018policy}, that are directly employed in the surrogate loss function that a recent algorithm, i.e., \pois \citep{metelli2018policy}, optimizes for. We then proposed an extension of \pois, \ttpois, that makes use of our optimized budget allocation, and we verified that it leads to performance improvements among multiple domains, and for different values of $\gamma$ and $\Lambda$.

We conclude by remarking that our work roots down to a main component of RL algorithms; i.e., the interaction with the environment. More specifically, we showed that it is possible to find principled strategies that optimize the \emph{collection} of the experience with the domain at hand. Our work, in this sense, does not close the problem but takes a first step toward this direction, thus paving the way for several exciting future works.

For example, while in this work we took a robust approach and derived a fixed data collection strategy that minimizes confidence intervals around the target return, other choices are also possible. In this sense, a complementary direction w.r.t. to the one we followed would be to find a dynamic DCS that performs \emph{online} minimization of the MSE of some estimator while interacting with the environment. This could be possible by integrating our approach and analysis with confidence intervals that rely on empirical quantities \citep[e.g.,][]{maurer2009empirical}, and, consequently, by designing strategies that aim at minimizing the MSE in an online fashion. Moreover, more effective strategies could be derived when restricting to specific subclasses of problems (e.g., goal-based) and leveraging their problem-dependent features.


Finally, we notice that our approach is based on Monte Carlo data collection and, therefore, does not deeply exploit the Markovian properties of the underlying MDP. Empowering the proposed methods with TD approaches that take into consideration value functions available in actor-critic algorithms \citep[e.g.,][]{schulman2017proximal} may lead to further improvements in the performance of policy-search algorithms.




\bibliography{biblio}
\bibliographystyle{icml2023}

\newpage
\appendix
\onecolumn

\section{Related Works}\label{app:rel-works}
Before diving into the details of the proofs, we provide a more in-depth discussion of previous works that are linked to ours.

In recent years, there has been a tremendous amount of interest in developing and improving policy search algorithms \citep[e.g.,][]{williams1992simple,baxter2001infinite,lillicrap2015continuous,schulman2015trust,schulman2017proximal,metelli2018policy,cobbe2021phasic}. Most of these methods interleave the two following phases during the training process: first, a batch of trajectories of fixed length is collected by interacting with the environment via Monte Carlo simulation, and, second, these data are used to update the parameters of the policy. Motivated by this setting, we study whether it is possible to optimize Monte Carlo simulations, by exploiting the reset possibility available in many real-world simulators.


To tackle this budget allocation problem, we started with analyzing the problem of estimating, with finite budget $\Lambda$, the performance of a given policy via Monte Carlo simulation \cite{kalos2009monte,owen2013monte}. We started with the on-policy setting by minimizing generalizations of the H\"{o}effding confidence intervals \citep{boucheron2003concentration} around the estimated return. Then, we extended the reasoning to the more intricated off-policy problem, through the lens of importance sampling \citep{hesterberg1988advances,owen2013monte}, and by building on top of the recent confidence intervals of \citet{metelli2018policy}. We remark that, in this sense, relevant works that are linked to ours can be found in, e.g., \citet{thomas2015high,thomas2016data,dai2020coindice,revar2022}. However, it has to be noticed that in previous studies, the focus was on different aspects of the estimation problem (e.g., building a policy that minimizes the variance of estimation for the return of a target policy), while we focus purely on the interaction within the environment. 

Given our analysis on the estimation problem, we propose to adopt our optimized data collection strategy, together with our optimized confidence intervals, to extend \pois \citep{metelli2018policy}, a recent off-policy optimization algorithm that relies on Monte Carlo simulation. More specifically, in \citet{metelli2018policy} the authors developed a principled off-policy method that directly models the uncertainty in the update step via controlling (upper bounds on) the variance of importance weights. In this sense, \pois defines an original concept of trust-region that constraints the target policy to be close to the behavioral one, thus limiting high-variance estimations problems that typically affect off-policy methods \citep{owen2013monte}. We notice that this concept of controlling the dissimilarity between the current policy and the next one is at the core of many well-known policy search methods \citep[e.g.,][]{schulman2015trust,schulman2017proximal}. Compared to this line of our work, we improve POIS by optimizing the dependence on the number of collected data in $\mathcal{L}(\bm{\bar{\theta}}/\bm{\theta})$. In this sense, we remark that, while we minimized confidence intervals around the empirical off-policy return (i.e., the POIS adaptive trust-region), other choices might better fit other algorithms: extending well-knowns algorithm such as TRPO and PPO with trajectory truncation mechanisms represents an exciting line for future works.

We also notice that several works have considered the problem of optimizing policy search methods by setting hyper-parameters, such as the learning rate \citep{pirotta2013adaptive}, the batch size \citep{papini2017adaptive}, or the amount of policy exploration \citep{papini2020balancing}, in a theoretically principled way. Compared to this line of work, in this paper, we are considering a \emph{novel type of hyper-parameter}, which is the budget that is spent by the agent to interact with the environment. To this end, we develop a principled theory that leads to collect truncated trajectories. This concept can be related to the recent strand of model-based policy optimization literature that simulates short trajectories in an estimated model of the considered domain to update the parameters of the policy \citep{janner2019trust,nguyen2018improving,bhatia2022adaptive}. Similar ideas on shortened/adaptive horizon have also arisen in the fields of multi-task reinforcement learning \citep{farahmand2016truncated} and imitation learning \citep{sun2018truncated}. However, in all these works, the motivation, the idea, the method, and the analysis completely differ. More specifically, we develop a theory that optimizes the interaction with the environment by exploiting the structure of the RL return. In this sense, our work is complementary to approaches that evolve the discount factor online to form a sort of curriculum \citep{franccois2015discount}. In particular, we notice that our method could be integrated into these approaches, by evolving the DCS based on the current value of the discount factor. Finally, most recently, \citet{poiani2022multifidelity} have introduced the idea of cutting trajectories while interacting with the environment to obtain a biased estimate of the return in planning algorithms such as depth-first search. This sort of idea was presented as an \emph{application} in the context of multi-fidelity bandits \citep[e.g.,][]{kandasamy2016multi,kandasamy2019multi}, where the crucial trade-off is between the introduced bias (i.e., the maximum error due to cutting the search at a given depth) and the cost of acquiring a given sample (i.e., the number of nodes generated by the algorithm). In this paper, we take a similar perspective, but we develop an approach that is tailored to Monte Carlo Reinforcement Learning settings.

\section{Proof and Derivations}\label{app:proofs}
In this Section, we provide proofs and derivations for all our theoretical claims. More specifically, Section \ref{app:proofs-on-pol} provides results for on-policy evaluation, Section \ref{app:proofs-off-pol} extends these results to the off-policy evaluation setting and Section \ref{app:proofs-further} provides further theoretical analysis.

\subsection{On-Policy Results}\label{app:proofs-on-pol}

We begin by proving that Equation \eqref{eq:onpolicy-dcs-est} is an unbiased estimate for $J(\bm{\theta})$.
\begin{theorem}\label{theo:onpolicy-unbias}
Consider an optimization budget $\Lambda \ge T$ and a DCS such that $m_T \ge 1$. Consider a policy $\pi_{\bm{\theta}} \in \Pi_{\Theta}$. Then:
\begin{align}
\mathop{\E}_{p_{\bm{m}}(\cdot|\bm{\theta})} \left[ \hat{J}_{\bm{m}}(\bm{\theta}) \right] = J(\bm{{\theta}}).
\end{align}
\end{theorem}
\begin{proof}
Let $r_{t,\bm{\theta}}$ be the expected $t$-th reward under policy $\pi_{\bm{\theta}}$. It is easy to verify that:
\begin{align*}
\mathop{\E}_{p_{\bm{m}}(\cdot|\bm{\theta})} \left[ \hat{J}_{\bm{m}}(\bm{\theta}) \right] & = \sum_{h=1}^T m_h \mathop{\E}_{p_{\bm{m}}(\cdot|\bm{\theta}, h)} \left[ \sum_{t=0}^{h-1} \gamma^t \frac{R(s_t, a_t)}{n_t} \right] = \sum_{h=1}^{T} m_h \sum_{t=0}^{h-1} \gamma^t \frac{r_{t,\bm{\theta}}}{n_t}
\end{align*}

At this point, unrolling the summation, we notice that, fixed $\bar{t} \in \{0, \dots, T-1 \}$ (i.e., the inner summation), its contribution appears in all $h \in \{1, \dots, T\}$ (i..e, outer summation) such that $h > \bar{t}$. Moreover, since $m_T \ge 1$, all $t \in \{0, \dots, T-1 \}$ appears at least once. Therefore:
\begin{align*}
\sum_{h=1}^T m_h \sum_{t=0}^{h-1} \gamma^t \frac{r_{t,\bm{{\theta}}}}{n_t} = \sum_{t=0}^{T-1} \gamma^t \frac{r_{t,\bm{{\theta}}}}{n_t} \sum_{h=t+1}^{T} m_h.
\end{align*}
However, given the relationship between $\bm{n}$ and $\bm{m}$, we have that:
\begin{align*}
\sum_{h=t+1}^{T} m_h = n_t - n_{t+1} + n_{t+1} - n_{t-2} + \dots + n_{T-2} - n_{T-1} + n_{T-1} = n_t.
\end{align*}
Therefore:
\begin{align*}
\sum_{t=0}^{T-1} \gamma^t \frac{r_{t,\bm{{\theta}}}}{n_t} \sum_{h=t+1}^{T} m_h = \sum_{t=0}^{T-1} \gamma^t r_{t,\bm{{\theta}}} = J(\bm{{\theta}}),
\end{align*}
which concludes the proof. 

\end{proof}

We now continue with a key technical Lemma that will be crucial to derive our generalizations of the H\"{o}effding confidence intervals.

\begin{lemma}\label{lemma:technical-lemma}
Consider an arbitrary DCS $\bm{m}$ such that $m_T \ge 1$. Then:
\begin{align}
\sum_{h=1}^{T} m_h \left( \sum_{t=0}^{h-1} \frac{\gamma_t}{n_t} \right)^2 = \sum_{t=0}^{T-1} \frac{c_t}{n_t}
\end{align}
where $c_t = \frac{\gamma^t (\gamma^t + \gamma^{t+1} - 2 \gamma^T)}{1-\gamma}$.
\end{lemma}
\begin{proof}
Consider:

\begin{align}\label{eq:tech-lemma-eq1}
\sum_{h=1}^{T} m_h \left( \sum_{t=0}^{h-1} \frac{\gamma^t}{n_t} \right)^2 = \sum_{h=1}^T m_h \left(\sum_{t=0}^{h-1} \frac{\gamma^{2t}}{n_t^2} + \sum_{t=0}^{h-2} \sum_{t'=t+1}^{h-1} \frac{2\gamma^{t+t'}}{n_t n_{t'}}  \right).
\end{align}

Then, focus on the first component, namely, $\sum_{h=1}^T m_h \sum_{t=0}^{h-1} \frac{\gamma^{2t}}{n_t^2}$. By unrolling the summations, we notice that, since $m_T \ge 1$, fixed $\bar{t} \in \{0,\dots,T-1\}$ each component $\frac{\gamma^{2t}}{n_t^2}$ appears for all $h$ such that $h>t$. Therefore:

\begin{align}\label{eq:tech-lemma-eq2}
\sum_{h=1}^T m_h \sum_{t=0}^{h-1} \frac{\gamma^{2t}}{n_t^2} = \sum_{t=0}^{T-1} \left( \frac{\gamma^{2t}}{n_t^2} \sum_{h=t+1}^{T} m_h \right) = \sum_{t=0}^{T-1} \frac{\gamma^{2t}}{n_t},
\end{align}

where in the last passage, we have used $\sum_{h=t+1}^{T} m_h = n_t$, which directly follow by the relationship between $\bm{n}$ and $\bm{m}$ and the fact that $m_T \ge 1$.

Then, consider the second part of Equation \eqref{eq:tech-lemma-eq1}, namely $\sum_{h=1}^T m_h \left( \sum_{t=0}^{h-2} \sum_{t'=t+1}^{h-1} \frac{2\gamma^{t+t'}}{n_t n_{t'}} \right)$. By unrolling the  summations, we notice that fixed the outer index of the inner summation, i.e., $t=\bar{t} \in \{0, \dots, T-2 \}$, its contribution will appear only for $h > t+1$, thus leading to:

\begin{align*}
\sum_{h=1}^T m_h \left( \sum_{t=0}^{h-2} \sum_{t'=t+1}^{h-1} \frac{2\gamma^{t+t'}}{n_t n_{t'}} \right) = \sum_{t=0}^{T-2} \sum_{h=t+2}^{T} m_h \sum_{t'=t+1}^{h-1} \frac{2\gamma^{t+t'}}{n_t n_{t'}}.
\end{align*}

At this point, fix $\bar{t} \in \{0, \dots, T-2 \}$ and consider $\sum_{h=\bar{t}+2}^{T} m_h \sum_{t'=\bar{t}+1}^{h-1} \frac{2\gamma^{\bar{t}+t'}}{n_{\bar{t}} n_{t'}}$. Unrolling the summation, we notice that a given term $t'$ appears only for $h > t'$. Therefore:
 
\begin{align}\label{eq:tech-lemma-eq3}
\sum_{h=\bar{t}+2}^{T} m_h \sum_{t'=\bar{t}+1}^{h-1} \frac{2\gamma^{\bar{t}+t'}}{n_{\bar{t}} n_t'} = \sum_{t'=\bar{t}+1}^{T-1} \frac{2\gamma^{\bar{t}+t'}}{n_{\bar{t}} n_{t'}} \left( \sum_{h=t'+1}^{T} m_h \right) = \sum_{t'=\bar{t}+1}^{T-1} \frac{2\gamma^{\bar{t}+t'}}{n_{\bar{t}}} .
\end{align}
 
Using Equations \eqref{eq:tech-lemma-eq2} and \eqref{eq:tech-lemma-eq3} in Equation \eqref{eq:tech-lemma-eq1}, we have that:

\begin{align*}
\sum_{h=1}^{T} m_h \left( \sum_{t=0}^{h-1} \frac{\gamma^t}{n_t}  \right)^2 & = \sum_{t=0}^{T-1} \frac{\gamma^{2t}}{n_t} + \sum_{t=0}^{T-2} \sum_{t'=t+1}^{T-1} \frac{2\gamma^{t'+t}}{n_t}  \\ & = \sum_{t=0}^{T-1} \frac{\gamma^{2t}}{n_t}  + \sum_{t=0}^{T-2} \frac{2\gamma^t}{n_t} \sum_{t'=t+1}^{T-1} \gamma^{t'} \\ & = \sum_{t=0}^{T-1} \frac{\gamma^{2t}}{n_t} + \sum_{t=0}^{T-2} \frac{2\gamma^t}{n_t}  \left( \frac{\gamma^{t+1}-\gamma^T}{1-\gamma} \right) \\ & = \sum_{t=0}^{T-1} \frac{\gamma^{2t}}{n_t} + \sum_{t=0}^{T-1} \frac{2\gamma^t}{n_t}  \left( \frac{\gamma^{t+1}-\gamma^T}{1-\gamma} \right) \\ & = \sum_{t=0}^{T-1} \frac{\gamma^t (\gamma^t + \gamma^{t+1} - 2\gamma^T)}{1-\gamma} \cdot \frac{1}{n_t} \\ & = \sum_{t=0}^{T-1} \frac{c_t}{n_t},
\end{align*} 
 
which concludes the proof.
\end{proof}

At this point, we are ready to prove Theorem \ref{prop:hoeffding-general}. We first report, for completeness, the H\"{o}effding's inequality for the sum of subgaussian random variables. 

\begin{lemma}\label{lemma:hoeffding}
Let $X_1, \dots, X_n$ be independent sub-gaussian r.v. with mean $\mu_1, \dots, \mu_n$ and subgaussianity parameters $\sigma_1^2, \dots, \sigma_n^2$, respectively. Let $\bar{\mu}_n \coloneqq \sum_{i=1}^n \mu_i$ and $\hat{\mu}_n \coloneqq \sum_{i=1}^n X_i$. Then, $\forall \epsilon > 0$, it holds that:

\begin{align*}
    \mathbb{P}(|\hat{\mu}_n - \bar{\mu}_n| > \epsilon) \le 2 \exp{\left( -\frac{\epsilon^2}{2 \sum_{i=1}^{n} \sigma_i^2} \right)}.
\end{align*}
\end{lemma}

\hoeffdinggeneral*
\begin{proof}
First of all, we notice that rewards are bounded in $[0,1]$. It follows that, given a trajectory of length $h$, $\sum_{t=0}^{h-1} \gamma^t \frac{R(s_t,a_t)}{n_t}$ is a subgaussian r.v. with subgaussianity parameter $\sigma^2_h = \frac{1}{4}\left( \sum_{t=0}^{h-1} \frac{\gamma^t}{n_t} \right)^2$.\footnote{We recall that a r.v. with bounded support over $[a, b]$ ($b > a$) is sub-gaussian with scale given by $\frac{(b-a)^2}{4}$} It follows that we can treat $\hat{J}_{\bm{m}}(\bm{\theta})$ as a sum of random variables with expected value $J({\bm{{\theta}}})$. Therefore, we can apply Lemma \ref{lemma:hoeffding} with $\epsilon = \sqrt{2 \sum_{h=1}^T m_h \sigma_h^2  \log(2/\delta)}$, obtaining that, with probability at least $1-\delta$:
\begin{align*}
|J_{\bm{m}}(\bm{\theta}) - J(\bm{\theta})| \le \sqrt{2 \sum_{h=1}^T m_h \sigma^2_h \log(2/\delta)} = \sqrt{\frac{1}{2} \sum_{h=1}^T m_h \left( \sum_{t=0}^{h-1} \frac{\gamma^t}{n_t} \right)^2 \log(2/\delta)} 
\end{align*}
The result, then, follow by combining the previous Equation with Lemma \ref{lemma:technical-lemma}.
\end{proof}

At this point, our focus shifts toward finding our approximately optimal DCS $\bm{\tilde{m}}^*$. We will derive our results for the general any-budget case, after which Theorem \ref{theo:onpolicy-sol} will follow as a special case. Our proofs follow by combining a closed form solution of a relaxation of \eqref{sys:onpolicy-hard}, where we drop the integer constraints on $n_t$, and some integrality gap arguments. More specifically, we are interested in the following relaxation of \eqref{sys:onpolicy-hard}:
\begin{equation}\label{sys:onpolicy-easy-orig}
\begin{aligned}
\min_{\bm{n}} \quad & \sqrt{\frac{1}{2} \log(2/\delta) \sum_{t=0}^{T-1} \frac{c_t}{n_t}}\\
\textrm{s.t.} \quad &  \sum_{t=0}^{T-1} n_t = \Lambda \\
  & n_t \ge n_{t+1}, \quad \forall t \in \{0, \dots, T-2\} \\
  & n_t \ge 1,  \quad \forall t \in \{0, \dots, T-1\}
\end{aligned}
\end{equation}
where the only difference stands in the fact that $n_t \in \mathbb{N}_+$ has now been replaced with $n_t \ge 1$. 
For this reason, we first present a simplified version of \eqref{sys:onpolicy-easy-orig}, that preserves the optimal solution. 

\begin{lemma}\label{lemma:convex-equivalence}
Consider an optimization $\Lambda \ge T$. The convex relaxation of the optimization problem \eqref{sys:onpolicy-hard} can be written as:
\begin{equation}\label{sys:onpolicy-easy}
\begin{aligned}
\min_{\bm{n}} \quad & \sum_{t=0}^{T-1} \frac{c_t}{n_t}\\
\textrm{s.t.} \quad &  \sum_{t=0}^{T-1} n_t = \Lambda \\
  & n_t \ge 1, \quad \forall t \in \{0, \dots, T-1\}  
\end{aligned}
\end{equation}
where $\bm{n} = (n_0, \dots, n_{T-1})$ and $c_t = \gamma^t (\gamma^t + \gamma^{t+1} - 2\gamma^T)$ Furthermore, the optimization problem \eqref{sys:onpolicy-easy} is convex in $\bm{n}$.
\end{lemma}
\begin{proof}
First, we prove the equivalence of the objective function. We notice that since the square root is a monotic function, it does not affect the optimal solution. Moreover, $\log(2/\delta)$ can be seen as a constant and, therefore, it can be neglected from the objective function as well. Thus, the optimal solution of the problem is preserved.

Then, we prove the equivalence of the constraints. The only difference between \eqref{sys:onpolicy-easy} and \eqref{sys:onpolicy-easy-orig} lies in the fact that 
$n_t \ge n_{t+1}$ has been neglected from the formulation. Given the structure of the simplified objective function (i.e., $\sum_{t=0}^{T-1} \frac{c_t}{n_t}$), we notice that $n_t \ge n_{t+1}$ will always be satisfied for an optimal solution. Indeed, suppose that $n_t < n_{t+1}$ for some $t$; then, since $c_{t} > c_{t+1}$, we can always improve the value of the objective function by swapping $n_t$ with $n_{t+1}$. Therefore, it is possible to neglect these constraints given that $n_t \ge 1$. 

Finally, we conclude by proving the convexity of \eqref{sys:onpolicy-easy}. First of all, we notice that all constraints are linear (and, thus, convex). It remains to prove the convexity of the objective function, i.e., $\sum_{t=0}^{T-1} \frac{c_t}{n_t}$. We begin by remarking that $\sum_{t=0}^{T-1} \frac{c_t}{n_t}$ is infinitely differentiable over the domain $\left( \mathbb{R} - \{0\} \right)^T$. In this case, to prove the convexity it is sufficient to ensure that the Hessian matrix is positive semidefinite. More specifically, in our case, the Hessian matrix is a diagonal matrix where the $i$-th element of the diagonal is given by $\frac{2c_i}{n_i^3}$, which, thus, is positive definite. This concludes the proof.
\end{proof}

At this point, we derive a closed-form solution for \eqref{sys:onpolicy-easy-orig} by analyzing the KKT condition of \eqref{sys:onpolicy-easy} \citep{boyd2004convex}.

\begin{lemma}\label{lemma:op}
Consider an optimization budget $\Lambda > T$, and consider $h \in \{1, \dots, T \}$. Let $n_t(h) = 1$ for $t \ge h$ and $n_t(h) = \frac{\sqrt{c_t}}{\sum_{i=0}^{h-1} \sqrt{c_i}}(\Lambda -T +h)$ for $t < h$. The optimal value of the objective function of the convex relaxation of \eqref{sys:onpolicy-hard} can be computed as:

\begin{align}\label{eq:sol-with-min}
\min_{h \in \{1, \dots, T\}} \sqrt{\frac{1}{2} \log(2/\delta) \sum_{t=0}^{T-1} \frac{c_t}{n_t(h)} }
\end{align}
and $h$ is such that, for all $t \ge h$ it holds that:
\begin{align}
\Lambda - T + h \le \frac{\sum_{i=0}^{h-1} \sqrt{c_i}}{\sqrt{c_t}},
\end{align}
and, for all $t < h$:
\begin{align}
\Lambda - T + h > \frac{\sum_{i=0}^{h-1} \sqrt{c_i}}{\sqrt{c_t}}.
\end{align}

\end{lemma}

\begin{proof}
Due to Lemma \ref{lemma:convex-equivalence}, we can study the optimal value of the convex relaxation by analyzing \eqref{sys:onpolicy-easy}. More specifically, we focus on the following variant: 
\begin{equation}\label{sys:equiv}
\begin{aligned}
\min_{\bm{n}} \quad & \sum_{t=0}^{T-1} \frac{c_t}{\tilde{n}_t + 1}  \\
\textrm{s.t.} \quad & \tilde{n}_t \ge 0 \; \; \; \forall t \in \{0, \dots, T-1 \}\\
  & \sum_{t=0}^{T-1} \tilde{n}_t \le \Lambda'\\
\end{aligned}
\end{equation}
where $\Lambda' = \Lambda - T$, and the fact that $n_t \ge 1$ has been directly forced in the objective function and in the constraints by applying the change of variables $n_t = \tilde{n}_t + 1$. 

At this point, the KKT conditions for the optimization problem \eqref{sys:equiv} are given by:

\begin{equation}\label{sys:kkt}
	\begin{cases}
			-\frac{c_t}{(\tilde{n}_t+1)^2} - \mu_t + \eta = 0 & \text{ } \forall t \in  \{0, \dots, T-1 \}\\
			\mu_t \tilde{n}_t = 0 & \text{ } \forall t \in \{0, \dots, T-1\} \\
			\eta (\sum_{t=0}^{T-1} \tilde{n}_t - \Lambda') = 0 \\
			\sum_{t=0}^{T-1} \tilde{n}_t - \Lambda' \le 0 \\
			\mu_t \ge 0 & \text{ } \forall t \in  \{0, \dots, T-1 \} \\
			\eta \ge 0
		\end{cases}.
\end{equation}

Since the problem is convex, the solution to the KKT conditions are the global optimum of the problem. To find it, we begin with the first equation for a general $t \in \{0, \dots, T-1\}$. From algebraic manipulations we obtain:

\begin{align}\label{eq:kkt-nt-1}
\tilde{n}_t = \sqrt{\frac{c_t}{\eta - \mu_t}} - 1.
\end{align}

At this point, we split our analysis into two cases that arise from the second equation of the system \eqref{sys:kkt}. More specifically, we notice that when $n_t > 0$, the second equation of the system \eqref{sys:kkt} leads to $\mu_t = 0$, \footnote{We notice that this satisfies the last constraint $\mu_t \ge 0$.} which reduces Equation \eqref{eq:kkt-nt-1} to:

\begin{align}\label{eq:kkt-nt-2}
\tilde{n}_t = \sqrt{\frac{c_t}{\eta}} - 1.
\end{align}

Therefore, since $c_t > 0$, this implies $\eta > 0$. At this point, since $\eta > 0$, the third equation  of the system \eqref{sys:kkt} leads to:

\begin{align}\label{eq:kkt-eta-1}
\sum_{t: \tilde{n}_t > 0} \tilde{n}_t = \sum_{t: \tilde{n}_t > 0} \left( \sqrt{\frac{c_t}{\eta}} - 1 \right) = \Lambda'.
\end{align}
However, due to the structure of the objective function, for $i, j \in \{0,\dots,T-1\}$ such that $i>j$, if $\tilde{n}_i > 0$, then $\tilde{n}_j > 0$. The main intuition is that, otherwise, we could set $\tilde{n}_j$ to the value of $\tilde{n}_i$, and $\tilde{n}_i$ to $0$, and improve the objective function.\footnote{This follows from the fact that $c_i < c_j$.}

 Therefore, we can write Equation \eqref{eq:kkt-eta-1} as a general function of an integer $h \in \{1, \dots, T\}$ that indicates the first time-step $t$ for which $n_h = 0$ holds.\footnote{We notice that, since $\Lambda' > 0$, we always have at least $\tilde{n}_0 > 0$.} More specifically, Equation \eqref{eq:kkt-eta-1} reduces to:

\begin{align}\label{eq:kkt-eta-2}
\sum_{t=0}^{h-1} \left( \sqrt{\frac{c_t}{\eta}} - 1 \right)  = \Lambda' .
\end{align}

Solving Equation \eqref{eq:kkt-eta-2} for $\eta$, we obtain:
\begin{align}\label{eq:kkt-eta-3}
\eta = \left(\frac{\sum_{i=0}^{h-1} \sqrt{c_i}}{\Lambda'+h} \right)^2,
\end{align}
which, as we can appreciate is always greater than 0, thus satisfying the constraints imposed so far.

At this point, using Equation \eqref{eq:kkt-eta-3} in Equation \eqref{eq:kkt-nt-2}, we obtain:
\begin{align}\label{eq:kkt-nt-3}
\tilde{n}_t = \frac{\sqrt{c_t}}{\sum_{i=0}^{h-1} \sqrt{c_i}} (\Lambda'+h) - 1,
\end{align}
which holds for a generic $t < h$, under the constraint that:
\begin{align*}
\frac{\sqrt{c_t}}{\sum_{i=0}^{h-1} \sqrt{c_i}} (\Lambda'+h) - 1 > 0.
\end{align*}
Or, equivalently:
\begin{align}\label{eq:kkt-first-c}
\Lambda' + h > \frac{\sum_{i=0}^{h-1} \sqrt{c_i}}{\sqrt{c_t}}.
\end{align}

Now, we consider the second case (i.e., $\tilde{n}_t = 0$) that arises from the second equation of the system \eqref{sys:kkt}. In this case, $\tilde{n}_t = 0$ and $\mu_t$ is possibly different from $0$, and we need to ensure that $\mu_t \ge 0$ to satisfy the last constraint of \eqref{sys:kkt}. More specifically, this reduces to study:

\begin{align*}
-c_t - \mu_t + \left(\frac{\sum_{i=0}^{h-1} \sqrt{c_i}}{\Lambda'+h} \right)^2 = 0,
\end{align*}

for $h \ge t$. In particular, we obtain:

\begin{align*}
\mu_t = \left(\frac{\sum_{i=0}^{h-1} \sqrt{c_i}}{\Lambda'+h} \right)^2 - c_t ,
\end{align*}

which we need to impose as greater or equal than $0$, that is:
\begin{align*}
\left(\frac{\sum_{i=0}^{h-1} \sqrt{c_i}}{\Lambda'+h} \right)^2 - c_t .
\end{align*}

Or, equivalently:
\begin{align}\label{eq:kkt-second-c}
\Lambda' + h \le \frac{\sum_{i=0}^{h-1} \sqrt{c_i}}{\sqrt{c_t}}.
\end{align}

At this point, we remark that there always at least exists one value of $h$ for which both Equations \eqref{eq:kkt-first-c} and \eqref{eq:kkt-second-c} are satisfied. Indeed, due to the Weierstrass theorem, if the objective function is continuous and considered on a closed and bounded domain, then a global optimum exists. Moreover, the KKT are sufficient conditions for global optimality in convex problems, from which follows the existence of at least one $h$ satisfying both equations.

Putting everything together, and rescaling $\tilde{n}_t$ to $n_t$ concludes the proof. 

\end{proof}

\begin{lemma}\label{lemma:convex-closed-form}
Consider an optimization budget $\Lambda > T$, and consider $h \in \{1, \dots, T \}$. Let $n_t(h) = 1$ for $t \ge h$ and $n_t(h) = \frac{\sqrt{c_t}}{\sum_{i=0}^{h-1} \sqrt{c_i}}(\Lambda -T +h)$ for $t < h$. The optimal value of the objective function of the optimization problem \eqref{sys:onpolicy-easy-orig} can be computed as:

\begin{align}\label{eq:offpolicy-obj-sol}
\sqrt{ \frac{1}{2} \log(2/\delta) \sum_{t=0}^{T-1} \frac{c_t}{n_t(h^*)} }
\end{align}
where $h^*$ is the only $h \in \{1, \dots, T\}$ for which the following holds: for all $t \ge h^*$:
\begin{align}\label{eq:c1}
\Lambda - T + h^* \le \frac{\sum_{i=0}^{h^*-1} \sqrt{c_i}}{\sqrt{c_t}}
\end{align}
and, for all $t < h^*$:
\begin{align}\label{eq:c2}
\Lambda - T + h^* > \frac{\sum_{i=0}^{h^*-1} \sqrt{c_i}}{\sqrt{c_t}}.
\end{align}
\end{lemma}
\begin{proof}
From Lemma \ref{lemma:op}, what remains to prove is that there exists a single $h \in \{1, \dots, T\}$ for which Equations \eqref{eq:c1} and \eqref{eq:c2} are satisfied.
Since at least one $h$ exists, we need to prove that it is impossible that Equations \eqref{eq:c1} and \eqref{eq:c2} are satisfied for two distinct $\bar{h}_1, \bar{h}_2 \in \{1, \dots, T\}$. Suppose w.l.o.g. that $\bar{h}_1 > \bar{h}_2$ and proceed by contradiction. 

First of all, consider a generic $h$ and focus on Equation \eqref{eq:c1}. A sufficient condition for Equation \eqref{eq:c1} to hold for all $t \ge h$ is given by:
\begin{align}\label{eq:c1-imply}
\Lambda - T + h \le \frac{\sum_{i=0}^{h-1} \sqrt{c_i}}{\sqrt{c_h}}.
\end{align}
Similarly, a sufficient condition for Equation \eqref{eq:c2} to hold for all $t < h$ is given by:
\begin{align}\label{eq:c2-imply}
\Lambda - T + h > \frac{\sum_{i=0}^{h-1} \sqrt{c_i}}{\sqrt{c_{h-1}}}.
\end{align}

Now, consider Equation \eqref{eq:c1-imply} for $\bar{h}_2$ and Equation \eqref{eq:c2-imply} for $\bar{h}_1$, namely:

\begin{equation}\label{sys:c1-c2-impossibility}
	\begin{cases}
			\Lambda - T + \bar{h}_2 \le \frac{\sum_{i=0}^{\bar{h}_2-1} \sqrt{c_i}}{\sqrt{c_{\bar{h}_2}}}\\
			\frac{\sum_{i=0}^{\bar{h}_1 -1} \sqrt{c_i}}{\sqrt{c_{\bar{h}_1-1}}}  < \Lambda - T + \bar{h}_1
		\end{cases}.
\end{equation}

Summing the two equations of System \eqref{sys:c1-c2-impossibility}, and rearranging the terms we obtain:
\begin{align*}
\bar{h}_1 - \bar{h}_2 > \frac{\sum_{i=0}^{\bar{h}_1 -1} \sqrt{c_i}}{\sqrt{c_{\bar{h}_1-1}}} - \frac{\sum_{i=0}^{\bar{h}_2-1} \sqrt{c_i}}{\sqrt{c_{\bar{h}_2}}}.
\end{align*}

However, as we shall show in a moment:

\begin{align}\label{eq:contradiction}
\bar{h}_1 - \bar{h}_2 \le \frac{\sum_{i=0}^{\bar{h}_1 -1} \sqrt{c_i}}{\sqrt{c_{\bar{h}_1-1}}} - \frac{\sum_{i=0}^{\bar{h}_2-1} \sqrt{c_i}}{\sqrt{c_{\bar{h}_2}}},
\end{align}

always holds, thus leading to a contradiction. Indeed, consider:

\begin{align*}
\frac{\sum_{i=0}^{\bar{h}_1 -1} \sqrt{c_i}}{\sqrt{c_{\bar{h}_1-1}}} - \frac{\sum_{i=0}^{\bar{h}_2-1} \sqrt{c_i}}{\sqrt{c_{\bar{h}_2}}} = \frac{\sum_{i= \bar{h}_2 }^{\bar{h}_1 -1} \sqrt{c_i}}{\sqrt{c_{\bar{h}_1-1}}} + \frac{\sum_{i=0}^{\bar{h}_2 -1} \sqrt{c_i}}{\sqrt{c_{\bar{h}_1-1}}} - \frac{\sum_{i=0}^{\bar{h}_2-1} \sqrt{c_i}}{\sqrt{c_{\bar{h}_2}}}.
\end{align*}
However, 
\begin{align*}
\frac{\sum_{i= \bar{h}_2 }^{\bar{h}_1 -1} \sqrt{c_i}}{\sqrt{c_{\bar{h}_1-1}}} \ge \frac{\sum_{i= \bar{h}_2 }^{\bar{h}_1 -1} \sqrt{c_{\bar{h}_1-1}}}{\sqrt{c_{\bar{h}_1-1}}}  = \bar{h}_1 - \bar{h}_2.
\end{align*}
Moreover:
\begin{align*}
\frac{\sum_{i=0}^{\bar{h}_2 -1} \sqrt{c_i}}{\sqrt{c_{\bar{h}_1-1}}} - \frac{\sum_{i=0}^{\bar{h}_2-1} \sqrt{c_i}}{\sqrt{c_{\bar{h}_2}}} \ge \frac{\sum_{i=0}^{\bar{h}_2 -1} \sqrt{c_i}}{\sqrt{c_{\bar{h}_2}}} - \frac{\sum_{i=0}^{\bar{h}_2-1} \sqrt{c_i}}{\sqrt{c_{\bar{h}_2}}} = 0 .
\end{align*} 
From which it follows Equation \eqref{eq:contradiction}, thus concluding the proof.

\end{proof}

Lemma \ref{lemma:convex-closed-form} deserves some comments. First of all, it provides an $\mathcal{O}(T)$ procedure to compute the optimal solution of the convex relaxation of \eqref{sys:onpolicy-hard}. Indeed, it is sufficient to iterate over the variable $h$ to find the only $h^*$ for which Equations \eqref{eq:c1} and \eqref{eq:c2} holds. 
In the rest of this text, we will refer to the optimal solution of the convex relaxation as $\bm{\bar{n}}^*$.
Secondly, we notice that Equations \eqref{eq:c1} and \eqref{eq:c2} put in a tight relationship $\bm{\bar{n}}^*$ with the available budget $\Lambda$. In particular, when the budget is sufficiently small (i.e., $\Lambda - T + h^* \le \frac{\sum_{i=0}^{h^*-1} \sqrt{c_i}}{\sqrt{c_t}}$), $\bm{\bar{n}}^*$ will necessarly allocate a single sample (i.e., $n_t(h^*) = 1$) for all $t \ge h^*$, which corresponds to the minimal amount of samples that is required to obtain an unbiased estimate. Conversely, when the budget is sufficiently large (i.e., $\Lambda - T + h^* > \frac{\sum_{i=0}^{h^*-1} \sqrt{c_i}}{\sqrt{c_t}}$), $\bm{\bar{n}}^*$ will allocate to all $t < h^*$ a non-uniform budget quantity that is given by $\frac{\sqrt{c_t}}{\sum_{i=0}^{h^*-1} \sqrt{c_i}}(\Lambda -T +h^*)$. 

At this point, we notice that, to study the solution in its closed form, we relaxed integer constraints. Given the relaxed solution of Lemma \ref{lemma:convex-closed-form}, it is easy to obtain a proper DCS by taking the element-wise floor of the optimal relaxed solution and allocating the remaining budget uniformly. The resulting DCS, which we refer to as the approximately optimal DCS $\bm{\tilde{n}}^*$, will thus differ at most by $1$ (element-wise) w.r.t. to the optimal relaxed solution of Lemma \ref{lemma:convex-closed-form}. More formally, we define the any-budget $\bm{\tilde{n}}^*$ in the followin way.

\begin{definition}\label{def:approx-opt-dcs}
Consider an optimization budget $\Lambda > T$ and consider $h \in \{1, \dots, T \}$. Let $n_t(h) = 1$ for $t \ge h$ and $n_t(h) = \frac{\sqrt{c_t}}{\sum_{i=0}^{h-1} \sqrt{c_i}}(\Lambda -T +h)$ for $t < h$. Let $h^*$ be the only $h$ such that Equation \eqref{eq:c1} and Equation \eqref{eq:c2} are satisfied. Let $k = \Lambda - \sum_{t=0}^{T-1} \lfloor n_t(h^*) \rfloor$. Then, we define the approximately optimal DCS $\bm{\tilde{n}}^* = (\tilde{n}^*_0, \dots, \tilde{n}^*_{T-1})$, where:
\begin{align}
\tilde{n}_t^* = \lfloor n_t(h^*) \rfloor + \mathbf{1} \{ t < k \}
\end{align}
\end{definition}

Notice that Definition \ref{def:approx-opt-dcs} reduces to the one of Theorem \ref{theo:onpolicy-sol} for sufficiently large budget. More specifically, taking $\Lambda \ge \Lambda_0 = \frac{\sum_{t=0}^{T-1} \sqrt{c_t}}{\sqrt{c_{T-1}}}$, then $h^* = T$, from which follows the expression of Theorem \ref{theo:onpolicy-sol}.

\begin{lemma}\label{lemma:diff-bound}
Consider an optimization budget $\Lambda > T$, let $\bm{\bar{n}}$ be the optimal solution of optimal solution of the convex relaxation given in Lemma \ref{lemma:convex-closed-form}, and let $\bm{\tilde{n}}^*$ as in Defintion \ref{def:approx-opt-dcs}. Then, for each $t \in \{0, \dots, T-1 \}$, $|\bar{n}_t - \tilde{n}_t^*| \le 1$.
\end{lemma}
\begin{proof}
Consider $t$ such that $t \ge k$ holds. Then $|\bar{n}_t - \tilde{n}^*_t| = |n_t(h^*) - \lfloor n_t(h^*) \rfloor| \le 1$. 

Consider $t$ such that $t < k$ holds. Then $|\bar{n}_t - \tilde{n}^*_t| = |n_t(h^*) - \lfloor n_t(h^*) \rfloor + 1| \le 1$, which concludes the proof.
\end{proof}

At this point, what is left is analyzing the quality of the approximately optimal DCS $\bm{\tilde{n}}^*$, which will to the proof of Theorem \ref{theo:onpolicy-sol}. 
We begin by reporting the equivalent version of Theorem \ref{theo:onpolicy-sol} that holds for the generic case of $\Lambda > T$.\footnote{Notice that for $\Lambda=T$ there exists only one DCS that satisfies $m_T \ge 1$, and, consequently, the problem is trivial.}

\begin{theorem}\label{theo:relaxation-remark}
Consider an optimization budget $\Lambda > T$, let $\bm{\tilde{n}}^*$ be the approximately optimal DCS given in Defintion \ref{def:approx-opt-dcs}, and let $\bm{n}^*$ be the optimal solution of the integer optimization problem \eqref{sys:onpolicy-hard}. Moreover, let $f(\bm{n}) = \sqrt{\frac{1}{2} \log(2/\delta) \sum_{t=0}^{T-1} \frac{c_t}{n_t}}$. Then,
\begin{align}
f(\bm{n}^*) \le f(\bm{\tilde{n}}^*) \le \sqrt{2} f(\bm{n}^*)
\end{align}
\end{theorem}
\begin{proof}
First of all, focus $1 \le \frac{f(\bm{\tilde{n}}^*)}{f(\bm{n}^*)}$. This clearly holds since $\bm{n}^*$ is the optimal solution of \eqref{sys:onpolicy-hard}, while $\bm{\tilde{n}}^*$ is a feasibile solution.

Now, what remains to prove is that $\frac{f(\bm{\tilde{n}}^*)}{f(\bm{n}^*)} \le \sqrt{2}$. Let $\bm{\bar{n}}^*$ be the optimal solution of the convex relaxation given in Lemma \ref{lemma:convex-closed-form}. Then, first of all, we notice that:
\begin{align}\label{eq:robust-eq1}
\frac{f(\bm{\tilde{n}}^*)}{f(\bm{n}^*)} \le \frac{f(\bm{\tilde{n}}^*)}{f(\bm{\bar{n}}^*)}
\end{align}
holds since $\bm{\bar{n}^*}$ is the optimal solution of the same optimization problem but with a removed constraint (i.e., the integer constraint on $n_t$). 
Then, consider:
\begin{align}\label{eq:robust-eq2}
f(\bm{\bar{n}}^*) = \sqrt{\frac{1}{2} \log(2/\delta) \sum_{t=0}^{T-1} \frac{c_t}{\bar{n}^*_t}} \ge \sqrt{\frac{1}{2} \log(2/\delta) \sum_{t=0}^{T-1} \frac{c_t}{\tilde{n}^*_t + 1}} \ge \sqrt{\frac{1}{2} \log(2/\delta) \sum_{t=0}^{T-1} \frac{c_t}{2\tilde{n}^*_t}} = \sqrt{\frac{1}{2}} f(\bm{\tilde{n}}^*)
\end{align}
where the first inequality follows from Lemma \ref{lemma:diff-bound} and the second one by $\tilde{n}^*_t \ge 1$.
Plugging Equation \eqref{eq:robust-eq2} into Equation \eqref{eq:robust-eq1} concludes the proof.
\end{proof}

\sol*
\begin{proof}
The proof is a direct consequence of Theorem \ref{theo:relaxation-remark}.
\end{proof}



We now continue by providing the PAC analysis for our approximately optimal DCS $\bm{\tilde{n}}^*$. Before diving into the proof of Theorem \ref{theo:pac}, we provide two intermediate technical results.
\begin{lemma}\label{lemma:pac-lemma-tech-1}
Consider an optimization budget $\Lambda\ge2T$ and let $\bm{\tilde{n}}^*$ be as in Definition \ref{def:approx-opt-dcs}. Then:
\begin{align}
\sqrt{\frac{1}{2} \log(2/\delta) \sum_{t=0}^{T-1} \frac{c_t}{\tilde{n}^*_t}} \le \sqrt{2 \frac{\log(2/\delta)}{\Lambda} \left(\sum_{t=0}^{T-1} \sqrt{c_t}\right)^2 }
\end{align}
\end{lemma}
\begin{proof}
Let $\bm{\bar{n}}^*$ be the optimal solution of the convex relaxation given in Lemma \ref{lemma:convex-closed-form}. Then, as in Theorem \ref{theo:relaxation-remark}, we have that:
\begin{align*}
\sqrt{\frac{1}{2} \log(2/\delta) \sum_{t=0}^{T-1} \frac{c_t}{\tilde{n}^*_t}} \le \sqrt{\log(2/\delta) \sum_{t=0}^{T-1} \frac{c_t}{\bar{n}^*_t}}
\end{align*}
Now, plugging in the definition of $\bm{\bar{n}}^*$, we have that:
\begin{align}\label{eq:tech-lemma-pac-eq1}
\sqrt{\log(2/\delta) \sum_{t=0}^{T-1} \frac{c_t}{\bar{n}^*_t}} & = \sqrt{\log(2/\delta) \sum_{t=0}^{h^*-1} \frac{c_t}{\bar{n}^*_t} + \log(2/\delta) \sum_{t=h^*}^{T-1} c_t} \\ & = \sqrt{\log(2/\delta) \sum_{t=0}^{h^*-1} \frac{c_t}{\sqrt{c_t} (\Lambda -T + h^*)} \sum_{i=0}^{h^*-1} \sqrt{c_i} + \log(2/\delta) \sum_{t=h^*}^{T-1} c_t}
\end{align}
where $h^*$ is the only $h$ that satisfies Equation \eqref{eq:c1} and \eqref{eq:c2}. 

Now, since $\Lambda \ge 2T$, we have that:
\begin{align}\label{eq:tech-lemma-pac-eq3}
\sum_{t=0}^{h^*-1}\frac{c_t}{\sqrt{c_t} (\Lambda -T + h^*)} \sum_{i=0}^{h^*-1} \sqrt{c_i} & \le \sum_{t=0}^{h^*-1}\frac{c_t}{\sqrt{c_t} (\Lambda - 1/2 \Lambda)} \sum_{i=0}^{h^*-1} \sqrt{c_i} \\ & \le  2 \sum_{t=0}^{h^*-1}\frac{c_t}{\sqrt{c_t} \Lambda} \sum_{i=0}^{h^*-1} \sqrt{c_i} \\ & \le \frac{2}{\Lambda} \sum_{t=0}^{h^*-1} \sqrt{c_t} \sum_{i=0}^{T-1} \sqrt{c_i}
\end{align}
which, plugged into Equation \eqref{eq:tech-lemma-pac-eq1}, leads to:
\begin{align}\label{eq:tech-lemma-pac-eq2}
\sqrt{\frac{1}{2} \log(2/\delta) \sum_{t=0}^{T-1} \frac{c_t}{\tilde{n}^*_t}} \le \sqrt{\log(2/\delta) \frac{2}{\Lambda} \sum_{t=0}^{h^*-1} \sqrt{c_t} \sum_{i=0}^{T-1} \sqrt{c_i} + \log(2/\delta) \sum_{t=h^*}^{T-1} c_t}
\end{align}
At this point, focus on $\sum_{t=h^*}^{T-1}c_t$. From Equation \eqref{eq:c1}, we know that for $h^*$ it holds that:
\begin{align*}
\frac{(\Lambda -T +h^*) \sqrt{c_t}}{\sum_{i=0}^{h^*-1} \sqrt{c_i}} \le 1 
\end{align*}
Therefore,
\begin{align}\label{eq:tech-lemma-pac-eq4}
\sum_{t=h^*}^{T-1}c_t = \sum_{t=h^*}^{T-1} \frac{c_t}{1} \le \sum_{t=h^*}^{T-1} \frac{c_t}{\sqrt{c_t} (\Lambda-T+h^*)} \sum_{i=0}^{h^*-1} \sqrt{c_i} \le \frac{2}{\Lambda}\sum_{t=h^*}^{T-1} \sqrt{c_t} \sum_{i=0}^{T-1} \sqrt{c_i}
\end{align}
where in the last inequality we have used the same arguments of Equation \eqref{eq:tech-lemma-pac-eq3}.

At this point, plugging Equation \eqref{eq:tech-lemma-pac-eq4} into Equation \eqref{eq:tech-lemma-pac-eq2}, leads to:
\begin{align*}
\sqrt{\frac{1}{2} \log(2/\delta) \sum_{t=0}^{T-1} \frac{c_t}{\tilde{n}^*_t}} & \le \sqrt{\log(2/\delta) \frac{2}{\Lambda} \sum_{t=0}^{h^*-1} \sqrt{c_t} \sum_{i=0}^{T-1} \sqrt{c_i} + \log(2/\delta) \frac{2}{\Lambda}\sum_{t=h^*}^{T-1} \sqrt{c_t} \sum_{i=0}^{T-1} \sqrt{c_i}} \\ & = \sqrt{ \frac{2 \log(2/\delta)}{\Lambda} \left(\sum_{t=0}^{T-1} \sqrt{c_t} \right)^2}
\end{align*}
which concludes the proof.
\end{proof}

\begin{lemma}\label{lemma:pac-lemma-tech-2}
Consider $\delta \in (0,1)$ and $\epsilon > 0$ such that $\log(2/\delta)c_0 \ge 8T\epsilon^2$ holds. Then $\Lambda \ge 2T$ is a necessary condition to guarantee that $|\hat{J}_{\bm{\tilde{m}}^*}(\bm{\theta}) - J(\bm{\theta})| \le \epsilon$ holds.
\end{lemma}
\begin{proof}
We proceed by contradiction: suppose that $\Lambda < 2T$. Then, we continue by lower-bounding the value of the confidence intervals when using our approximately optimal DCS $\bm{\tilde{m}}^*$. More specifically, let $\bm{\bar{n}}^*$ be the optimal solution of the convex relaxation of \eqref{sys:onpolicy-hard}. Then, we have that:
\begin{align*}
\sqrt{\frac{1}{2} \log(2/\delta) \sum_{t=0}^{T-1} \frac{c_t}{\tilde{n}^*_t}} \ge \sqrt{\frac{1}{2} \log(2/\delta)  \frac{c_0}{\tilde{n}^*_0}} \ge \sqrt{\frac{1}{2} \log(2/\delta)  \frac{c_0}{\bar{n}^*_0 + 1}} \ge \sqrt{\frac{1}{4} \log(2/\delta)  \frac{c_0}{\bar{n}^*_0}}
\end{align*}
where in the first inequality we have removed positive terms, in the second one we have used Lemma \ref{lemma:diff-bound} and in the third one we have used $\bar{n}^*_0 \ge 1$. At this point, by noticing that $\bar{n}^*_0 < \Lambda$, we obtain:
\begin{align*}
\sqrt{\frac{1}{2} \log(2/\delta) \sum_{t=0}^{T-1} \frac{c_t}{\tilde{n}^*_t}} & \ge \sqrt{\frac{1}{4} \log(2/\delta)  \frac{c_0}{\Lambda} } 
\end{align*}

We can now focus on:
\begin{align*}
\sqrt{\frac{1}{4} \log(2/\delta)  \frac{{c_0}}{\Lambda}} \le \epsilon
\end{align*}
which, in turn, leads to:
\begin{align*}
\Lambda \ge \frac{1}{4} \log(2/\delta)  \frac{{c_0}}{\epsilon^2}
\end{align*}
which however, leads to $\Lambda \ge 2T$, thus concluding the proof.
\end{proof}

We are now ready to prove the PAC bound on our approximately optimal DCS $\bm{\tilde{m}}^*$.

\pac*
\begin{proof}
First of all, consider the value of the confidence intervals of $\bm{\tilde{m}}^*$. Due to Lemma \ref{lemma:pac-lemma-tech-2}, we know that $\Lambda \ge 2T$ holds; therefore, by applying Lemma \ref{lemma:pac-lemma-tech-1} we obtain that:
\begin{align*}
\sqrt{\frac{1}{2} \log(2/\delta) \sum_{t=0}^{T-1} \frac{c_t}{\tilde{n}^*_t}} \le \sqrt{2 \frac{\log(2/\delta)}{\Lambda} \left(\sum_{t=0}^{T-1} \sqrt{c_t}\right)^2 }  \le \epsilon
\end{align*}
This, in turn, leads to:
\begin{align}\label{eq:pac-bound-eq-1}
2 \frac{\log(2/\delta)}{\Lambda} \left(\sum_{t=0}^{T-1} \sqrt{c_t}\right)^2 \le \epsilon^2 
\end{align}
At this point, focus on $\left(\sum_{t=0}^{T-1} \sqrt{c_t}\right)^2$:
\begin{align*}
\left(\sum_{t=0}^{T-1} \sqrt{c_t}\right)^2 & = \frac{1}{1-\gamma} \left( \sum_{t=0}^{T-1} \sqrt{\gamma^t (\gamma^t + \gamma^{t+1} - 2\gamma^T) } \right)^2 \\ & = \frac{1}{1-\gamma} \left(\sum_{t=0}^{T-1} \gamma^t (\gamma^t + \gamma^{t+1} - 2\gamma^T) + 2\sum_{t=0}^{T-2} \sum_{t'=t+1}^{T-1} \sqrt{\gamma^t (\gamma^t + \gamma^{t+1} - 2\gamma^T)}\sqrt{\gamma^{t'} (\gamma^{t'} + \gamma^{t'+1} - 2\gamma^T)}\right) 
\end{align*}

First, consider $\sum_{t=0}^{T-1} \gamma^t (\gamma^t + \gamma^{t+1} - 2\gamma^T)$:
\begin{align*}
\sum_{t=0}^{T-1} \gamma^t (\gamma^t + \gamma^{t+1} - 2\gamma^T) \le \sum_{t=0}^{T-1} \gamma^t (\gamma^t + \gamma^{t+1}) \le \sum_{t=0}^{T-1} \gamma^t (2\gamma^t) \le 2 \sum_{t=0}^{T-1} \gamma^{2t} \le 2 \frac{1-\gamma^T}{1-\gamma} 
\end{align*}

Then, consider $2\sum_{t=0}^{T-2} \sum_{t'=t+1}^{T-1} \sqrt{\gamma^t (\gamma^t + \gamma^{t+1} - 2\gamma^T)}\sqrt{\gamma^{t'} (\gamma^{t'} + \gamma^{t'+1} - 2\gamma^T)}$:
\begin{align*}
2\sum_{t=0}^{T-2} \sum_{t'=t+1}^{T-1} \sqrt{\gamma^t (\gamma^t + \gamma^{t+1} - 2\gamma^T)}\sqrt{\gamma^{t'} (\gamma^{t'} + \gamma^{t'+1} - 2\gamma^T)} & \le 2\sum_{t=0}^{T-2} \sum_{t'=t+1}^{T-1} \sqrt{\gamma^t (\gamma^t + \gamma^{t+1})}\sqrt{\gamma^{t'} (\gamma^{t'} + \gamma^{t'+1})} \\ &  \le 2\sum_{t=0}^{T-2} \sum_{t'=t+1}^{T-1} \sqrt{\gamma^t (2\gamma^t)}\sqrt{\gamma^{t'} (2\gamma^{t'})} \\ & \le 4\sum_{t=0}^{T-2} \sum_{t'=t+1}^{T-1} \sqrt{\gamma^{2t}} \sqrt{\gamma^{2t'}} \\ & = 4\sum_{t=0}^{T-2} \gamma^{t} \sum_{t'=t+1}^{T-1} \gamma^{t'} \\ & \le 4 \left(\frac{1-\gamma^T}{1-\gamma}\right)^2
\end{align*}
Plugging everything together into Equation \eqref{eq:pac-bound-eq-1} leads to:
\begin{align*}
\frac{2\log(2/\delta)}{\Lambda} \left(\sum_{t=0}^{T-1} \sqrt{c_t}\right)^2 \le \frac{2\log(2/\delta)}{\Lambda(1-\gamma)} \left(2\frac{1-\gamma^T}{1-\gamma} + 4 \frac{(1-\gamma^T)^2}{(1-\gamma)^2} \right) \le \frac{12\log(2/\delta)}{\Lambda(1-\gamma)^3} 
\end{align*}
Solving $\frac{12\log(2/\delta)}{\Lambda(1-\gamma)^3} \le \epsilon^2$ for $\Lambda$ leads to:
\begin{align*}
\Lambda = \mathcal{O}\left( \frac{\log(2/\delta)}{(1-\gamma)^3\epsilon^2} \right)
\end{align*}
which concludes the first part of the proof.

Concerning the second part of the proof, we can bound:
\begin{align*}
\sum_{t=0}^{T-1} \gamma^t (\gamma^t + \gamma^{t+1} - 2\gamma^T) \le 2 \sum_{t=0}^{T-1} \gamma^{2t} \le 2T
\end{align*}
and:
\begin{align*}
2\sum_{t=0}^{T-2} \sum_{t'=t+1}^{T-1} \sqrt{\gamma^t (\gamma^t + \gamma^{t+1} - 2\gamma^T)}\sqrt{\gamma^{t'} (\gamma^{t'} + \gamma^{t'+1} - 2\gamma^T)} & \le 4 \sum_{t=0}^{T-2} \sum_{t'=t+1}^{T-1} \gamma^t \gamma^{t'} \\ & \le 4 T \sum_{t=0}^{T-2} \gamma^t \\ & \le 4 T \frac{1-\gamma^T}{1-\gamma}  
\end{align*}
Plugging everything together into Equation \eqref{eq:pac-bound-eq-1} leads to:
\begin{align*}
\frac{2\log(2/\delta)}{\Lambda} \left(\sum_{t=0}^{T-1} \sqrt{c_t}\right)^2 \le \frac{2\log(2/\delta)}{\Lambda(1-\gamma)} \left(2T + 4T \frac{1-\gamma^T}{1-\gamma} \right) \le \frac{12T\log(2/\delta)}{\Lambda(1-\gamma)^2} 
\end{align*}
Solving $\frac{12T\log(2/\delta)}{\Lambda(1-\gamma)^2} \le \epsilon^2$ for $\Lambda$ leads to:
\begin{align*}
\Lambda = \mathcal{O}\left( \frac{T\log(2/\delta)}{(1-\gamma)^2\epsilon^2} \right)
\end{align*}
which concludes the proof.
\end{proof}

\subsection{Off-Policy Results}\label{app:proofs-off-pol}
As for the on-policy setting, we begin by providing the unbiasedness results for Equation \eqref{eq:offpolicy-dcs-est}.
\begin{theorem}
Consider an optimization budget $\Lambda \ge T$ and a DCS such that $m_T \ge 1$. Consider policies $\pi_{\bm{\bar{\theta}}}, \pi_{\bm{{\theta}}} \in \Pi_{\Theta}$ such that $\pi_{\bm{\bar{\theta}}}(\cdot|s) \ll \pi_{\bm{{\theta}}}(\cdot|s)$ a.s. for every $s \in \mathcal{S}$, then:
\begin{align}
\mathop{\E}_{p_{\bm{m}}(\cdot|\bm{\theta})} \left[ \hat{J}_{\bm{m}}(\bm{\bar{\theta}}/\bm{\theta}) \right] = J(\bm{\bar{\theta}}).
\end{align}
\end{theorem}
\begin{proof}
Define $r_{t,\bm{\bar{\theta}}}$ as the expected $t$-th reward under policy $\pi_{\bm{\bar{\theta}}}$ and consider:

\begin{align*}
\mathop{\E}_{p_{\bm{m}}(\cdot|\bm{\theta})} \left[ \hat{J}_{\bm{m}}(\bm{\bar{\theta}}/\bm{\theta}) \right] & = \mathop{\E}_{p_{\bm{m}}(\cdot|\bm{\theta})} \left[ \sum_{h=1}^T \sum_{i=1}^{m_h} \omega_{\bm{\bar{\theta}}, \bm{\theta}}(\bm{\tau}_{h}^{(i)}) \sum_{t=0}^{h-1} \gamma^t \frac{R(a_t^{(i)},s_t^{(i)})}{n_t} \right] \\ & = \sum_{t=1}^{T} m_h \mathop{\E}_{\bm{\tau}_h \sim p(\cdot|\bm{\theta},h)}\left[ \omega_{\bm{\bar{\theta}}, \bm{\theta}}(\bm{\tau}_h) \sum_{t=0}^{h-1} \gamma^t \frac{R(a_t,s_t)}{n_t}  \right] \\ & = \sum_{t=1}^{T} m_h \mathop{\E}_{\bm{\tau}_h \sim p(\cdot|\bm{\bar{\theta}},h)}\left[ \sum_{t=0}^{h-1} \gamma^t \frac{R(a_t,s_t)}{n_t}  \right] \\ & = \sum_{h=1}^T m_h \sum_{t=0}^{h-1} \gamma^t \frac{r_{t,\bm{\bar{\theta}}}}{n_t},
\end{align*}
where the first equality follows by the definition of $\hat{J}_{\bm{m}}(\bm{\bar{\theta}}/\bm{\theta})$, the second from the linearity of the expectation together with the definition of the data generation process $p_{\bm{m}}(\cdot|\bm{\theta},h)$, the third one from the IS property \citep{owen2013monte}, and the forth one by the linearity of the expectation together with the definition of $r_{t,\bm{\bar{\theta}}}$.
At this point, the rest of the proof follows directly from the one of Theorem \ref{theo:onpolicy-unbias}.
\end{proof}

We now continue by extending the high-probability confidence intervals of \citet{metelli2018policy}. In the rest of this section, we assume that rewards are bounded in $[-R_{\text{MAX}}, R_{\text{MAX}}]$ to allow for a direct comparison with \citet{metelli2018policy}. We also notice that, all the following results are derived based on the Cantelli's inequality, which is an appropriate choice for one-sided tail bounds. Two-sided tail bounds can be straightforwardly derived by using Chebyshev’s inequality. For completeness, we begin by reporting the original result of \citet{metelli2018policy}. 

\begin{restatable}{theorem}{originalpoisbound}\label{theo:originalpoisbound}
Let $\pi_{\bm{\bar{\theta}}},\pi_{\bm{\theta}} \in \Pi_{\Theta}$ such that $\pi_{\bm{\bar{\theta}}}(\cdot| s) \ll \pi_{\bm{\theta}}(\cdot| s)$ a.s. for every $s \in \mathcal{S}$. Let us define the off-policy expected return estimator with $K$ trajectories of horizon $T$ collected with $\pi_{\bm{\theta}}$:
\begin{align}\label{eq:classical-off-policy-est}
\hat{J}(\bm{\bar{\theta}}/\bm{\theta}) = \frac{1}{K}\sum_{i=1}^{K} \omega_{\bm{\bar{\theta}}/\bm{\theta}}(\bm{\tau}_T^{(i)}) \sum_{t=0}^{T-1} \gamma^t R(s_t^{(i)}, a_t^{(i)}) 
\end{align}
Let $\Lambda = KT$, $\beta_{\delta} = \frac{1-\delta}{\delta}$ and $\phi=R_{\text{MAX}}\frac{1-\gamma^T}{1-\gamma}$, then,
Then, with probability at least $1-\delta$ it holds that:
\begin{align}\label{eq:original-pois-bound}
J(\bm{\bar{\theta}}) \ge \hat{J}(\bm{\bar{\theta}}/\bm{\theta}) - \phi \sqrt{\frac{T \beta_{\delta} {d}_2\left( p(\cdot |\bm{\bar{\theta}},T)\|p(\cdot |\bm{{\theta}},T) \right)}{\Lambda}},
\end{align}
\end{restatable}

At this point, we are ready to provide our generalization of Theorem \ref{theo:originalpoisbound} (which, for $R_{\text{MAX}}=1$, reduces to Theorem \ref{theo:off-policy-ci-main} of Section \ref{sec:eval}).

\begin{theorem}\label{theo:off-policy-ci}
Consider $\pi_{\bm{\bar{\theta}}}$, $\pi_{\bm{{\theta}}} \in \Pi_{\Theta}$ such that $\pi_{\bm{\bar{\theta}}}(\cdot|s) \ll \pi_{\bm{\theta}}(\cdot |s)$ a.s. for all $s \in \mathcal{S}$. Consider an optimization budget $\Lambda \ge T$ and a generic DCS $\bm{m}$.
Then, with probability at least $1-\delta$ it holds that:
\begin{align}\label{eq:general-offpol-bound}\resizebox{.42\textwidth}{!}{$\displaystyle
J(\bm{\bar{\theta}}) \ge \hat{J}_{\bm{m}}(\bm{\bar{\theta}}/\bm{\theta}) - \sqrt{ \beta_{\delta} \sum_{h=1}^T m_h \phi_h^2 d_2(p(\cdot|\bm{\bar{\theta}},h) \| p(\cdot|\bm{\theta},h)) },$}
\end{align}
where $\beta_{\delta} = \frac{1-\delta}{\delta}$ and $\phi_h \coloneqq R_{\textrm{MAX}} \sum_{t=0}^{h-1} \frac{\gamma^t}{n_t}$ .
\end{theorem}
\begin{proof}
As in \citet{metelli2018policy}, we split the proof into two parts, i.e., first we upper bound the variance of the estimator, and then we make use of the Cantelli's inequality to prove Equation \eqref{eq:general-offpol-bound}.

Let us start with the variance bound. Consider:

\begin{align*}
\mathop{\Var}_{p_{\bm{m}}(\cdot|\bm{\theta})} \left[ \hat{J}_{\bm{m}}(\bm{\bar{\theta}}/\bm{\theta}) \right] = \mathop{\Var}_{p_{\bm{m}}(\cdot|\bm{\theta})} \left[ \sum_{h=1}^{T} \sum_{i=1}^{m_h} \omega_{\bm{\bar{\theta}}, \bm{\theta}}(\bm{\tau}_{h}^{(i)}) \sum_{t=0}^{h-1} \gamma^t \frac{R(a_t^{(i)},s_t^{(i)})}{n_t}  \right].
\end{align*}

Since the different trajectories are independent, we can write:
\begin{align*}
\mathop{\Var}_{p_{\bm{m}}(\cdot|\bm{\theta})} \left[ \sum_{h=1}^T \sum_{i=1}^{m_h} \omega_{\bm{\bar{\theta}}, \bm{\theta}}(\bm{\tau}_{h}^{(i)}) \sum_{t=0}^{h-1} \gamma^t \frac{R(a_t^{(i)},s_t^{(i)})}{n_t}  \right] & = \sum_{h=1}^T m_h \mathop{\Var}_{\bm{\tau}_h \sim p(\cdot|\bm{\theta},h)} \left[ \omega_{\bm{\bar{\theta}}, \bm{\theta}}(\bm{\tau}_{h}) \sum_{t=0}^{h-1} \gamma^t \frac{R(s_t,a_t)}{n_t}\right] \\ & \le \sum_{h=1}^T m_h \mathop{\E}_{\bm{\tau}_h \sim p(\cdot|\bm{\theta},h)} \left[ \left( \omega_{\bm{\bar{\theta}}, \bm{\theta}}(\bm{\tau}_{h}) \sum_{t=0}^{h-1} \gamma^t \frac{R(s_t,a_t)}{n_t} \right) ^2 \right]  \\ & \le \sum_{h=1}^T m_h \mathop{\E}_{\bm{\tau}_h \sim p(\cdot|\bm{\theta},h)} \left[ \left( \omega_{\bm{\bar{\theta}}, \bm{\theta}}(\bm{\tau}_{h}) \sum_{t=0}^{h-1} \gamma^t \frac{R_{\textrm{MAX}}}{n_t} \right) ^2 \right] \\ & = \sum_{h=1}^T m_h \mathop{\E}_{\bm{\tau}_h \sim p(\cdot|\bm{\theta},h)} \left[ \left( \omega_{\bm{\bar{\theta}}, \bm{\theta}}(\bm{\tau}_{h}) \phi_h \right) ^2 \right] \\ & = \sum_{h=1}^T m_h \phi_h^2 d_2(p(\cdot|\bm{\bar{\theta}},h) || p(\cdot|\bm{{\theta}},h)),
\end{align*}
where the last passage follows from the relationship between the moments of the importance weights and the Rényi divergence. This concludes the first part of the proof.

Concerning the second part, we start from the Cantelli's inequality applied on the random variable $\hat{J}_{\bm{m}}(\bm{\bar{\theta}}/\bm{\theta})$, namely:



\begin{align*}
\mathbb{P}\left( \hat{J}_{\bm{m}}(\bm{\bar{\theta}}/\bm{\theta}) - J(\bm{\bar{\theta}})  \ge \alpha  \right) \le \frac{1}{1 + \frac{\alpha^2}{\mathop{\Var}_{p_{\bm{m}}(\cdot | \bm{\theta})}\left[ \hat{J}_{\bm{m}}(\bm{\bar{\theta}}/\bm{\theta}) \right]}}
\end{align*}

Set $\delta = \frac{1}{1 + \frac{\alpha^2}{\mathop{\Var}_{p_{\bm{m}}(\cdot | \bm{\theta})}\left[ \hat{J}_{\bm{m}}(\bm{\bar{\theta}}/\bm{\theta}) \right]}}$  and consider the complementary event. Then, with probability at least $1-\delta$ it holds that:

\begin{align*}
J(\bm{\bar{\theta}}) & \ge  \hat{J}_{\bm{m}}(\bm{\bar{\theta}}/\bm{\theta}) - \sqrt{ \frac{1-\delta}{\delta} \mathop{\Var}_{p_{\bm{m}}(\cdot | \bm{\theta})}\left[ \hat{J}_{\bm{m}}(\bm{\bar{\theta}}/\bm{\theta}) \right]} \\ & \ge \hat{J}_{\bm{m}}(\bm{\bar{\theta}}/\bm{\theta}) - \sqrt{ \beta_{\delta} \sum_{h=1}^T m_h \phi_h^2 d_2(p(\cdot|\bm{\bar{\theta}},h) \| p(\cdot|\bm{\theta},h)) },
\end{align*}

which concludes the proof.

\end{proof}

As we can appreciate, using the uniform DCS in Equation \eqref{eq:general-offpol-bound} we recover exactly Equation \eqref{eq:original-pois-bound} of Theorem \ref{theo:originalpoisbound}. At this point, one might be tempted to directly minimize Equation the confidence intervals around $J(\bm{\bar{\theta}}/\bm{{\theta}})$ as a function of $\bm{m}$ to obtain a tighter high-probability bound. However, as noted in \citet{metelli2018policy}, computing the Rényi divergence over the trajectory space requires both the approximation of a complex integral, and, for stochastic environments, the knowledge of the transition kernel $P$ of the underlying MDP. Therefore, to derive a tractable expression that can be optimized as a function of the DCS, we further bound each term $d_2(p(\cdot|\bm{\bar{\theta}},h) \| p(\cdot|\bm{\theta},h))$ with $d_2(p(\cdot|\bm{\bar{\theta}},T) \| p(\cdot|\bm{\theta},T))$, which is justified by the following result.

\begin{lemma}\label{lemma:renyi-bound}
Consider two policies $\pi_{\bm{\bar{\theta}}}, \pi_{\bm{\theta}} \in \Pi_{\Theta}$ such that $\pi_{\bm{\bar{\theta}}} \ll \pi_{\bm{\theta}}$ a.s. for every $s \in \mathcal{S}$. Consider $h \in \{1, \dots, T-2\}$, then:
\begin{align*}
d_2(p(\cdot|\bm{\bar{\theta}},h) \| p(\cdot|\bm{\theta},h)) \le d_2(p(\cdot|\bm{\bar{\theta}},h+1) \| p(\cdot|\bm{\theta},h+1)).
\end{align*}
\end{lemma}

\begin{proof}
Focus on $h+1$. Due to the link between $d_2(p(\cdot|\bm{\bar{\theta}},h+1) \| p(\cdot|\bm{\theta},h+1))$ and the second moment of the importance weights, we have that:

\begin{align*}
d_2(p(\cdot|\bm{\bar{\theta}},h+1) \| p(\cdot|\bm{\theta},h+1)) & = \mathop{\E}_{\bm{\tau}_{h+1} \sim p(\cdot|\bm{\theta},h+1)} \left[ \prod_{t=0}^h \left(\frac{\pi_{\bm{\bar{\theta}}}(a_t|s_t)}{\pi_{\bm{{\theta}}}(a_t|s_t)}\right)^2 \right] \\ & = \mathop{\E}_{\bm{\tau}_{h+1} \sim p(\cdot|\bm{\theta},h+1)} \left[ \prod_{t=0}^{h-1} \left(\frac{\pi_{\bm{\bar{\theta}}}(a_t|s_t)}{\pi_{\bm{{\theta}}}(a_t|s_t)}\right)^2  \left(\frac{\pi_{\bm{\bar{\theta}}}(a_h|s_h)}{\pi_{\bm{{\theta}}}(a_h|s_h)}\right)^2 \right].
\end{align*}
Now, since $\bm{\tau}_{h+1} = (s_0, a_0, \dots, s_{h}, a_{h}, s_{h+1})$, we can write the last expectation as:
\begin{align}\label{eq:ren-proof-eq1}
\mathop{\E}_{(s_0, a_0, \dots, s_{h-1}, a_{h-1}) \sim p(\cdot|\bm{\theta},h)} \left[ \prod_{t=0}^{h-1} \left(\frac{\pi_{\bm{\bar{\theta}}}(a_t|s_t)}{\pi_{\bm{{\theta}}}(a_t|s_t)}\right)^2  \mathop{\E}_{(s_{h}, a_{h}, s_{h+1}) \sim p_h(\cdot|\bm{\theta},h+1)} \left[ \left(\frac{\pi_{\bm{\bar{\theta}}}(a_h|s_h)}{\pi_{\bm{{\theta}}}(a_h|s_h)}\right)^2 \right] \right].
\end{align}
where with $p_h(\cdot|\bm{\theta},h+1)$ we denote the $h$-th step (i.e., the last one) in a trajectory of length $h+1$. With a little abuse of notation, we drop the dependency on $p_h(\cdot|\bm{\theta},h+1)$ and we write:
\begin{align}\label{eq:ren-proof-eq2}
\mathop{\E}_{s_{h}, a_{h}} \left[ \left(\frac{\pi_{\bm{\bar{\theta}}}(a_h|s_h)}{\pi_{\bm{{\theta}}}(a_h|s_h)}\right)^2 \right] =  \mathop{\E}_{s_{h}} \left[ \mathop{\E}_{a_h \sim \pi_{\bm{\theta}}(\cdot|s_h)} \left[ \left(\frac{\pi_{\bm{\bar{\theta}}}(a_h|s_h)}{\pi_{\bm{{\theta}}}(a_h|s_h)}\right)^2 \right] \right] \ge \inf_{s \in \mathcal{S}} \mathop{\E}_{a \sim \pi_{\bm{\theta}}(\cdot|s)} \left[ \left(\frac{\pi_{\bm{\bar{\theta}}}(a|s)}{\pi_{\bm{{\theta}}}(a|s)}\right)^2 \right] \ge 1,
\end{align}
where the last inequality follows from the fact that $\mathop{\E}_{a \sim \pi_{\bm{\theta}}(\cdot|s)} \left[ \left(\frac{\pi_{\bm{\bar{\theta}}}(a|s)}{\pi_{\bm{{\theta}}}(a|s)}\right)^2 \right]$ can be interpreted as the exponentiated Rényi divergence with $\alpha=2$ at state $s$. 
At this point, plugging Equation \eqref{eq:ren-proof-eq2} in Equation \eqref{eq:ren-proof-eq1}, it follows that:
\begin{align*}
d_2(p(\cdot|\bm{\bar{\theta}},h+1) \| p(\cdot|\bm{\theta},h+1)) & \ge \mathop{\E}_{(s_0, a_0, \dots, s_{h-1}, a_{h-1}) \sim p(\cdot|\bm{\theta},h)} \left[ \prod_{t=0}^{h-1} \left(\frac{\pi_{\bm{\bar{\theta}}}(a_t|s_t)}{\pi_{\bm{{\theta}}}(a_t|s_t)}\right)^2 \right] \\ & = \mathop{\E}_{\bm{\tau_h} \sim p(\cdot|\bm{\theta},h)} \left[ \prod_{t=0}^{h-1} \left(\frac{\pi_{\bm{\bar{\theta}}}(a_t|s_t)}{\pi_{\bm{{\theta}}}(a_t|s_t)}\right)^2 \right] \\ & = d_2(p(\cdot|\bm{\bar{\theta}},h) \| p(\cdot|\bm{\theta},h)),
\end{align*}
which concludes the proof.
\end{proof}

At this point, combining Lemma \ref{lemma:technical-lemma} with Lemma \ref{lemma:renyi-bound} and Theorem \ref{theo:off-policy-ci}, we obtain the following contrained optimization problem.
\begin{equation}\label{sys:offpolicy-hard}
\begin{aligned}
\min_{\bm{n}} \quad & \sqrt{\beta_{\delta} d_2(p(\cdot|\bm{\bar{\theta}},T) || p(\cdot|\bm{\theta},T)) \sum_{t=0}^{T-1} \frac{c_t}{n_t} } \\
\textrm{s.t.} \quad & n_t \ge n_{t+1}, \quad \forall t \in \{0, \dots, T-2\} \\
  & \sum_{t=0}^{T-1} n_t = \Lambda  \\
  & n_t \in \mathbb{N}_+, \quad \forall t \in \{0, \dots, T-1\}
\end{aligned}
\end{equation}
Notice that \eqref{sys:offpolicy-hard} is equivalent, up to constant factors, to the one of the on-policy case. Consequently, it can be solved using the same methodology applied in the previous section. More specifically, following the same proof scheme, it is possible to derive the following result.

\begin{theorem}\label{theo:off-pol-relaxation-remark}
Consider an optimization budget $\Lambda > T$, let $\bm{\tilde{n}}^*$ be the approximately optimal DCS given in Defintion \ref{def:approx-opt-dcs}, and let $\bm{n}^*$ be the optimal solution of the integer optimization problem \eqref{sys:offpolicy-hard}. Moreover, let $f(\bm{n}) = \sqrt{\beta_{\delta} d_2(p(\cdot|\bm{\bar{\theta}},T) || p(\cdot|\bm{\theta},T)) \sum_{t=0}^{T-1} \frac{c_t}{n_t} }$. Then,
\begin{align}
f(\bm{n}^*) \le f(\bm{\tilde{n}}^*) \le \sqrt{2} f(\bm{n}^*)
\end{align}
\end{theorem}

\subsection{Further analysis}\label{app:proofs-further}

In Theorem \ref{theo:relaxation-remark}, we have seen that $f(\bm{\tilde{n}}^*) \le \sqrt{2}f(\bm{n^*})$. We have now ask ourselves if we can obtain tighter values for the constant. The following results provides a positive answer.

\begin{proposition}\label{prop:relaxation-remark-improved}
Consider an optimization budget $\Lambda > T$, let $\bm{\tilde{n}}^*$ be the approximately optimal DCS given in Defintion \ref{def:approx-opt-dcs}, let $\bm{n}^*$ be the optimal solution of the integer optimization problem \eqref{sys:onpolicy-hard}, and let $\bm{\bar{n}}^*$ be the solution of the convex relaxation of \eqref{sys:onpolicy-hard} given in Lemma \ref{lemma:convex-closed-form}. Define $\mathcal{X} = \{x \in (0,1) : \bar{n}^*_{T-1} \ge \frac{1}{1-x} \}$. Then, if $\bar{n}^*_{T-1} > 1$ holds, we have that:
\begin{align}\label{eq:relaxation-improved-bound}
f(\bm{\tilde{n}}^*) \le \min_{x \in \mathcal{X}} \sqrt{\frac{1}{x}}f(\bm{n}^*)
\end{align}
\end{proposition} 

\begin{proof}
Let us analyze: $\frac{f(\bm{\tilde{n}}^*)}{f(\bm{n}^*)}$. For the same reasoning of Theorem \ref{theo:relaxation-remark}, we have that:
\begin{align}\label{eq:relax-imp-eq1}
\frac{f(\bm{\tilde{n}}^*)}{f(\bm{n}^*)} \le \frac{f(\bm{\tilde{n}}^*)}{f(\bm{\bar{n}}^*)}
\end{align}
Then, we provide an upper bound on $f(\bm{\tilde{n}}^*)$ that holds whenever $\bar{n}^*_{T-1} \ge 1$. More specifically, consider a generic $x \in \mathcal{X}$:\footnote{Notice that when $\bar{n}^*_{T-1} > 1$, $\mathcal{X}$ is a non-empty set.}
\begin{align}\label{eq:relax-imp-eq2}
f(\bm{\tilde{n}}^*) = \sqrt{\frac{1}{2}\log(2/\delta) \sum_{t=0}^{T-1} \frac{c_t}{\tilde{n}^*_t}} \le \sqrt{\frac{1}{2}\log(2/\delta) \sum_{t=0}^{T-1} \frac{c_t}{\bar{n}^*_t - 1}} \le \sqrt{\frac{1}{2}\log(2/\delta) \sum_{t=0}^{T-1} \frac{c_t}{x\bar{n}^*_t}} = \sqrt{\frac{1}{x}} f(\bm{\bar{n}^*})
\end{align}
 where, in the first inequality we have used Lemma \ref{lemma:diff-bound} together with $\bar{n}^*_{T-1} > 1$ \footnote{Notice that $\bar{n}^*_{T-1} > 1$ implies $\bar{n}^*_{t} > 1$ for all $t \in \{0, \dots, T-1 \}$.}, and in the second one we have used the definition of $\mathcal{X}$. 
 
Plugging Equation \eqref{eq:relax-imp-eq2} into \eqref{eq:relax-imp-eq1} concludes the proof.
\end{proof}

Proposition \ref{prop:relaxation-remark-improved} provides a tighter upper bound that depends on the number of samples allocated to $\bar{n}^*_{T-1}$ by the solution of the convex relaxation of \eqref{sys:onpolicy-hard}. Due to Lemma \ref{lemma:convex-closed-form}, we know that there is a tight relationship between $\bar{n}^*_t$ and the available budget $\Lambda$. More specifically, we can appreciate that, as the budget increase, so does $\bar{n}^*_{T-1}$. Due to Equation \eqref{eq:relaxation-improved-bound}, this, in turn, implies tighter upper bounds on the quality of the approximately optimal DCS. As an example, suppose that $\bar{n}^*_{T-1} = 100$. Then, we have that:
\begin{align*}
f(\bm{\tilde{n}}^*) \le \sqrt{\frac{100}{99}}f(\bm{n}^*)
\end{align*}

\begin{figure}[t]
\centering\includegraphics[width=4in]{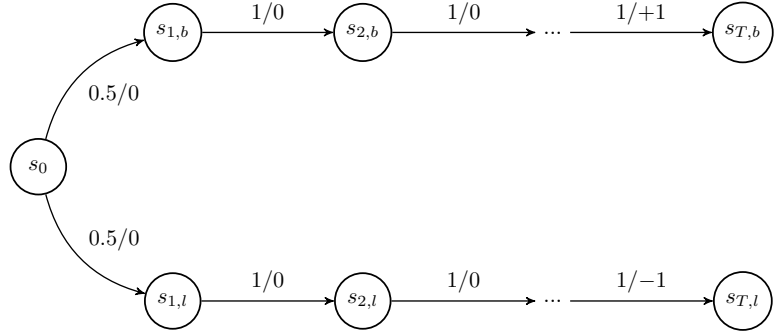} 
\vspace{.3in}
\caption{MDP example with a fixed policy $\pi_{\bm{\bar{\theta}}}$ in which rewards are gathered only at the end of the episode. Each edge reports the probability of taking that action, together with its associated reward. For states in which there is a single edge that is followed with probability $1$ (e.g., $s_{1,l}$), other actions have been masked.}
\label{fig:mrpgrid}
\end{figure}

We now provide some variance analysis for settings in which rewards are gathered at the end of the episode.

\begin{proposition}\label{prop:variance-gridworld}
Consider the MDP of Figure \ref{fig:mrpgrid} together with the indicated policy. Fix a DCS $\bm{m}$ such that $m_T \ge 1$. Then, it holds that:
\begin{align}\label{eq:var-mrp-app-gridworld}
\mathop{\Var}_{p_{\bm{m}}(\cdot|{\bm{\bar{\theta}}})} \left[ \hat{J}_{\bm{m}}({\bm{\bar{\theta}}}/\bm{\bar{\theta}}) \right] = \frac{\gamma^{2(T-1)}}{n_{T-1}}.
\end{align}
\end{proposition}

\begin{proof}
\begin{align*}
\mathop{\Var}_{p_{\bm{m}}(\cdot|{\bm{\bar{\theta}}})} \left[ \hat{J}_{\bm{m}}({\bm{\bar{\theta}}}/\bm{\bar{\theta}}) \right]  = \mathop{\Var}_{p_{\bm{m}}(\cdot|{\bm{\bar{\theta}}})} \left[ \sum_{h=1}^T \sum_{i=1}^{m_h} \sum_{t=0}^{h-1} \gamma^t \frac{R(s_t^{(i)}, a_t^{(i)})}{n_t} \right]
\end{align*}
Since the different trajectories are independent, we can write:
\begin{align*}
\mathop{\Var}_{p_{\bm{m}}(\cdot|{\bm{\bar{\theta}}})} \left[ \sum_{h=1}^T \sum_{i=1}^{m_h} \sum_{t=0}^{h-1} \gamma^t \frac{R(s_t^{(i)}, a_t^{(i)})}{n_t} \right] & = \sum_{h=1}^T m_h \mathop{\Var}_{p(\cdot|\bm{\bar{\theta}},h)} \left[ \sum_{t=0}^{h-1} \gamma^t \frac{R(s_t, a_t)}{n_t} \right] \\ & = \sum_{h=1}^T m_h \mathop{\E}_{p(\cdot|\bm{\bar{\theta}},h)} \left[ \left( \sum_{t=0}^{h-1} \gamma^t \frac{R(s_t, a_t)}{n_t} - \mathop{\E}_{p(\cdot|\bm{\bar{\theta}},h)} \left[ \sum_{t=0}^{h-1} \gamma^t \frac{R(s_t, a_t)}{n_t} \right]  \right)^2 \right] \\ & = \sum_{h=1}^T m_h \mathop{\E}_{p(\cdot|\bm{\bar{\theta}},h)} \left[ \left( \sum_{t=0}^{h-1} \gamma^t \frac{R(s_t, a_t)}{n_t}   \right)^2 \right] \\ & = \sum_{h=1}^T m_h \mathop{\E}_{p(\cdot|\bm{\bar{\theta}},h)} \left[ \sum_{t=0}^{h-1} \gamma^{2t} \frac{R(s_t, a_t)^2}{n_t^2} + \sum_{t=0}^{h-2} \sum_{t'=t+1}^{h-1} \gamma^{t+t'} \frac{R(s_t, a_t) R(s_{t'},a_{t'}) }{n_t n_{t'}}  \right] \\ & = m_{T} \frac{\gamma^{2(T-1)}}{n_{T-1}^2} \\ & = \frac{\gamma^{2(T-1)}}{n_{T-1}}
\end{align*} 
which concludes the proof.
\end{proof}

From Equation \eqref{eq:var-mrp-app-gridworld}, we can make the following consideration. As noticed in Section \ref{sec:eval}, the uniform strategy is intuitively a good choice for settings such as the one of Figure~\ref{fig:mrpgrid}. Indeed, the only relevant data is gathered at the end of the episode, and, therefore, we need to interact with the environment as much as possible at time $t=T-1$. The variance of the uniform strategy will be given by:
\begin{align}\label{eq:gridworld-uniform-variance}
\frac{\gamma^{2(T-1)}T}{\Lambda}
\end{align}
At this point, what can we say about our approximately optimal DCS? The following proposition summarizes the result.

\begin{proposition}
Consider the MDP of Figure \ref{fig:mrpgrid} together with the indicated policy. Suppose that $T \ge \frac{\log\left( \frac{1}{(1-\gamma)\gamma}\right)}{\log\left( \frac{1}{\gamma}\right)}$ holds. Consider $\bm{\tilde{m}}^*$ as in Definition \eqref{def:approx-opt-dcs}. 
Then, for sufficiently large values of $\Lambda$, it holds that:
\begin{align}\label{eq:gridworld-our-variance}
\mathop{\Var}_{p_{\bm{\tilde{m}}^*}(\cdot|{\bm{\bar{\theta}}})} \left[ \hat{J}_{\bm{m}}({\bm{\bar{\theta}}}/\bm{\bar{\theta}}) \right] \le \frac{2 \gamma^{\frac{1}{2}(T-1)}T}{\Lambda}.
\end{align}
\end{proposition}
\begin{proof}
Due to Proposition \ref{prop:variance-gridworld}, we are interested in studying:
\begin{align*}
\frac{\gamma^{2(T-1)}}{\tilde{n}^*_{T-1}}
\end{align*}
Let $\bm{\bar{n}}^*$ be the solution of the convex relaxation given in Lemma \ref{lemma:convex-closed-form}. Then,
\begin{align*}
\frac{\gamma^{2(T-1)}}{\bar{n}^*_{T-1}} \ge \frac{\gamma^{2(T-1)}}{\tilde{n}^*_{T-1} + 1} \ge \frac{\gamma^{2(T-1)}}{2\tilde{n}^*_{T-1}}
\end{align*}
Which leads to:
\begin{align}\label{eq:gridworld-our-variance-eq1}
\frac{\gamma^{2(T-1)}}{\tilde{n}^*_{T-1}} \le \frac{2 \gamma^{2(T-1)}}{\bar{n}^*_{T-1}} 
\end{align}
At this point, due to Lemma \ref{lemma:convex-closed-form}, for sufficiently large values of $\Lambda$, we can substitute $\bar{n}_{T-1}^*$ with:
\begin{align}\label{eq:gridworld-our-variance-eq2}
\frac{2 \gamma^{2(T-1)}}{\bar{n}^*_{T-1}} = \frac{2 \gamma^{2(T-1)} \sum_{t=0}^{T-1}\sqrt{c_t}}{\Lambda\sqrt{c_{T-1}}} \le \frac{2 \gamma^{2(T-1)} T}{\Lambda\sqrt{\gamma^{T-1}(\gamma^{T-1}-\gamma^{T})}} = \frac{2 \gamma^{2(T-1)} T}{\Lambda\sqrt{\gamma^{2(T-1)}(1-\gamma)}}
\end{align}
where in the inequality step we have used the definition of $c_t$. Moreover, since $T \ge \frac{\log\left( \frac{1}{(1-\gamma)\gamma}\right)}{\log\left( \frac{1}{\gamma}\right)}$ holds by assumption, we have that $(1-\gamma) \ge \gamma^{T-1}$, therefore, we can further upper-bound Equation \eqref{eq:gridworld-our-variance-eq2} with:
\begin{align*}
\frac{2 \gamma^{2(T-1)} T}{\Lambda\sqrt{\gamma^{3(T-1)}}} = \frac{2 \gamma^{\frac{1}{2}(T-1)} T}{\Lambda}
\end{align*}
which concludes the proof.
\end{proof}

We now make some remarks both on the setting and the comparison between the uniform strategy and our approximately optimal DCS $\bm{\tilde{m}}^*$. First of all, from Equation \eqref{eq:var-mrp-app-gridworld} we notice that the variance tends to $0$ with an exponential rate w.r.t. the horizon $T$, meaning that, when the horizon is sufficiently large, any method will enjoy numerically low variance. Furthermore, Equation \eqref{eq:gridworld-our-variance} shows that, for sufficiently large values of $\Lambda$, the variance of $\hat{J}_{\bm{m}}({\bm{\bar{\theta}}}/\bm{\bar{\theta}})$ displays a very similar behavior w.r.t. the one obtained by the uniform approach; the only difference, indeed, stands in a different power of $\gamma$. Finally, we remark that, in such sparse reward settings, $\gamma$ is usually very close to $1$ to avoid nullifying the rewards that are gathered at the end of the episode. However, whenever this happens, $\bm{\tilde{m}}^*$ will tend to the uniform strategy.


\section{Visualizations}\label{app:vis}
In this Section, we provide some visualizations of the improvement in the confidence intervals in Figure~\ref{fig:bound-imp} when varying $\gamma$ and $T$, and keeping fixed $\Lambda$. The improvement is reported as $\frac{100 f(\bm{\tilde{n}}^*)}{f(\bm{n}_u)}$, where $\bm{n}_u$ is the uniform allocation strategy and $f(\bm{n}) = \sqrt{\frac{1}{2}\log(2/\delta) \sum_{t=0}^{T-1} \frac{c_t}{n_t}}$ . Moreover, Figure~\ref{fig:strat-vis} provides visualizations on $\bm{\tilde{n}}^*$ as a function of $\gamma$ and for different values of $\Lambda$ when $T=10$.

\begin{figure*}[t!]
\centering\includegraphics[width=10cm]{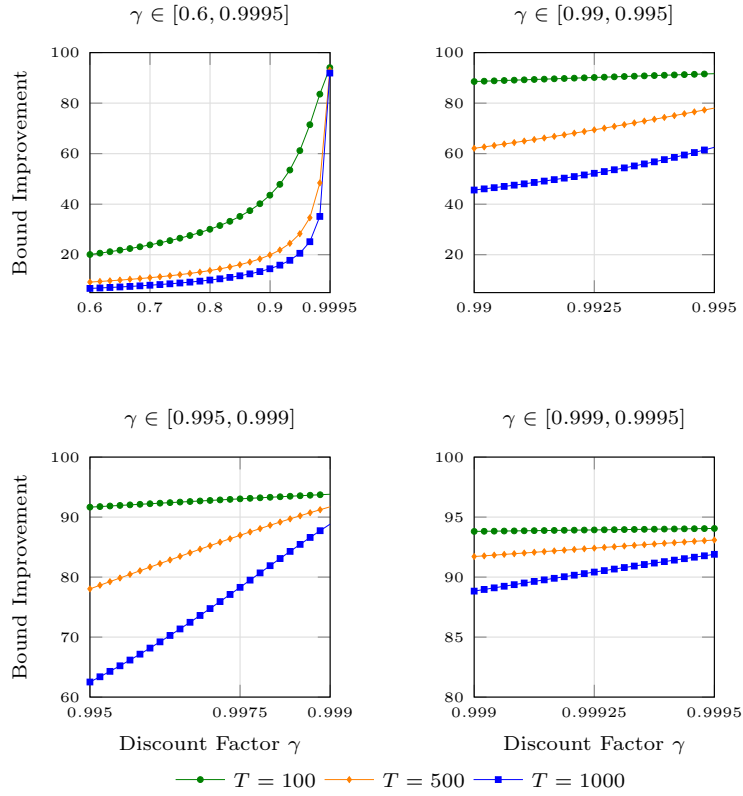} 
\vspace{.3in}
\caption{Visualization of the confidence interval improvements reported as a function of $\gamma$ for different values of $T$, when using $\Lambda=10k$.} 
\label{fig:bound-imp}
\end{figure*}

\begin{figure*}[t!]
\centering\includegraphics[width=18cm]{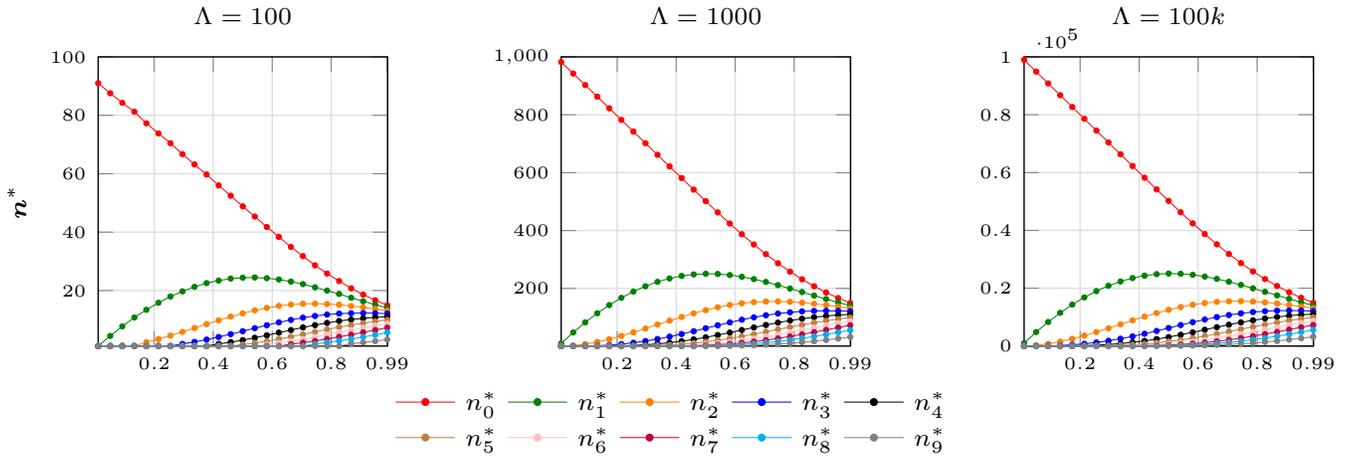} 
\vspace{.3in}
\caption{Visualization of $\bm{n}^*$ for different values of $\gamma$ and $\Lambda$ when using $T=10$.} 
\label{fig:strat-vis}
\end{figure*}


\section{Additional details on POIS and TT-POIS}\label{app:pois-details}

\begin{algorithm}[t]
\caption{Truncating Trajectories in Policy Optimization via Importance Sampling (TT-POIS)} \label{alg:ttpo}
\begin{algorithmic}[1]
\REQUIRE{Optimization budget $\Lambda$, confidence level $\delta$}
\STATE{Initialize $\bm{\theta}_0^0$ arbitrarly}
\STATE{Compute $\bm{\tilde{m}}^*$ as in Definition \ref{def:approx-opt-dcs}}
\FOR{$j=0,1,2,\dots $}
\STATE{Collect dataset $\mathcal{D} \sim p_{\bm{\tilde{m}}^*}(\cdot | \bm{\theta}_0^j)$}
\FOR{$k=0, 1, 2, \dots$}
\STATE{Compute $\nabla \mathcal{L}_{\delta}(\bm{\theta}_k^j / \bm{\theta}_0^j)$ and $\alpha_k$}
\STATE{$\bm{\theta}_{k+1}^j = \bm{\theta}_k^j + \alpha_k \nabla \mathcal{L}(\bm{\theta}_k^j / \bm{\theta}_0^j)$}
\ENDFOR
\ENDFOR
\end{algorithmic}
\end{algorithm}

\subsection{Pseudo-code and other details}

The pseudo-code for our algorithm, \ttpois, can be found in Algorithm \ref{alg:ttpo}. As one can notice, by replacing $\bm{\tilde{m}}^*$ with the uniform-in-the-horizon DCS, we recover the original pseudo-code of \pois \citep{metelli2018policy}. We remark that, as in \pois, $\delta$ is treated as an hyper-parameter, and that in Line $6$, the step size is computed online via line search. 

We now recall the definition of the objective function $\mathcal{L}_{\delta}$ provided in Section \ref{sec:optim}, namely:
\begin{align}\label{eq:obj-fun-app}
\mathcal{L}_{\delta}(\bm{\bar{\theta}}/\bm{\theta})  \coloneqq \hat{J}_{\bm{\tilde{m}}^*}(\bm{\bar{\theta}}/\bm{\theta}) -\sqrt{\beta_{\delta} \sum_{h=1}^T \tilde{m}_h^* (\tilde{\phi}_h^*)^2 \hat{d}_2(p(\cdot | \bm{\bar{\theta}},h ) | p(\cdot| \bm{\theta},h) )},
\end{align}
where $\tilde{\phi}_h^* = \sum_{t=0}^{h-1} \frac{\gamma^t}{\tilde{n}_t^*}$ and $\hat{d}_2(p(\cdot | \bm{\bar{\theta}},h ) | p(\cdot| \bm{\theta},h))$ is a sampled-based approximation for ${d}_2(p(\cdot | \bm{\bar{\theta}},h) | p(\cdot| \bm{\theta},h))$. More specifically,
\begin{align}\label{eq:ttpois-emp-ren}
\hat{d}_2(p(\cdot | \bm{\bar{\theta}},h ) | p(\cdot| \bm{\theta},h)) = \frac{1}{n_{h}} \sum_{i=h}^{T} \sum_{j=1}^{\tilde{m}_i^*} \prod_{t=0}^{h-1} d_2( \pi_{\bm{\bar{\theta}}}(\cdot|s_{\tau_i^{(j)},t})  \|  \pi_{\bm{{\theta}}}(\cdot|s_{\tau_i^{(j)},t})) 
\end{align}
Notice that Equation \eqref{eq:ttpois-emp-ren}, when applied with the uniform DCS, recovers the same approximation used in \citet{metelli2018policy} (see their Equation 41).

As an additional comment, we notice that in Section \ref{sec:eval}, Equation \eqref{eq:obj-fun-app} has been presented for rewards in $[0,1]$. For the more general case in which rewards are defined in $[-R_{\text{MAX}}, R_{\text{MAX}}]$, it is sufficient to replace $\tilde{\phi}_h^* = \sum_{t=0}^{h-1} \frac{\gamma^t}{\tilde{n}_t^*}$ with $\tilde{\phi}_h^* = R_{\text{MAX}} \sum_{t=0}^{h-1} \frac{\gamma^t}{\tilde{n}_t^*}$ (see Theorem \ref{theo:off-policy-ci}).

\subsection{Implementation Details}

Our implementation follows directly from the original one of POIS \citep{metelli2018policy}. More specifically, the line search method adopted for performing the update (Line $6$ in Algorithm \ref{alg:ttpo}) is the same of \citet{metelli2018policy}. In this sense, the reader can refer to Appendix E.1 of \citet{metelli2018policy} for futher details. Compared to \citet{metelli2018policy}, however, we introduce the two following hyper-parameters that have been used, in our experiments, both for \pois and \ttpois. 

\paragraph{Minimum-maximum empirical reward}
In the original version of POIS \citep{metelli2018policy} (and also in our experiments), in $\mathcal{L}(\bm{\bar{\theta}}/\bm{\theta})$, $R_{\textrm{MAX}}$ is replaced with the maximum empirical reward $\hat{R}_{\textrm{MAX}}$ that is collected at the current training iteration. This leads to a further adaptivity of $\mathcal{L}(\bm{\bar{\theta}}/\bm{\theta})$. However, in domain such as the Reacher, where the rewards tend to be close to $0$ when good policies are learnt, using the maximum empirical reward might lead to numerical instabilities. Indeed, in these situations, the adaptive trust region will approach $0$, and both \pois and \ttpois will simply maximimize $\hat{J}(\bm{\bar{\theta}}/\bm{\theta})$, with no control on the variance of the importance weights. For this reason, we define an additional hyper-parameter, $R_{\textrm{MIN-MAX}}$, that defines a minimum threshold for $\hat{R}_{\textrm{MAX}}$. If $\hat{R}_{\textrm{MAX}}$ falls below ${R}_{\textrm{MIN-MAX}}$, then ${R}_{\textrm{MIN-MAX}}$ will be used in $\mathcal{L}(\bm{\bar{\theta}}/\bm{\theta})$.

\paragraph{Importance weights clipping}
When employing \pois and \ttpois in domains with discrete actions (e.g., the supply chain), it might happen that in some states some actions are highly sub-optimal. In this case, even if we are controlling the variance of the importance weights, the objective function might lead to shrink their probability to $0$. In training, this can results in NAN gradients and numerical instabilitities. For this reason, we clip the importance weights in $\hat{J}_{\bm{m}}(\bm{\bar{\theta}}/\bm{\theta})$ with an hyper-parameter $IW_c$.


\section{Experiment Details and Additional Results}\label{app:exp}
In this Section, we provide further details on the experiments and additional results.
More specifically:
\begin{itemize}
\item Section~\ref{app:env-details} provides an in-depth description for each environment that has been considered.
\item Section~\ref{app:exp-policy-eval} provides results that purely focus on the evaluation setting of Section \ref{sec:eval}.
\item Section~\ref{app:exp-opt-ablation} provides ablation experiments on the policy optimization setting.
\item Section~\ref{app:additional-opt-results} provides additional results on the experimental setting of Section \ref{sec:exp}. More specifically, results with additional values of $\Lambda$ and $\gamma$ are presented.
\item Section~\ref{app:undiscounted-res} provides additional results on the experimental setting of Section \ref{sec:exp}. More specifically, the undiscounted return metric is reported.
\item Section~\ref{app:hyper-param} reports hyper-parameters and other practical details.
\end{itemize}

\subsection{Environment Details}\label{app:env-details}

\subsubsection{Evaluation Domain}
We now provide a description of the environment that is used to conduct evaluation experiments, whose results are presented in Section \ref{app:exp-policy-eval}.
More specifically, we designed a domain with the following features:
\begin{itemize}
\item The performance of any policy can be easily computed in closed form.
\item It can easily generalize to any value of $T$ so that we can study the behavior of the algorithm varying $T$.
\end{itemize}
Given these general features, we designed the following environment. The state is described solely by the integer variable $t$, which represents the step in which the action is taken.  The action space is discrete, with $2$ possible actions. Concerning the reward function, since we want it to generalize to any horizon $T$, we made the following design choices. We restricted ourselves to $T \in \{100, 1000, 2000\}$. Then, focus for the sake of exposition on $T=100$. Define $\bm{g}_1 = (1, 4, 3, 1, 1.5, 0.4, 4, 4.1, 3, 2, 4)$ and $\bm{g}_2 = (4, 1, 1, 3, 4, 1.5, 0.1, 5, 1, 1, 4)$. Then, if $t \notin \{0, 10, 20, \dots, 90, 99\}$, $R(a_1, s_t) = 0$ and $R(a_2, s_t) = 0$. If $ t \in  \{0, 10, 20, \dots, 90, 99\}$, denote with $i(t)$ the corresponding index of the element $t$ within the vector $\{0, 10, 20, \dots, 90, 99\}$; then $R(a_1, s_t) = \mathcal{N}\left(g_{1,i(t)}, 0.1\right)$ and $R(a_2, s_t) = \mathcal{N}\left(g_{2,i(t)}, 0.1\right)$.

Similar reasoning extends to the cases in which $T$ is equal to $1000$ and $2000$ by considering the vectors $\{0, 100, 200, 300, 400, 500, 600, 700, 800, 900, 999\}$ and $\{0,  200,  400,  600,  800, 1000, 1200, 1400, 1600, 1800, 1999\}$ respectively. Further details can be found in the code base we provide.



\subsubsection{Corridor Domain}
In order to compare \pois and \ttpois on very similar domains but with different reward functions we design the following experiment. More specifically, the domain represents a corridor: the agent starts in the middle, and needs to reach the right extreme. To this end, it has two possible actions: ``go left" and ``go right", which both succeed with a certain probability. Then, we consider the following reward functions:
\begin{itemize}
\item Sparse reward: the reward is equal to $1$ only if the agent has reached the right extreme of the corridor, and $0$ otherwise. 
\item Dense reward: if the selected action is ``go right", the agent receives with high probability reward $1$, while if the selected action is ``go left", it receives with high probability reward $-1$.
\end{itemize}

More formally, we consider a continuous state space $\mathcal{S} \in [-x_{\text{MAX}}, x_{\text{MAX}}]$ for some $x_{\text{MAX}} > 0$. The initial state is fixed, and equal to $0$, namely $s_0 = 0$. Then, denote with $a_1$ the action ``go right" and with $a_2$ the action ``go left". Then, consider $\bar{p} \in (0.5, 1)$ (i.e., the probability of success of a given action). Let $x_t$ be the current state, then, if $a_t = a_1$, $x_{t+1} = 1+x_t+q$ where $q \sim \mathcal{N}(0, 0.1)$ with probability $\bar{p}$, and $x_{t+1} = -1+x_t+q$ where $q \sim \mathcal{N}(0, 0.1)$ with probability $1-\bar{p}$. Similarly, if $a_t = a_2$, $x_{t+1} = -1+x_t+q$ where $q \sim \mathcal{N}(0, 0.1)$ with probability $\bar{p}$, and $x_{t+1} = 1+x_t+q$ where $q \sim \mathcal{N}(0, 0.1)$ with probability $1-\bar{p}$. Then, we also clip the value of $x_t$ to be in the specified range; namely we clip $x_t$ in $[-x_{\text{MAX}}, x_{\text{MAX}}]$. Then, let the goal state be $x_g = x_{\text{MAX}}$. We say that the goal is reached whenever $|x - x_g| < 0.5$ holds. Furthermore, states for which the goal is reached are modelled as absorbing states.

For what concerns the reward function, instead, let us first consider the sparse reward setting. In this case, let the goal state be $x_g = x_{\text{MAX}}$. Then, $R(x_t, a_t) = 1$ if $|x_t - x_g| < 0.5$, $0$ otherwise. Furthermore, we set $T=100$ and $x_g = 12$.\footnote{As we shall show, at the beginning of the training process, with these parameters, the agent will rarely reach the goal.}

For the dense reward setting, instead, $R(x, \cdot) = 0$ for $x$ such that $|x - x_g| < 0.5$ holds. For $x$ such that $|x - x_g| < 0.5$ does not hold instead, $R(\cdot, a_1)$ is $0.2$ with probability $\bar{p}$ and $-0.2$ with probability $1-\bar{p}$. Moreover, $R(\cdot, a_2)$ is $-0.2$ with probability $\bar{p}$ and $0.2$ with probability $1-\bar{p}$. 
Furthermore, we set $T=1000$ and $x_g = 1000$.

Further details can be found in the code base we provide.

\subsubsection{Dam Control}
We now provide additional details on the dam environment. First, we detail the adopted parameters of the environment, then we describe the state observed by the agents and the action space considered. 

\begin{figure*}[t!]
\centering\includegraphics[width=8cm]{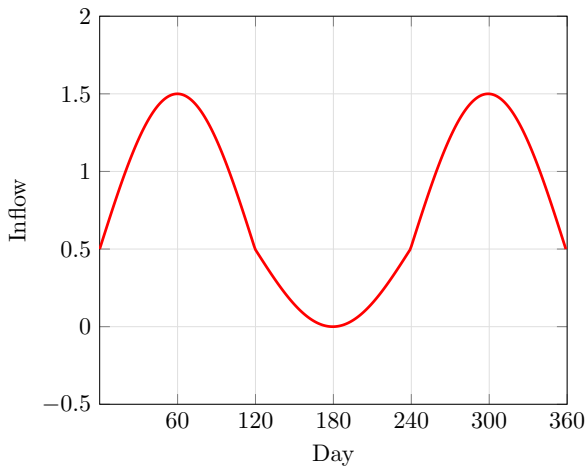} 
\vspace{.3in}
\caption{Mean inflow $i_t$ per day $t$ of the Dam control environment considered in our experiments. The inflow is considered over a period of $1$ year.} 
\label{fig:daminflow}
\end{figure*}

The parameters of the environment are the default ones of \cite{tirinzoni2018importance}:
\begin{itemize}
\item The demand $D$ is fixed and equal to $10$.
\item The inflow profile $i_t$ is a period function (i.e., Figure~\ref{fig:daminflow}) plus Gaussian noise with $\sigma=2$.
\item The initial storage $s_0$ is set to $200$.
\item The flooding threshold $F$ is set to $300$.
\item The reward is computed as $-c_1 \max\{0, s_t - F \} -c_2 \max\{0, D - a_t \}^2$ with $c_1 = c_2 = 0.5$. Moreover, rewards are rescaled with $0.01$ for stability purposes.
\end{itemize}

The state given by the agent is a $7$-dimensional vector given by:
\begin{itemize}
\item The storage at the current day, namely $s_t$. This quantity is normalized using $2 * \frac{s_t - 50}{500-50} - 1$.
\item $6$ basis functions $\phi_i(t)$ (with $i \in \{1, \dots 6\}$) are used to describe the time $t$. More specifically, $\phi_i(t) = |t - c_i|$, where $c_i = \{60, 120, 180, 240, 300, 360 \}$. The agent then observes a normalized version of $\phi_i(t)$, namely $2 * \frac{\phi_i(t)}{360} - 1$.
\end{itemize}
The agent always start at day $t=0$ with a storage $s_0$ of $200$ and the interaction proceeds for $1080$ days (i.e., $3$ years).
The available actions $a_t$ are discrete and $21$. Each action $i$ represents the amount of water that the agent intends to release at day $t$. Differently from \citet{tirinzoni2018importance}, we are considering a case in which there are operational constraints on the amount of water to release (i.e., the agent cannot take all the possible values $a_t \in [0, s_t]$, but it is limited to $\{0, \dots, 20\}$). Moreover, we consider a control frequency of $3$ days: once an action has been chosen at day $t$, it is persisted for $3$ days in a row. This is mainly for performance reasons: all the policy-gradient based method that we tried were failing to learn without this additional trick. 


\subsubsection{Reacher}
The domain is the standard one from the MuJoCo control suite \citep{mujoco}. We set the episode duration to $T=200$ timesteps, with a new goal target popping up if the previous one is reached.

\subsubsection{Multi-Echelon Supply Chain}
As originally done, we consider the optimization problem over a period of $30$ days. All the details concerning this domain (e.g., demands, lead times, inventory costs, initial states, backlog costs, prices are products are sold) can be found in \citet{hubbs2020or}, indeed, we rely on their publicly available repository for our experiments. We report here, however, a couple of modifications that we have taken to improve the performances.  The state of the agent (i.e., a vector $v$ with dimension $33$) has been normalized according to $\frac{v}{20} - 1$. The action space, that was originally a multi-discrete space of dimension $[100, 90, 80]$, has been shrinked to $[25, 25, 25]$ to speed up the learning process. Indeed, larger action values are highly sub-optimal given the demand curve and the inventory costs. 


\subsection{Policy Evaluation Experiments}\label{app:exp-policy-eval}

\begin{figure*}[t!]
\centering\includegraphics[width=\textwidth]{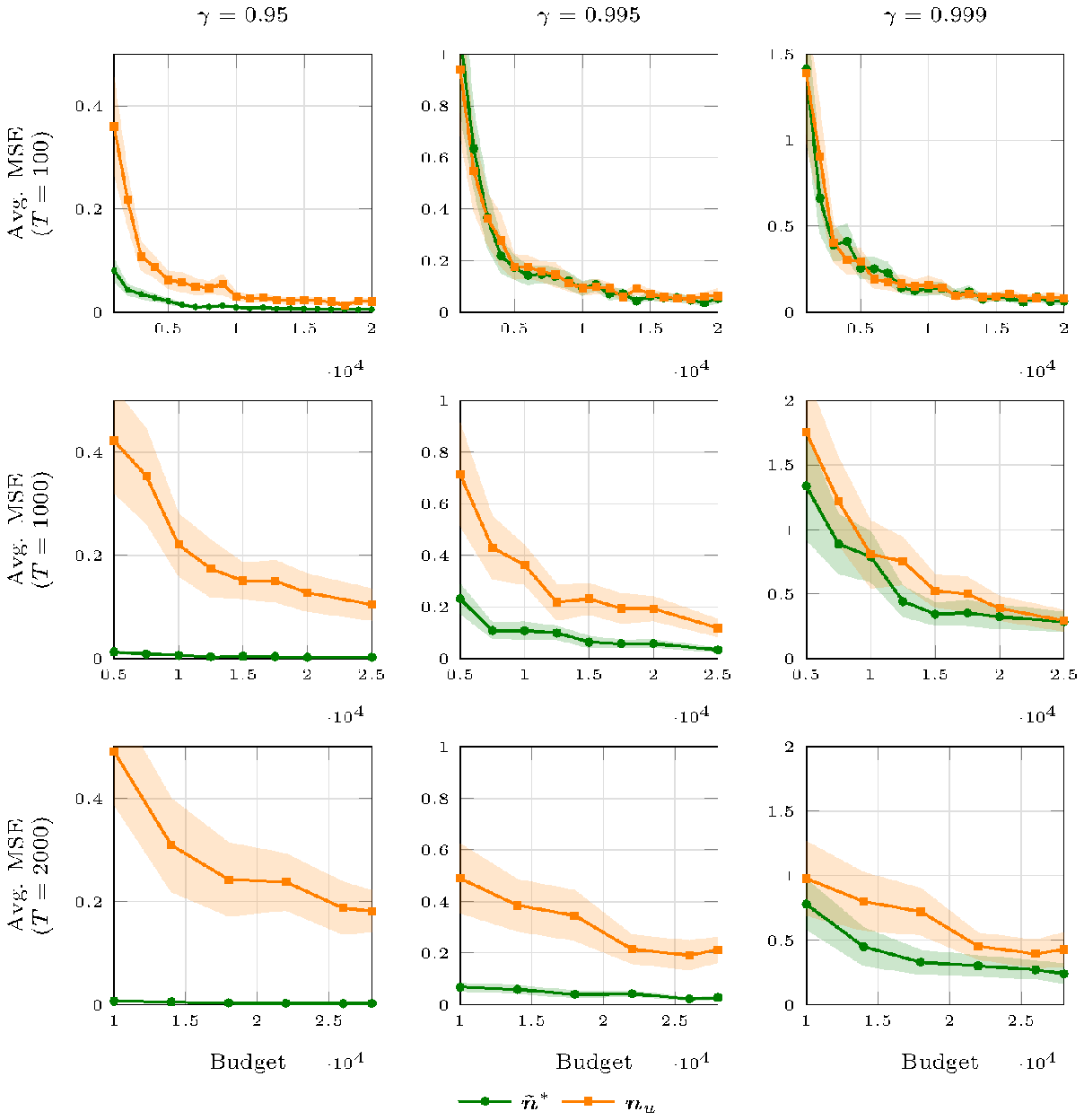} 
\vspace{.3in}
\caption{Experimental results (mean and $95\%$ confidence intervals of $100$ runs) on the Policy Evaluation domain described in Section \ref{app:env-details}. Each picture reports the average MSE (i.e., $y$-axis) against the budget that has been spent to collect trajectories (i.e., $x$-axis). More specifically, we report results when using our approximately optimal DCS $\bm{\tilde{n}}^*$ and the uniform data collection strategy $\bm{n}_u$. The policy that is estimated is the uniform one, and the data have been collected on-policy. The first row of the figure is obtained with $T=100$, the second one with $T=1000$, and the third one with $T=2000$. The third column with $\gamma=0.95$, the second one with $\gamma=0.995$, and the third one with $\gamma=0.999$.} 
\label{fig:onpoleval}
\end{figure*}

\begin{figure*}[t!]
\centering\includegraphics[width=\textwidth]{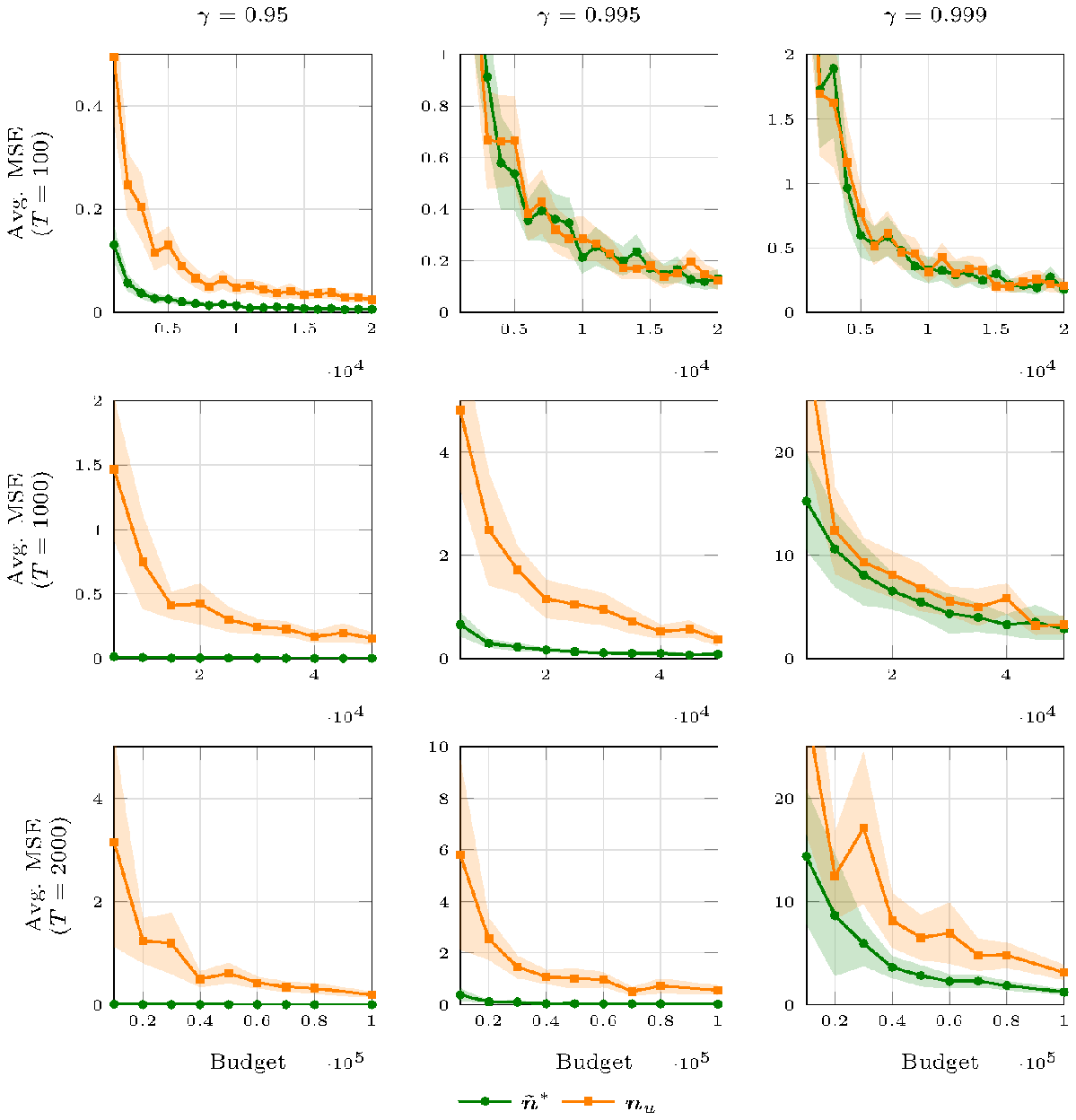} 
\vspace{.3in}
\caption{Experimental results (mean and $95\%$ confidence intervals of $100$ runs) on the Policy Evaluation domain described in Section \ref{app:env-details}. Each picture reports the average MSE (i.e., $y$-axis) against the budget that has been spent to collect trajectories (i.e., $x$-axis). More specifically, we report results when using our approximately optimal DCS $\bm{\tilde{n}}^*$ and the uniform data collection strategy $\bm{n}_u$. The policy that is estimated takes $a_1$ with probability $0.49$ and $a_2$ with probability $0.51$; the data have been collected using the random policy (i.e., off-policy evaluation). The first row of the figure is obtained with $T=100$, the second one with $T=1000$, and the third one with $T=2000$. The third column with $\gamma=0.95$, the second one with $\gamma=0.995$, and the third one with $\gamma=0.999$.} 
\label{fig:offpoleval}
\end{figure*}

Figure~\ref{fig:onpoleval} and \ref{fig:offpoleval} report results for our experimental evaluation setting. Figure~\ref{fig:onpoleval} studies the on-policy evaluation problem, where we want to evaluate the random policy with data collected from the random policy itself. Figure~\ref{fig:offpoleval}, instead, focuses on the off-policy setting: we consider the problem of estimating the policy that takes $a_1$ with probability $0.49$ and $a_2$ with probability $0.51$ with data collected from the random policy. In both Figures, each picture compares the performance, in term of MSE, of $\bm{\tilde{m}}^*$ against the usual uniform-in-the-horizon DCS.\footnote{Notice that in the considered environment, the exact value of any policy can easily be computed in closed form. It follows that the MSE can be computed exactly.} Each experiment shows the mean MSE, together with $95\%$ confidence intervals, over $100$ runs. To conduct exhaustive experimentation, we have  varied both the value of $\Lambda$ abd $\gamma$. The results are consistent with our theory: for small values of $\gamma$ (or, equivalently, for larger values of $T$) the benefits of $\bm{\tilde{m}}^*$ increases.

\subsection{Policy Optimization Ablations}\label{app:exp-opt-ablation}

\begin{figure*}[t!]
\centering\includegraphics[width=\textwidth]{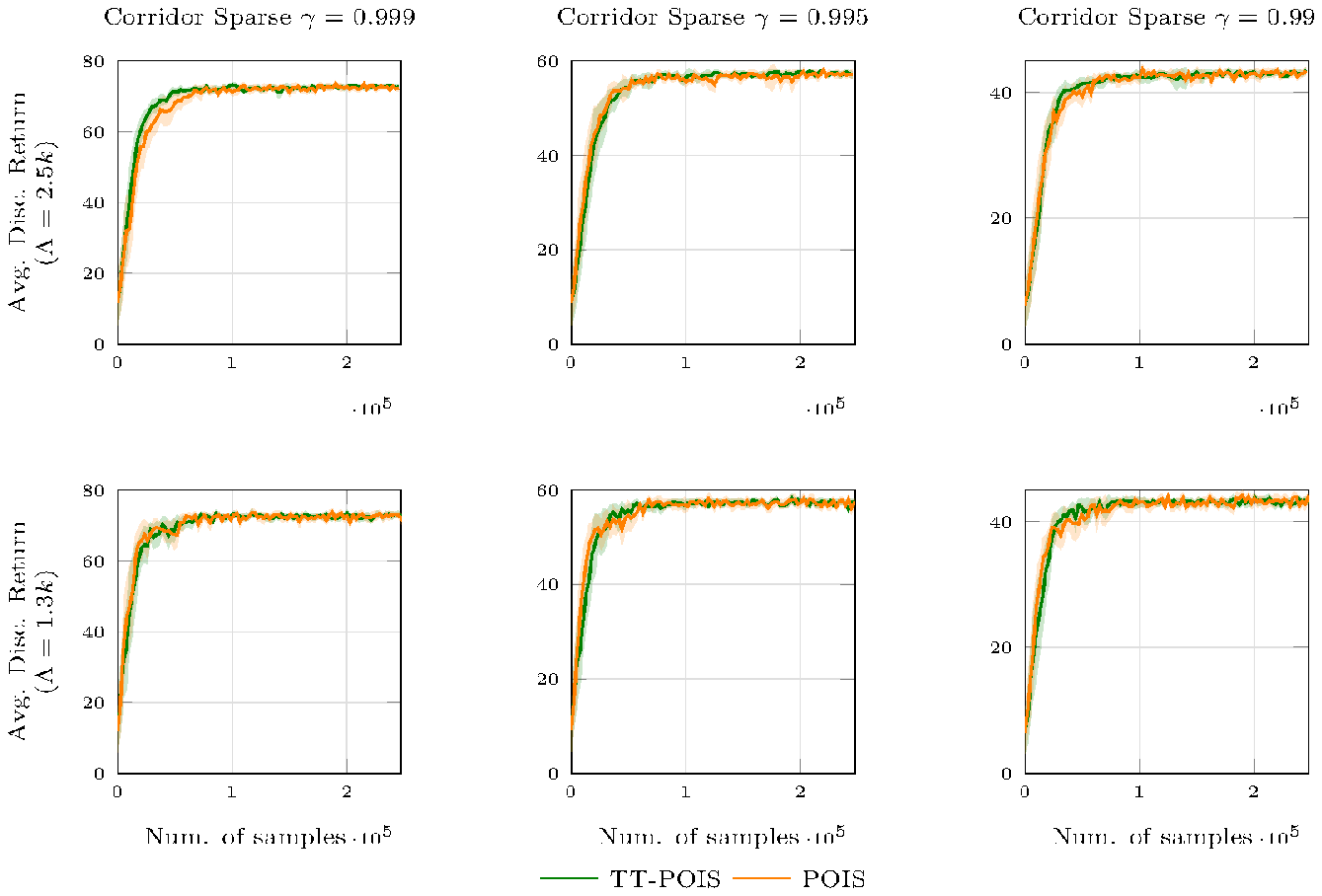} 
\vspace{.3in}
\caption{Experimental results (mean and $95\%$ confidence intervals of $15$ runs) on the Corridor Sparse domain with different values of $\gamma$ and $\Lambda$. More specifically, the first row is obtained with $\Lambda=2500$ and the second one with $\Lambda=1300$. The first column is obtained training the algorithm with $\gamma=0.999$, the second one with $\gamma=0.995$, and the third one with $\gamma=0.99$. The reported metric is the average of the discounted return with the corresponding value of $\gamma$.} 
\label{fig:corrsparsedisc}
\end{figure*}

\begin{figure*}[t!]
\centering\includegraphics[width=\textwidth]{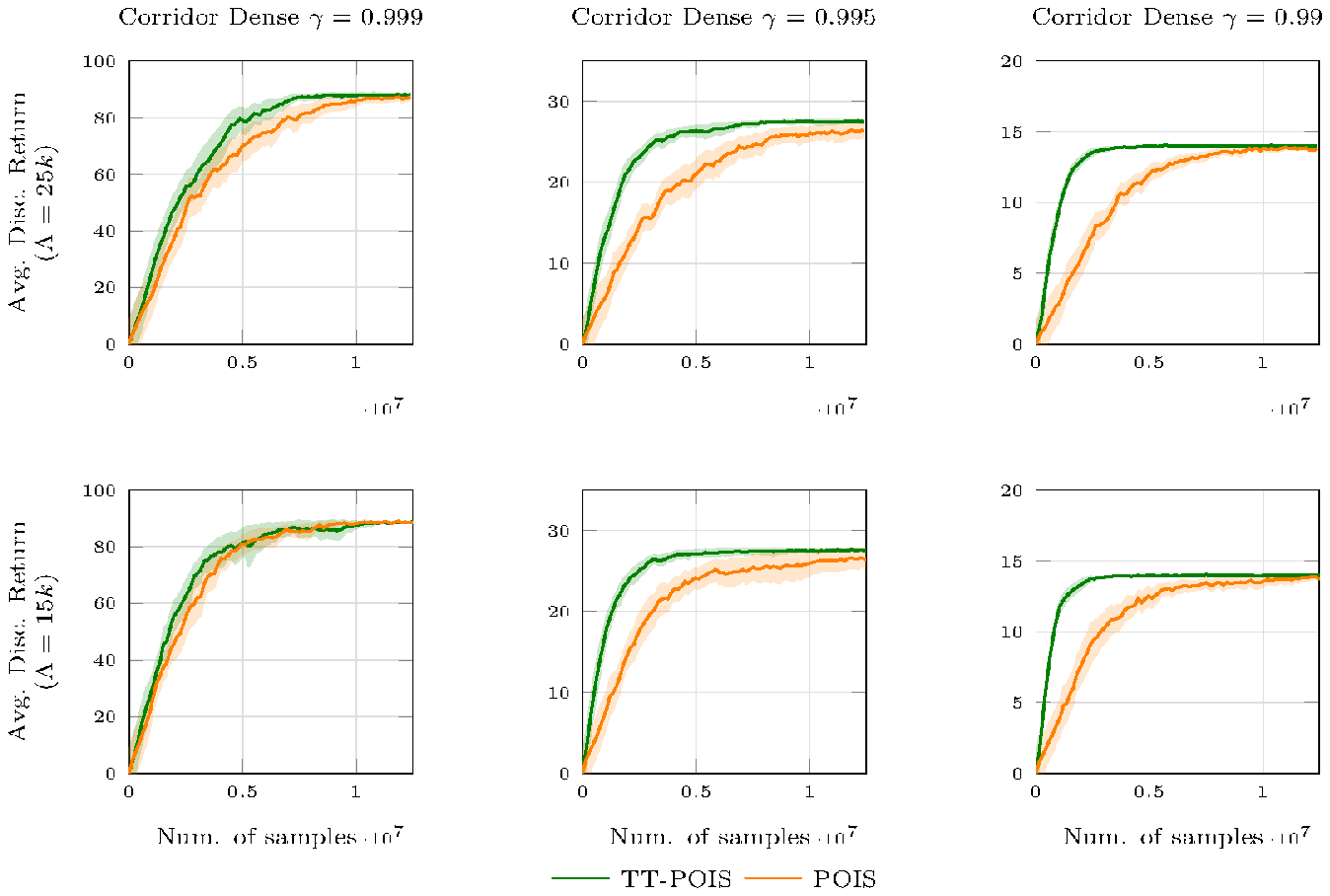} 
\vspace{.3in}
\caption{Experimental results (mean and $95\%$ confidence intervals of $15$ runs) on the Corridor Dense domain with different values of $\gamma$ and $\Lambda$. More specifically, the first row is obtained with $\Lambda=25000$ and the second one with $\Lambda=15000$. The first column is obtained training the algorithm with $\gamma=0.999$, the second one with $\gamma=0.995$, and the third one with $\gamma=0.99$. The reported metric is the average of the discounted return with the corresponding value of $\gamma$.} 
\label{fig:corrdensedisc}
\end{figure*}

\begin{figure*}[t!]
\centering\includegraphics[width=\textwidth]{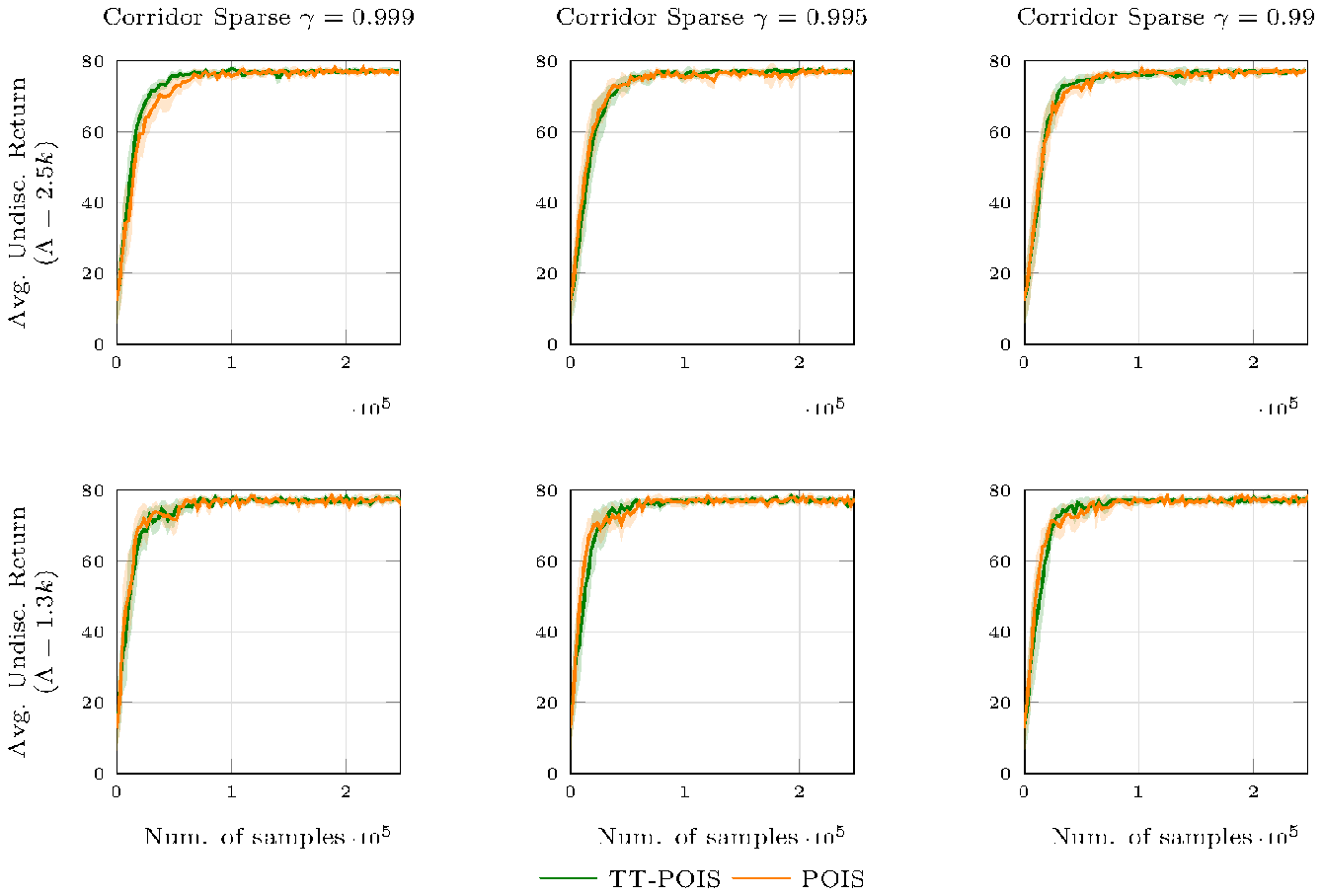} 
\vspace{.3in}
\caption{Experimental results (mean and $95\%$ confidence intervals of $15$ runs) on the Corridor Sparse domain with different values of $\gamma$ and $\Lambda$. More specifically, the first row is obtained with $\Lambda=2500$ and the second one with $\Lambda=1300$. The first column is obtained with $\gamma=0.999$, the second one with $\gamma=0.97$, and the third one with $\gamma=0.99$. The reported metric is the average of undiscounted return (i.e., $\gamma=1$.)} 
\label{fig:corrsparse}
\end{figure*}

\begin{figure*}[t!]
\centering\includegraphics[width=\textwidth]{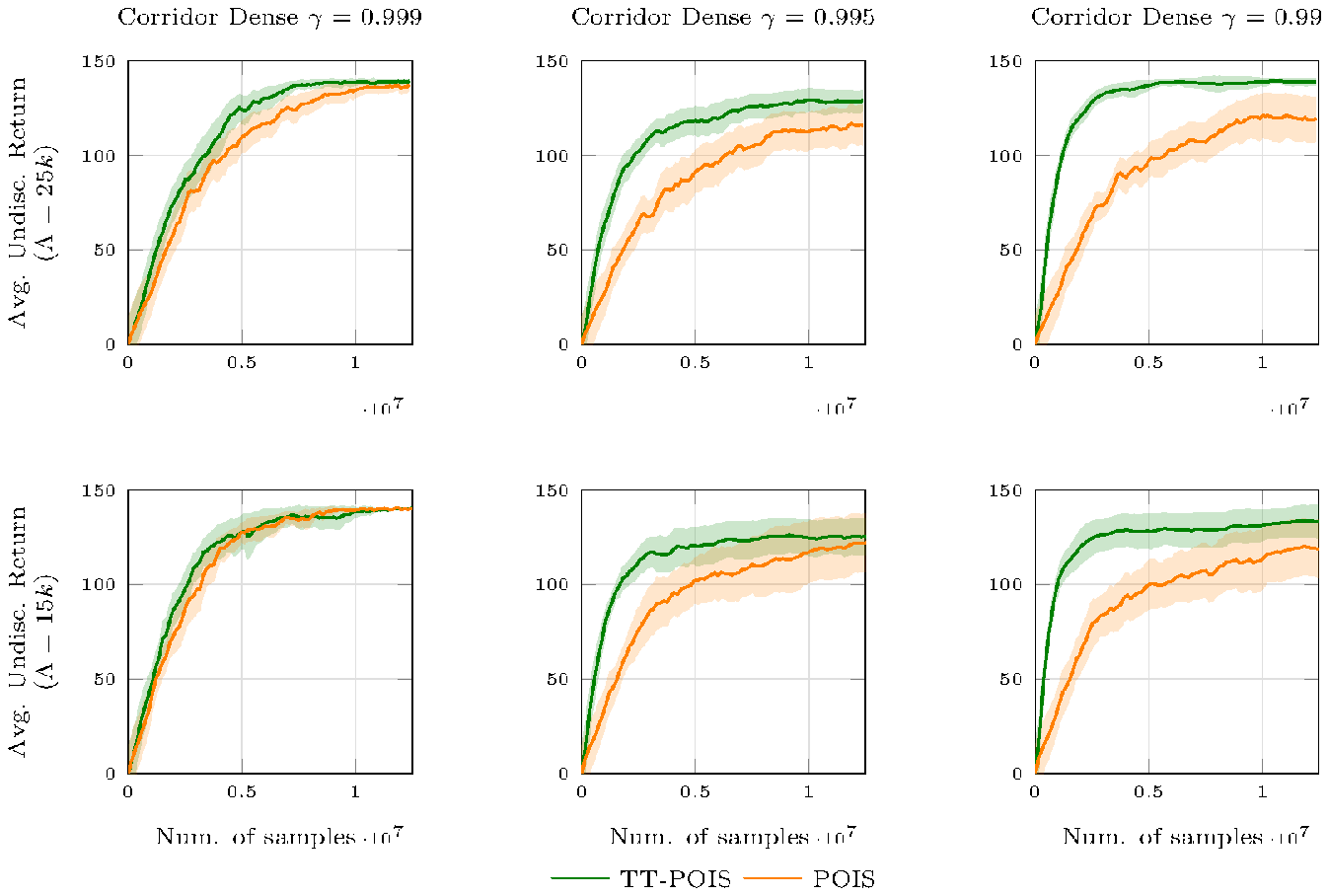} 
\vspace{.3in}
\caption{Experimental results (mean and $95\%$ confidence intervals of $15$ runs) on the Corridor Dense domain with different values of $\gamma$ and $\Lambda$. More specifically, the first row is obtained with $\Lambda=25000$ and the second one with $\Lambda=15000$. The first column is obtained with $\gamma=0.999$, the second one with $\gamma=0.97$, and the third one with $\gamma=0.99$. The reported metric is the average of undiscounted return (i.e., $\gamma=1$.)} 
\label{fig:corrdense}
\end{figure*}

Figure~\ref{fig:corrsparsedisc} and \ref{fig:corrdensedisc} report results on the Corridor domain presented in Section \ref{app:env-details}. More specifically, Figure~\ref{fig:corrsparsedisc} reports the result for the sparse reward setting, while Figure~\ref{fig:corrdensedisc} focuses on the dense reward one. Each plot shows the mean and the $95\%$ confidence intervals over $15$ runs of the average discounted return. To conduct exhaustive experimentation, we varied the values of $\gamma$ and $\Lambda$.

Some observations are in order. First of all, from Figure~\ref{fig:corrsparsedisc}, \ttpois shows a robust behavior even in this sparse reward scenario. Notice that, at the beginning of the learning process, since the performance is close to $0$, the agent rarely reaches the goal (i.e., it sparsely receives positive feedback from the environment). Yet, \ttpois still obtains the same learning curves as \pois. Secondly, from Figure~\ref{fig:corrdensedisc}, we can appreciate a significant benefit of \ttpois over \pois for the dense reward setting, especially with small values of $\gamma$. This is consistent with what have been highlighted in Section~\ref{sec:exp}. 
Finally, Figure~\ref{fig:corrdense} and \ref{fig:corrsparse} report the undiscounted average return as a metric. The previous considerations extend to these results as well.



\subsection{Additional Optimization Results: varying $\Lambda$ and $\gamma$}\label{app:additional-opt-results}

\begin{figure*}[t!]
\centering\includegraphics[width=\textwidth]{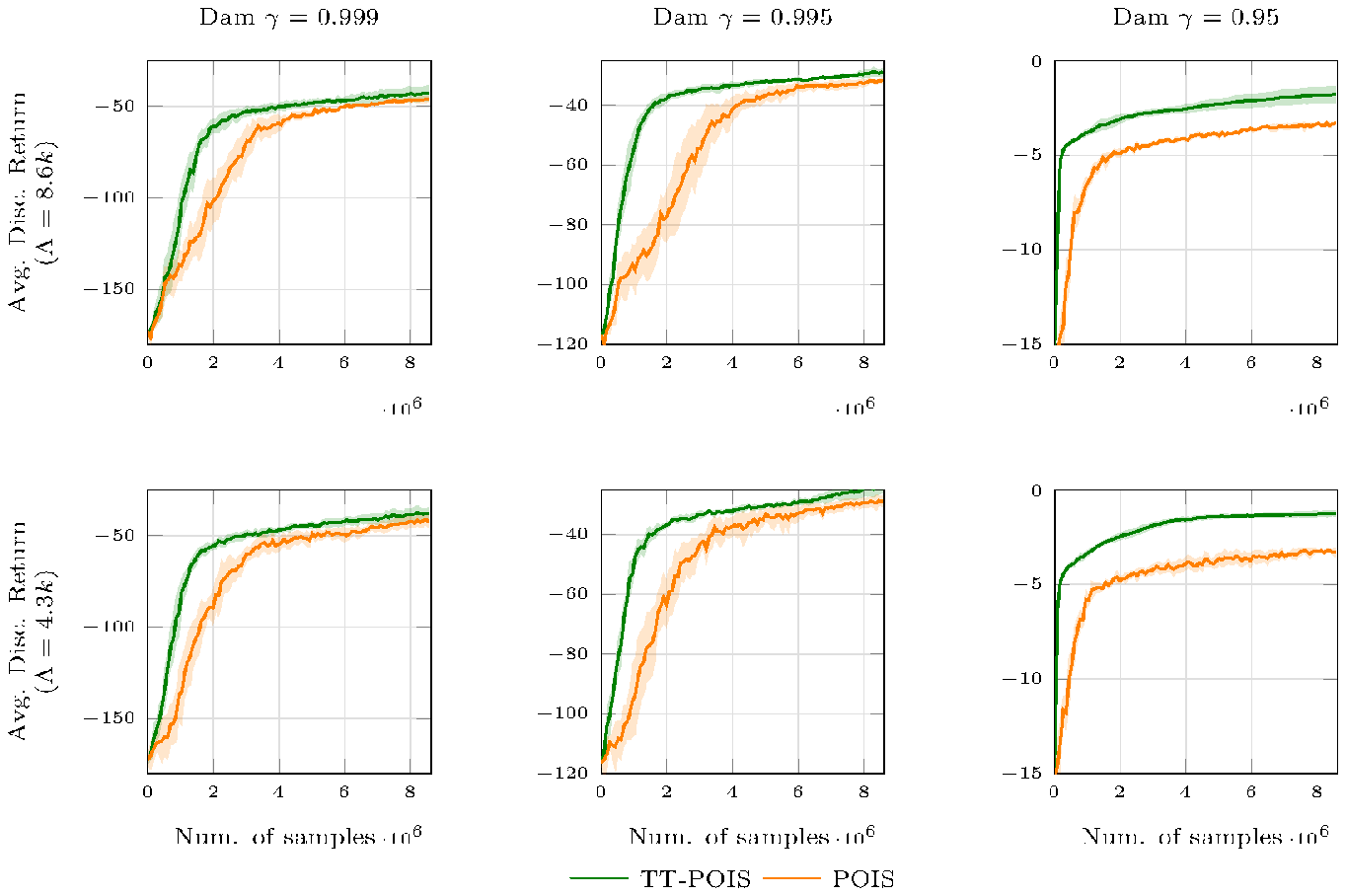} 
\vspace{.3in}
\caption{Experimental results (mean and $95\%$ confidence intervals of $5$ runs) on the Dam domain with different values of $\gamma$ and $\Lambda$. More specifically, the first row is obtained with $\Lambda=8640$ and the second one with $\Lambda=4320$. The first column is obtained training the algorithm with $\gamma=0.999$, the second one with $\gamma=0.995$, and the third one with $\gamma=0.95$. The reported metric is the average of the discounted return with the corresponding value of $\gamma$.} 
\label{fig:damresultsdisc}
\end{figure*}

\begin{figure*}[t!]
\centering\includegraphics[width=\textwidth]{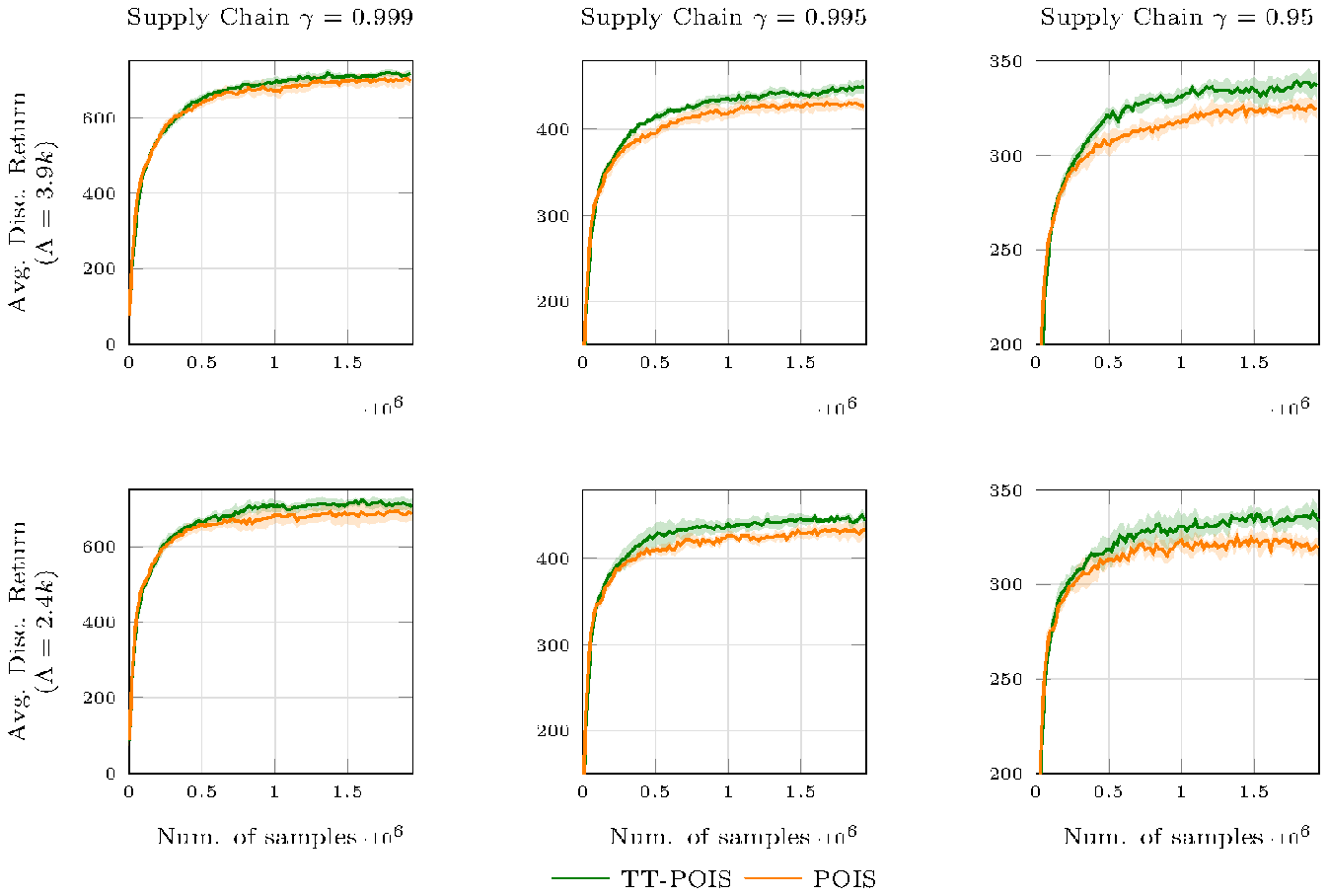} 
\vspace{.3in}
\caption{Experimental results (mean and $95\%$ confidence intervals of $5$ runs) on the Supply Chain domain with different values of $\gamma$ and $\Lambda$. More specifically, the first row is obtained with $\Lambda=3900$ and the second one with $\Lambda=2400$. The first column is obtained training the algorithm with $\gamma=0.999$, the second one with $\gamma=0.97$, and the third one with $\gamma=0.95$. The reported metric is the average of the discounted return with the corresponding value of $\gamma$.} 
\label{fig:imresultsdisc}
\end{figure*}

\begin{figure*}[t!]
\centering\includegraphics[width=\textwidth]{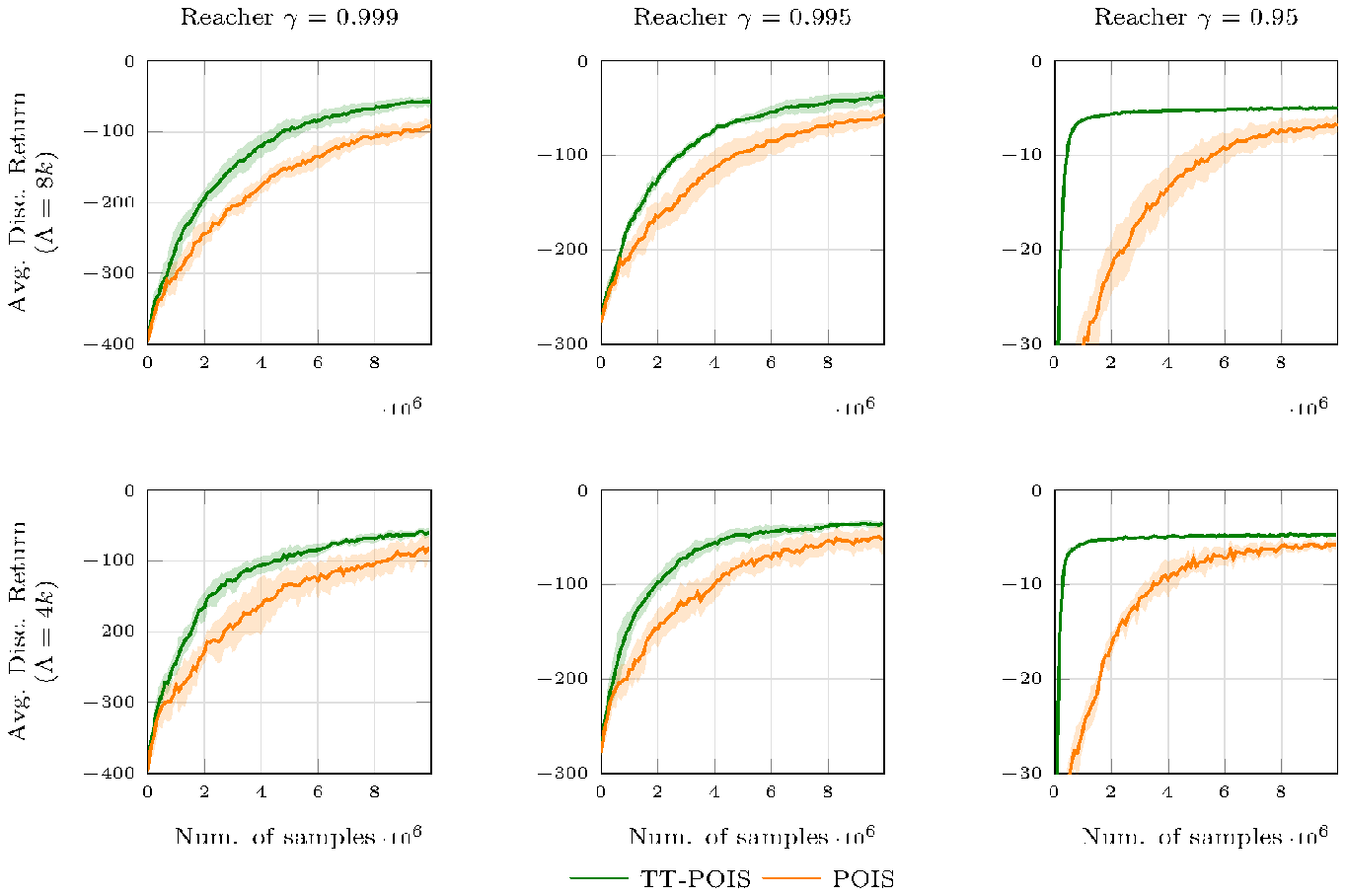} 
\vspace{.3in}
\caption{Experimental results (mean and $95\%$ confidence intervals of $5$ runs) on the Reacher domain with different values of $\gamma$ and $\Lambda$. More specifically, the first row is obtained with $\Lambda=8000$ and the second one with $\Lambda=4000$. The first column is obtained training the algorithm with $\gamma=0.999$, the second one with $\gamma=0.995$, and the third one with $\gamma=0.95$. The reported metric is the average of the discounted return with the corresponding value of $\gamma$.} 
\label{fig:reachresultsdisc}
\end{figure*}

In this Section, we provide additional experimental optimization results. More specifically, Figures \ref{fig:damresultsdisc}, \ref{fig:imresultsdisc} and \ref{fig:reachresultsdisc} reports results for the Dam, Supply Chain and Reacher environments respectively (average return over $5$ run with $95\%$ confidence intervals), while varying the value of $\Lambda$ and $\gamma$. For the Dam domain we test the following combinations of $\Lambda$ and $\gamma$: $\gamma \in \{0.95, 0.995, 0.999 \}$ and $\Lambda \in \{4320, 8640 \}$. For the Reacher domain, instead, we test $\gamma \in \{0.95, 0.995, 0.999 \}$ and $\Lambda \in \{4000, 8000\}$. Finally, for the Supply Chain, we report $\gamma \in \{0.95, 0.97, 0.999 \}$ and $\Lambda \in \{2400, 3900\}$.

As one can see, what has been highlighted in Section~\ref{sec:exp} replicates consistently.

\clearpage

\subsection{Additional Optimization Results: undiscounted performance}\label{app:undiscounted-res}

\begin{figure*}[t!]
\centering\includegraphics[width=\textwidth]{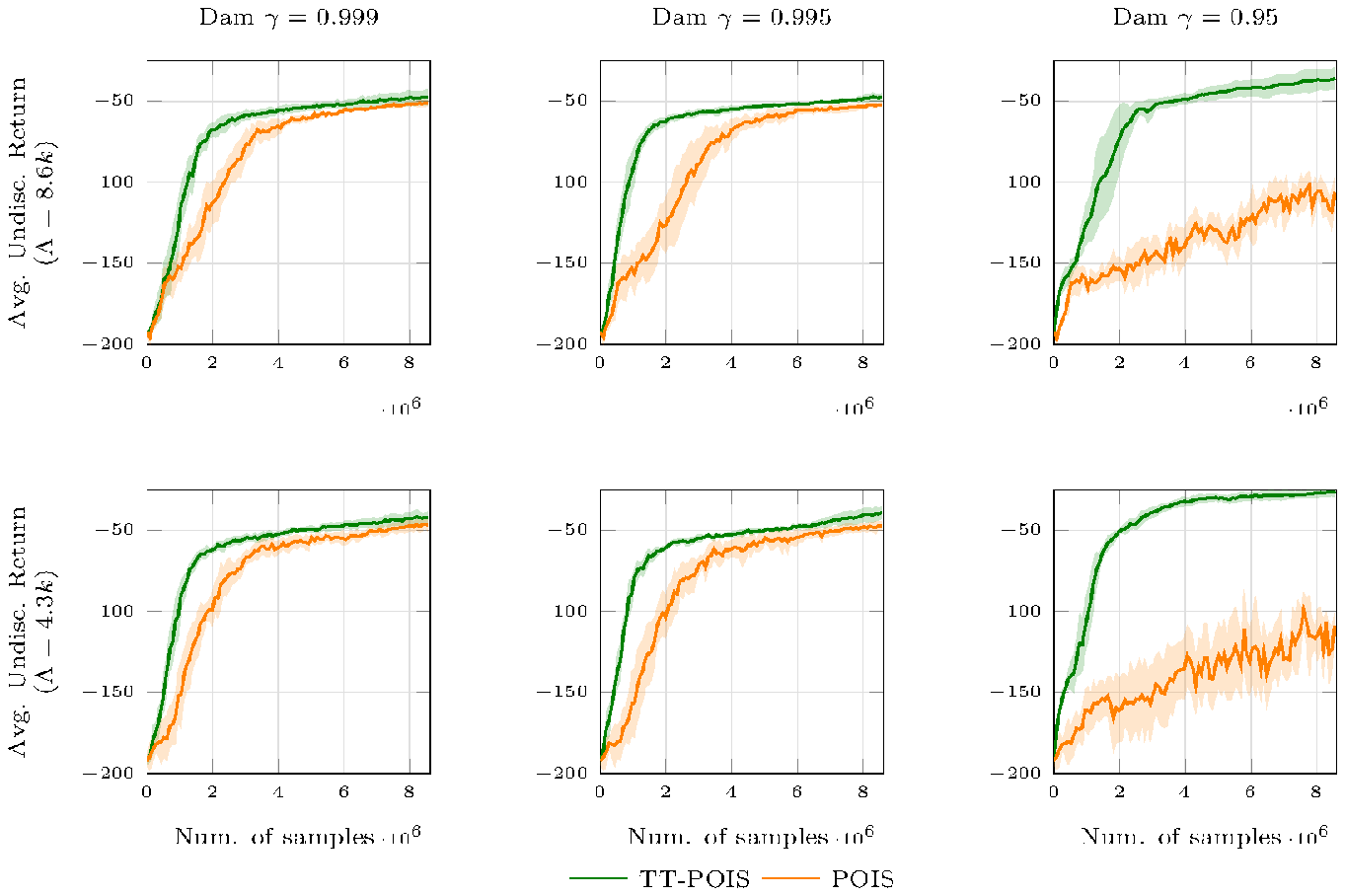} 
\vspace{.3in}
\caption{Experimental results (mean and $95\%$ confidence intervals of $5$ runs) on the Dam domain with different values of $\gamma$ and $\Lambda$. More specifically, the first row is obtained with $\Lambda=8640$ and the second one with $\Lambda=4320$. The first column is obtained with $\gamma=0.999$, the second one with $\gamma=0.995$, and the third one with $\gamma=0.95$. The reported metric is the average of undiscounted return (i.e., $\gamma=1$.)} 
\label{fig:damresults}
\end{figure*}

\begin{figure*}[t!]
\centering\includegraphics[width=\textwidth]{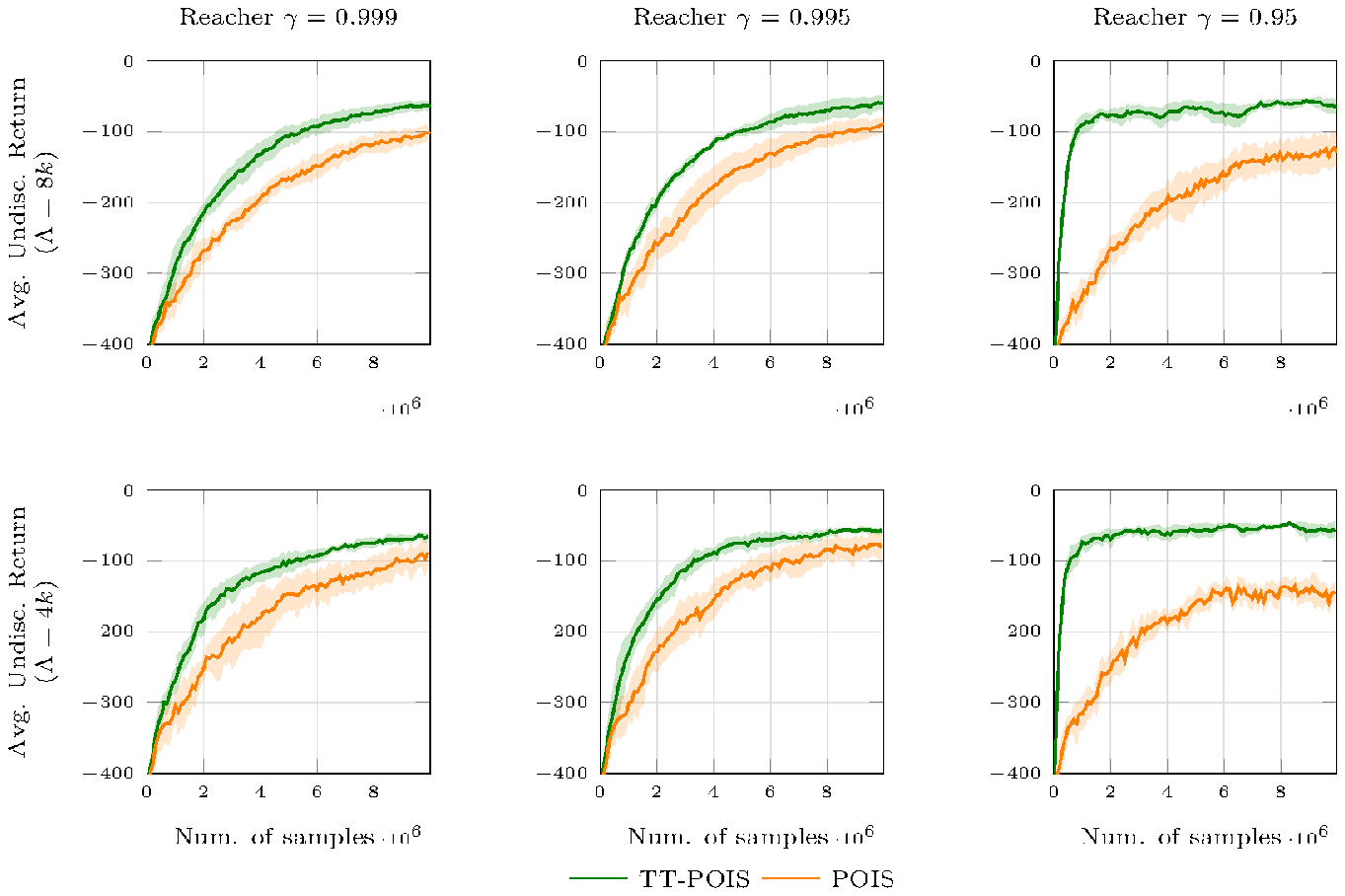} 
\vspace{.3in}
\caption{Experimental results (mean and $95\%$ confidence intervals of $5$ runs) on the Reacher domain with different values of $\gamma$ and $\Lambda$. More specifically, the first row is obtained with $\Lambda=8000$ and the second one with $\Lambda=4000$. The first column is obtained with $\gamma=0.999$, the second one with $\gamma=0.995$, and the third one with $\gamma=0.95$. The reported metric is the average of undiscounted return (i.e., $\gamma=1$.)} 
\label{fig:reachresults}
\end{figure*}

\begin{figure*}[t!]
\centering\includegraphics[width=\textwidth]{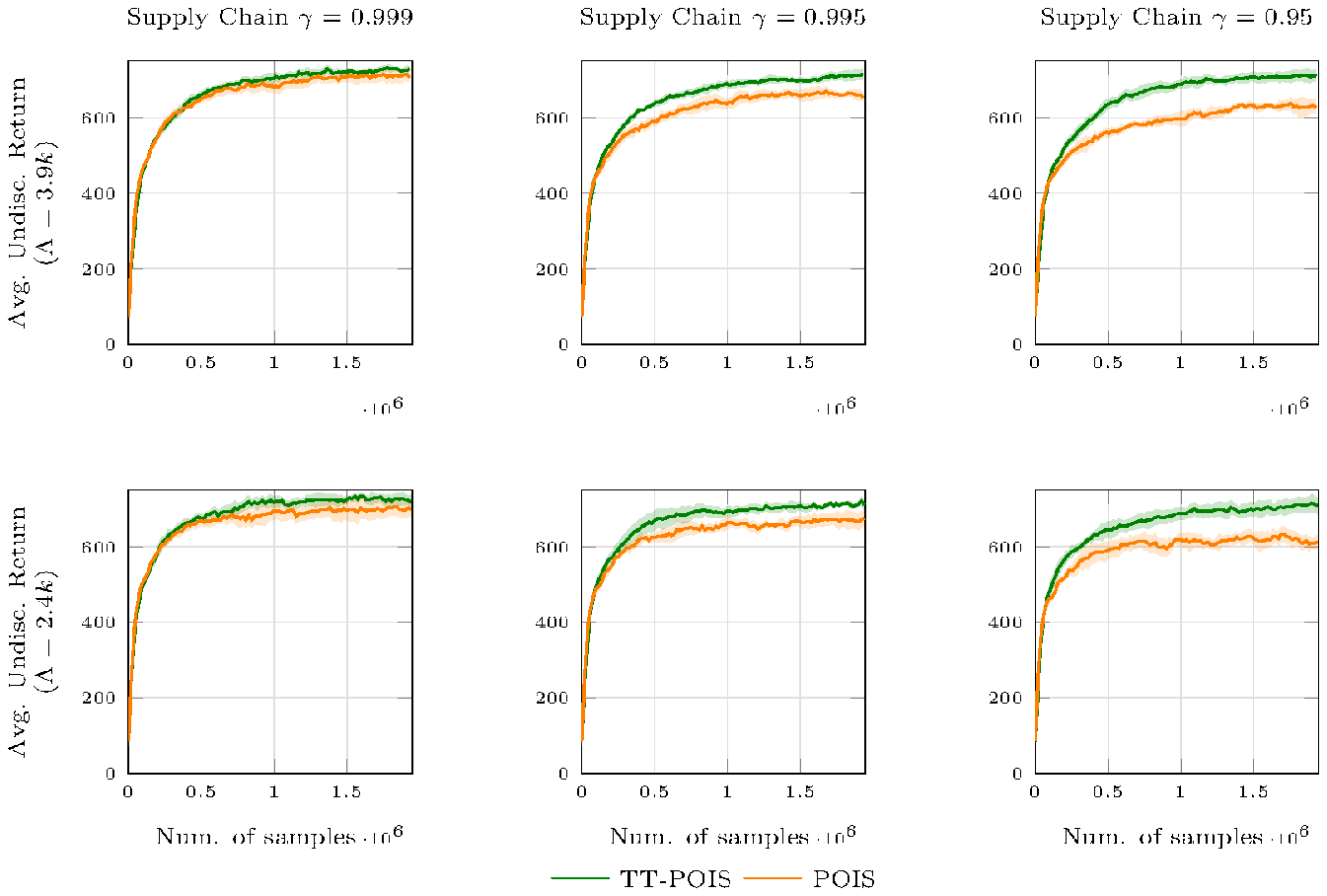} 
\vspace{.3in}
\caption{Experimental results (mean and $95\%$ confidence intervals of $5$ runs) on the Supply Chain domain with different values of $\gamma$ and $\Lambda$. More specifically, the first row is obtained with $\Lambda=3900$ and the second one with $\Lambda=2400$. The first column is obtained with $\gamma=0.999$, the second one with $\gamma=0.97$, and the third one with $\gamma=0.95$. The reported metric is the average of undiscounted return (i.e., $\gamma=1$.)} 
\label{fig:imresults}
\end{figure*}

In this Section, we report the undiscounted average return for the experiments of Section \ref{app:additional-opt-results}. More specifically, Figures~\ref{fig:damresults}, \ref{fig:reachresults} and \ref{fig:imresults} reports results for the Dam, Reacher and Supply Chain environments respectively (average undiscounted return over $5$ run with $95\%$ confidence intervals). 

As we can appreciate, in these scenarios, the advantages of \ttpois over \pois replicates even for the undiscounted return metric. 

\clearpage

\subsection{Hyper-parameters and other details}\label{app:hyper-param}
We have run the experiments using $88$ Intel(R) Xeon(R) CPU E7-8880 v4 @ 2.20GHz cpus and $94$ GB of RAM. 

We now provide details on the hyper-parameters that were used to generate the results. Table \ref{tab:corr-sparse}, \ref{tab:corr-dense}, \ref{tab:dam-pois}, \ref{tab:im-pois} and \ref{tab:reacher-pois} provides hyper-parameters for \pois and \ttpois on the different domains. For all values of $\gamma$, we used the same hyper-parameters. The hyper-parameters on the number of offline iterations refers to Line $4$ of Algorithm \ref{alg:ttpo}.

\begin{table}[h]
\caption{Corridor Sparse Rewards Hyper-parameters for \pois and \ttpois} \label{tab:corr-sparse}
\begin{center}
\begin{tabular}{lll}
\textbf{Hyper-parameter}  & $\bm{\Lambda=2500}$ & $\bm{\Lambda=1300}$ \\
\hline \\
Neural Network Size         					& $[64,32]$ & $[64,32]$ \\
Weight Initialization             			& Normc & Normc \\
Activation Function 							& Xavier & Xavier \\
Confidence $\delta$ 							& $0.9$ & $0.9$ \\
Number of offline iterations 				& $10$ & $10$ \\
Importance Weight Clipping 					& $100$ & $100$ \\
$R_{\textrm{MIN-MAX}}$ 						& Not applied & Not applied\\
\end{tabular}
\end{center}
\end{table}

\begin{table}[h]
\caption{Corridor Dense Rewards Hyper-parameters for \pois and \ttpois} \label{tab:corr-dense}
\begin{center}
\begin{tabular}{lll}
\textbf{Hyper-parameter}  & $\bm{\Lambda=25000}$ & $\bm{\Lambda=15000}$ \\
\hline \\
Neural Network Size         					& $[64,32]$ & $[64,32]$ \\
Weight Initialization             			& Normc & Normc \\
Activation Function 							& Xaiver & Xavier \\
Confidence $\delta$ 							& $0.7$ & $0.7$ \\
Number of offline iterations 				& $10$ & $10$ \\
Importance Weight Clipping 					& $100$ & $100$ \\
$R_{\textrm{MIN-MAX}}$ 						& Not applied & Not applied\\
\end{tabular}
\end{center}
\end{table}

\begin{table}[h]
\caption{Dam Hyper-parameters for \pois and \ttpois} \label{tab:dam-pois}
\begin{center}
\begin{tabular}{lll}
\textbf{Hyper-parameter}  & $\bm{\Lambda=8600}$ & $\bm{\Lambda=4320}$ \\
\hline \\
Neural Network Size         					& $[64,32]$ & $[64,32]$ \\
Weight Initialization             			& Normc & Normc \\
Activation Function 							& Tanh & Tanh \\
Confidence $\delta$ 							& $0.7$ & $0.6$ \\
Number of offline iterations 				& $10$ & $10$ \\
Importance Weight Clipping 					& Not applied & Not applied \\
$R_{\textrm{MIN-MAX}}$ 						& Not applied & Not applied\\
\end{tabular}
\end{center}
\end{table}

\begin{table}[h]
\caption{Supply Chain Hyper-parameters for \pois and \ttpois} \label{tab:im-pois}
\begin{center}
\begin{tabular}{lll}
\textbf{Hyper-parameter}  & $\bm{\Lambda=3900}$ & $\bm{\Lambda=2400}$ \\
\hline \\
Neural Network Size         					& $[100,50,25]$ & $[100,50,25]$ \\
Weight Initialization             			& Normc & Normc \\
Activation Function 							& Tanh & Tanh \\
Confidence $\delta$ 							& $0.005$ & $0.005$ \\
Number of offline iterations 				& $20$ & $20$ \\
Importance Weight Clipping 					& $100$ & $100$ \\
$R_{\textrm{MIN-MAX}}$ 						& Not applied & Not applied\\
\end{tabular}
\end{center}
\end{table}

\begin{table}[h]
\caption{Reacher Hyper-parameters for \pois and \ttpois} \label{tab:reacher-pois}
\begin{center}
\begin{tabular}{lll}
\textbf{Hyper-parameter} 					 & $\bm{\Lambda=8000}$ & $\bm{\Lambda=4000}$ \\
\hline \\
Neural Network Size         					& $[100,50,25]$ & $[100,50,25]$ \\
Weight Initialization             			& Normc & Normc \\
Activation Function 							& Tanh & Tanh \\
Confidence $\delta$ 							& $0.8$ & $0.8$ \\
Number of offline iterations 				& $20$ & $20$ \\
Importance Weight Clipping 					& Not applied & Not applied \\
$R_{\textrm{MIN-MAX}}$ 						& $5$ & $5$ \\
\end{tabular}
\end{center}
\end{table}

\end{document}